%% file: main.tex
\documentclass[11pt]{article}

\advance \oddsidemargin by -0.8in
\newcommand{\myparskip}{3pt}
\parskip \myparskip
\usepackage[utf8]{inputenc}      
\usepackage[numbers]{natbib}
\usepackage{booktabs}      
\usepackage{amsfonts}       
\usepackage{nicefrac}       
\usepackage{microtype}      
\usepackage{amssymb,amsmath,amsthm}
\usepackage{algpseudocode}
\usepackage{dsfont}
\usepackage[table]{xcolor}
\usepackage{hhline}
\usepackage{makecell}
\usepackage{multirow}
\usepackage{pifont}
\usepackage{graphics}
\usepackage[markup=default]{changes}
\usepackage{appendix}
\usepackage{algorithm}
\usepackage[utf8]{inputenc}
\usepackage[T1]{fontenc}  
\usepackage{hyperref}    
\usepackage{amsfonts}      
\usepackage{amssymb,amsmath,amsthm}
\usepackage{subcaption}
\usepackage{caption}
\newcommand{\fullversion}[2]{#2}
\newcommand{\old}[1]{}
\newcommand{\new}[1]{#1}
\newtheorem{theorem}{Theorem}
\newtheorem{corollary}{Corollary}
\newtheorem{lemma}{Lemma}
\newtheorem{proposition}{Proposition}
\title{Online Submodular Maximization via Online Convex Optimization}

\input{notations_macros}

\begin{document}

\title{Online Submodular Maximization via Online Convex Optimization}

\author{Tareq~Si~Salem\thanks{Northeastern University, Boston, MA, USA \tt{\{sisalem, gozcan, ioannidis\}@ece.neu.edu}}\and
Gözde~Özcan\footnotemark[1]
\and
Iasonas~Nikolaou\thanks{Boston University, Boston, MA, USA, \tt{\{nikolaou, evimaria\}@bu.edu}}
\and
Evimaria~Terzi\footnotemark[2]
\and 
Stratis Ioannidis\footnotemark[1]
}

\begin{titlepage}
\maketitle
\begin{abstract}
    We study monotone submodular maximization under general matroid constraints in the online setting. We prove that online optimization of a large class of submodular functions, namely, weighted threshold potential functions, reduces to online convex optimization (OCO). This is precisely because functions in this class admit a concave relaxation; 
    as a result, OCO policies,
    coupled with an appropriate rounding scheme, can be used to achieve sublinear regret in the combinatorial setting. 
    We  show that  
    our reduction extends to many different versions of the online learning problem, including the dynamic regret, bandit, and optimistic-learning settings.  
    
\end{abstract}
\end{titlepage}
\section{Introduction}\label{sec:intro}

 In  online submodular optimization (OSM)~\cite{onlineassignement}, 
 submodular reward functions 
  chosen by an adversary are revealed over several rounds. In each round, a decision maker first commits to a set satisfying, e.g., matroid constraints. Subsequently, the reward function is revealed and evaluated over this set. 
  The objective is to minimize $\alpha$-regret, i.e., the difference of the cumulative reward accrued from the one attained by a static set, selected by an $\alpha$-approximation algorithm operating in hindsight. OSM  has received considerable interest recently, both in the full information \cite{harvey2020improved,matsuoka2021tracking,onlineassignement,niazadeh2021online} and bandit setting \cite{niazadeh2021online,matsuoka2021tracking, wan2023bandit}, where only the reward values (rather than the entire functions) are revealed.

Online convex optimization (OCO) studies a similar online setting in which reward functions are concave, and decisions are selected from a compact convex set. First proposed by~\citet{zinkevich2003}, who showed that projected gradient ascent attains sublinear regret, OCO generalizes previous online problems like prediction with expert advice~\cite{littlestone94}, and has become widely influential in the learning community~\cite{hazan2016introduction, shalev2012online,mcmahan2017survey}. Its success is evident from the multitude of OCO variants in literature:  in \emph{dynamic regret} OCO~\cite{zinkevich2003, besbes2015non, jadbabaie2015online, mokhtari2016online}, the regret is evaluated w.r.t.~an optimal comparator sequence instead of an optimal static decision in hindsight. \emph{Optimistic}~OCO~\cite{rakhlin2013online,dekel2017online,mohri2016accelerating} takes advantage of benign sequences, in which reward functions are predictable: the decision maker attains tighter regret bounds when predictions are correct, falling back to the existing OCO regret guarantees when predictions are unavailable or inaccurate. \emph{Bandit} OCO algorithms~\cite{hazan2014bandit, kleinberg2004nearly, Flaxman2005} study the aforementioned bandit setting, where again only rewards (i.e., function evaluations) are observed.

We make the following contributions:
\begin{packed_itemize}
    \item We provide a methodology for reducing OSM to OCO, when submodular functions selected by the adversary are bounded from above and below by concave relaxations and coupled with an opportune rounding. We prove that algorithms and regret guarantees in the OCO setting transfer to $\alpha$-regret guarantees in the OSM setting, via a transformation that we introduce. Ratio $\alpha$ is determined by how well concave relaxations approximate the original submodular functions.
    
    \item We show that the above condition is satisfied by a wide class of submodular functions, namely, \emph{weighted threshold potential} (WTP) functions. This class strictly generalizes weighted coverage functions \cite{karimi2017stochastic,stobbe2010efficient}, and includes many important applications, including influence maximization~\cite{kempe2003maximizing}, facility location~\cite{krause2014submodular,frieze1974cost},   cache networks~\cite{ioannidis2016adaptive,li2021online}, similarity caching~\cite{sisalem2022ascent},  demand forecasting~\cite{ito2016large}, and team formation~\cite{li2018learning}, to name a few.
    
    \item We show our reduction also \old{extents}\new{extends} to the dynamic regret and optimistic settings, reducing such full-information OSM settings to the respective OCO variants. To the best of our knowledge, our resulting algorithms are the first to come with guarantees for the dynamic regret and optimistic settings in the context of OSM with general matroid constraints. 
    Finally, we also provide a different reduction for the bandit setting, again for all three (static, dynamic regret, and optimistic) variants; this reduction applies to general submodular functions, but \new{is} restricted to partition matroids. 
\end{packed_itemize}
The remainder of the paper is organized as follows. We present related work and a technical preliminary in Sec.~\ref{sec:related} and Sec.~\ref{s:tech_preliminary}, respectively, and our setting in Sec.~\ref{s:problem}. We state our main results in Sec.~\ref{sec:osmviaoco}, \new{and} discuss extensions in Sec.~\ref{sec:extensions}. Our experimental results are in Sec.~\ref{sec:experiments}; we conclude in Sec.~\ref{sec:conclusions}.

\begin{table*}[!t]
\resizebox{\textwidth}{!}{
\begin{tiny}
\begin{tabular}{||c||c||c|c|c||c|c|c||c|c|c||c||}
\hline
 & \multicolumn{1}{c||}{}  &\multicolumn{9}{c||}{\textbf{$(1-1/e)$-regret (Full Information)}} & \\
\cline{3-11}
\textbf{Paper}  & \textbf{Prob. Class} &\multicolumn{3}{c||}{\textbf{Static}}&\multicolumn{3}{c||}{\textbf{Dynamic}} & \multicolumn{3}{c||}{\textbf{Optimistic}} &  \textbf{Time} \\
\cline{3-11}
&&\textbf{Uni.}&\textbf{Part.}&\textbf{Gen.}&\textbf{Uni.}&\textbf{Part.}&\textbf{Gen.}&\textbf{Uni.}&\textbf{Part.}&\textbf{Gen.}&\\
\hline
\cite{niazadeh2021online} & GS& $r\sqrt{\log\parentheses{n} T}$&\xmark&\xmark&\xmark&\xmark&\xmark&\xmark&\xmark & \xmark & $T^4 O_b$\\
\hline
\cite{harvey2020improved}& GS& \multicolumn{3}{c||}{$ \sqrt{r\log\parentheses{\frac{n}{r}} T}$}&\xmark &\xmark &\xmark&\xmark&\xmark&\xmark&\parbox{2cm}{ \centering $n r^2 + O_{\mathrm m}\cdot n^4/\epsilon^3$ \\ $\textstyle\log({{n^3 T}/{\epsilon}})$ }\\
\hline
\cite{matsuoka2021tracking}  & GS&$r  \sqrt{ r\log(nT) T}$&\xmark&\xmark&$\sqrt{r (r\log(nT) + P_T) T}$&\xmark&\xmark&\xmark & \xmark & \xmark &   ${n r}$\\
\hline
\cite{onlineassignement}  & GS & \multicolumn{2}{c|}{\!\!\!$r^{\frac{3}{2}} \sqrt{ \log(n) T} $ }&\xmark & \xmark & \xmark & \xmark & \xmark & \xmark & \xmark  & ${n^2 c_{\mathrm p} }$\\ 
\hline
\hline
\cite{chen2018online}  & DR-S&\multicolumn{3}{c||}{$\sqrt{rnT}$}& \xmark& \xmark&\xmark&\xmark&\xmark&\xmark& $\sqrt{T} O_{\mathrm {oco}}\cdot O_{\mathrm m}+ n r^2$ \\
\hline
\cite{zhang2019online}  & DR-S&\multicolumn{3}{c||}{$T^{\frac{4}{5}}$ }&  \xmark& \xmark&\xmark&\xmark&\xmark&\xmark& $ T^{\frac{3}{5}} O_{\mathrm{oco}}\cdot O_{\mathrm m}+ n r^2 $ \\
\hline
\cite{zhang2022stochastic}  & DR-S&\multicolumn{3}{c||}{$\sqrt{rnT}$ }& \xmark& \xmark&\xmark&\xmark&\xmark&\xmark& $O_{\mathrm{oco}} \cdot O_{\mathrm m}+ n r^2 $  \\
\hline
\hline
\cite{kakade2007playing}  & LWD&\multicolumn{3}{c||}{$n(\alpha+2) \sqrt{T}$ }&  \xmark& \xmark&\xmark&\xmark&\xmark&\xmark& $ T O_{\mathrm{\alpha}}$ \\
\hline
\hline
\paintbg \parbox{.25cm}{\centering This work}& \parbox{1cm}{\centering\paintbg WTP}& \multicolumn{3}{c||}{\paintbg $r\sqrt{\log(\frac{n}{r}) T}$}  & \multicolumn{3}{c|}{\paintbg $\sqrt{r (r\log(\frac{n}{r}) + \log(n) P_T) T}$}&\multicolumn{3}{c|}{\paintbg \parbox{2.7cm}{\centering $\sqrt{r (r\log(\frac{n}{r}) + \log(n) P_T)}t$\\  $\cdot \sqrt{\displaystyle\sum^T_{t =1} \norm{\g_t - \g_t^\pi}_\infty^2}$} }&\paintbg ${n r^2}$\\
\hline
\end{tabular}
\end{tiny}
}
\caption{Order ($\BigO{\,\cdot\,}$) comparison of regrets and time complexities attained by different online submodular optimization algorithms for general submodular  (GS) and others for continuous DR-submodular  (DR-S) functions. We also include~\citet{kakade2007playing}, who operate on Linearly Weighted Decomposable (LWD) functions, and our work on Weighted Threshold Potential (WTP) functions. 
We also indicate whether algorithms operate over uniform, partition, or general matroid constraints, in the static, bandit, and optimistic settings. Regret and complexity are characterized in terms of the time horizon $T$, the ground set size $n$, and the matroid rank $r$. Additional, algorithm-specific parameters are specified in Appendix~\fullversion{C.VI of \citet{si2023online}}{\ref{appendix:additional_related_work}}, and the derivation of the regret constants in this work is provided in Appendix~\fullversion{C of \citet{si2023online}}{\ref{appendix:setups_omd}}.
The SoTA general submodular+general matroid algorithm \cite{harvey2020improved} has a tighter regret than us by a factor $\sqrt{r}$, but has a much higher computational complexity. We attain the same or \old{or} better regret than the DR-S \cite{chen2018online,zhang2019online,zhang2022stochastic}  and remaining algorithms, that also either operate on restricted constraint sets \cite{niazadeh2021online,matsuoka2021tracking, onlineassignement} or on the much more restrictive LWD class \cite{kakade2007playing}.
Most importantly, our work readily generalizes to the dynamic and optimistic settings. With the sole exception of \citet{matsuoka2021tracking}, who study dynamic regret restricted to uniform matroids,  our work is the first to provide guarantees for online submodular optimization in the dynamic and optimistic settings under general matroid constraints. Leveraging the concave relaxation and eschewing computing the multilinear relaxation also yields significant computational complexity dividends. }
\label{table:compare}
\end{table*}

\section{Related Work}\label{sec:related}

\noindent\textbf{Offline Submodular Maximization and Relaxations of Submodular Functions.} Continuous relaxations of submodular functions play a prominent role in submodular maximization. The so-called continuous greedy algorithm~\cite{calinescu2011maximizing} maximizes 
the multilinear relaxation of a submodular objective over the convex hull of a matroid, using a variant of the Frank-Wolfe algorithm \cite{frank1956algorithm}. The fractional solution is then rounded via pipage \cite{ageev2004pipage} or swap rounding \cite{chekuri2010dependent}, which we also use.  The multilinear relaxation is not convex but is continuous DR-submodular \cite{bach2019submodular,bian2017guaranteed}, and continuous greedy comes with a $1-\frac{1}{e}$ approximation guarantee. However, the multilinear relaxation is generally not tractable and is usually estimated via sampling. 
Ever since the seminal paper by Ageev and Sviridenko~\cite{ageev2004pipage}, several works have exploited the existence of concave relaxations of weighted coverage functions (e.g.,~\cite{karimi2017stochastic,ioannidis2016adaptive}), a strict subset of the threshold potential functions we consider here. For coverage functions, a version of our ``sandwich'' property (Asm.~\ref{assumption:sandwich}) follows directly by the Goemans \& Williamson inequality \cite{goemans1994new}, which we also use.  Both in the standard~\cite{ageev2004pipage} and stochastic offline~\cite{karimi2017stochastic,ioannidis2016adaptive} submodular maximization setting, in which the objective is randomized, exploiting concave relaxations of coverage functions yields significant computational dividends, as it eschews any sampling required for estimating the multilinear relaxation. We depart from both by considering a much broader class than coverage functions and 
studying the online/no-regret setting.

\noindent\textbf{OSM via Regret Minimization.} Several online algorithms have been proposed for maximizing general submodular functions \cite{niazadeh2021online,harvey2020improved,matsuoka2021tracking,onlineassignement} under different matroid constraints. There has also been recent work~\cite{chen2018online,zhang2019online,zhang2022stochastic} on the online maximization of continuous DR-submodular functions~\cite{bach2019submodular}. Proposed algorithms are applicable to our setting, because the multi-linear relaxation is DR-submodular, and guarantees can be extended to matroid constraint sets again through rounding~\cite{chekuri2010dependent}, akin to the approach we follow here. Also pertinent is the work by~\citet{kakade2007playing}: their proposed online algorithm operates over reward functions that can be decomposed as the weighted sum of finitely many (non-parametric) reference functions---termed Linearly Weighted Decomposable (LWD); the adversary selects only the weights. Applied to OSM, this is a more \old{restricive}\new{restrictive} function class than the ones we study. 

We compare these algorithms to our \new{work} in Table~\ref{table:compare}. In the full information setting, the OSM algorithm   by~\citet{harvey2020improved} has a slightly tighter $\alpha$-regret than us, but  \new{also a much} higher computational complexity. We attain the same or better regret than  DR-S \cite{chen2018online,zhang2019online,zhang2022stochastic}  and remaining algorithms that either operate on restricted constraint sets \cite{niazadeh2021online,matsuoka2021tracking, onlineassignement} or on the much more restrictive LWD class~\cite{kakade2007playing}.  Most importantly, our work  generalizes to the dynamic and optimistic settings. With the  exception of \citet{matsuoka2021tracking}, who \new{study} dynamic regret restricted to uniform matroids,  our work provides the first guarantees for OSM in the dynamic and optimistic settings under general matroid constraints.

\noindent\textbf{OSM in the Bandit Setting.} Our reduction to OCO in the bandit setting extends the analysis by~\citet{wan2023bandit}, who provide a reduction to just FTRL in the static setting, under general submodular functions and partition matroid constraints. We generalize this to any OCO algorithm and to the dynamic and optimistic settings. Interestingly, \citet{wan2023bandit} conjecture that no sublinear regret algorithm exists for general submodular functions under general matroid constraints in the bandit setting. We compare to bounds attained by \citet{wan2023bandit} and other bandit algorithms for OSM \cite{niazadeh2021online,onlineassignement,zhang2019online,kakade2007playing} in Table~\fullversion{4}{\ref{table:compare2}} in \fullversion{\cite{si2023online}}{Appendix~\ref{appendix:bandit:setting}}. Our main contribution is again the extension to the dynamic and optimistic settings. 

\section{Technical Preliminary}
\label{s:tech_preliminary}

\noindent\textbf{Submodularity and Matroids.}
   Given a ground set $V=[n]\triangleq \set{1,2,\dots, n}$, 
a set function $f: 2^V \to \reals$ is \emph{submodular} if  $f(S \cup \set{i, j}) - f(S \cup \set{j})  \leq  f(S \cup \set{i}) - f(S)$ for all $S \subseteq V$ and $i,j \in V \setminus S$ and \emph{monotone} if    $f(A) \leq f(B)$  for all $A \subseteq B \subseteq 2^V$. 
A \emph{matroid} is a pair $\M = (V, \I)$, where  $\I \subseteq  2 ^{V}$, 
for which the following holds:  (1) if $B \in \I$ and $A \subseteq B$, then $A \in \I$, 
(2) if $A, B \in \I$ and $\card{A} < \card{B}$, then there \old{exits}\new{exists} an $x \in B \setminus A$ s.t. $A \cup \set{x} \in \I$. 
The \emph{rank} $r \in \naturals$ of $\M$ is the cardinality of the largest set in $\I$. 
 With slight abuse of notation, we  represent set functions $f: 2^V \to \reals$  as functions over $\{0,1\}^n$: given a set function $f$ and an $\x\in\{0,1\}^n$, we denote by $f(\x)$ as the value $f(\supp{\x})$,  where  $\supp{\x} \triangleq \set{i \in V: x_i\neq 0} \subseteq V$ is the support of $\x$. Similarly, we treat matroids as subsets of  $\{0,1\}^n$.  



 \noindent\textbf{Online Learning.} \label{s:online_learning}
 In the general protocol of \emph{online learning} \cite{cesabianchi2006}, a decision-maker makes sequential decisions and incurs rewards as follows: 
 at timeslot $t \in [T]$, where $T\in \naturals$ is the time horizon, the {decision-maker} first commits to a decision $\x_t \in \X$ from some set $\X$. Then, a reward function  $f_t:\X\to\reals_{\geq 0}$ is selected by an adversary from a set of functions $\F$ and revealed to the decision-maker, who accrues a reward $f_t(\x_t)$.  The decision $\x_t$ is determined  according to a (potentially random) mapping  $\P_{\X,t}: \X^{t-1} \times  \F^{t-1}\to \Y$, i.e., 
 \begin{align}\label{eq:onlineprotocol}
     \x_{t} = \P_{\X, t} \parentheses{\parentheses{\x_s}^{t-1}_{s=1} , \parentheses{f_s}^{t-1}_{s=1}} .
 \end{align}
Let $\vec\P_\X = (\P_{\X,t})_{t \in \T}$ be the \emph{online policy} of the decision-maker. 
We define the \emph{regret} of  $\vec\P_\X$ at horizon $T$ as:
\begin{equation}
    \begin{split}
        \mathrm{regret}_T ({\vec\P}_\X)  
     \triangleq  \sup_{({f}_t)_{t=1}^{T}  \in {\F}^T}  \bigg\{ \max_{\x\in \X}  \sum^T_{t=1} {f}_t(\x) -\underset{{{\vec\P}_\X}}{\E}\left[\sum^T_{t=1}{f}_t(\x_t)\right]\bigg\} \label{eq:general_regret}.
    \end{split}
\end{equation}
We seek policies that attain a sublinear (i.e., $o(T)$) regret; intuitively, such policies perform on average as
well as the static optimum in hindsight.
Note that the regret is defined w.r.t.~the optimal \emph{fixed} decision $\x$, i.e., the time-invariant decision $\x\in \X$ that would be optimal in hindsight, after the sequence $(f_t)_{t=1}^{T}$ is revealed. When selecting  $\x_t$, the decision-maker has no information about the upcoming reward $f_t$. Finally, this is the \emph{full-information} setting: at each timeslot, the decision maker observes the entire reward function $f_t(\,\cdot\,)$, rather than just the reward $f_t(\x_t)\in \reals_{\geq 0}$.

Deviating from these assumptions is of both practical and theoretical interest. 
In the \emph{dynamic regret} setting \cite{zinkevich2003}, the regret is measured w.r.t. a time-variant optimum, appropriately constrained so that changes from one timeslot to the next do not vary significantly. 
In  \emph{learning with optimism}  \cite{rakhlin2013online}, additional information is assumed to be available w.r.t. $f_t$, in the form of so-called predictions. In the \emph{bandit} setting \cite{auer1995gambling}, the online policy ${\vec\P}_\X$ of the decision maker only has access to rewards $f_t(\x_t)\in \reals_{\geq 0}$, as opposed to the entire reward function.  

\paragraph{Online Convex Optimization.} The \emph{online convex optimization} (OCO) framework~\cite{zinkevich2003, hazan2016introduction} follows the above online learning protocol~\eqref{eq:onlineprotocol},  where (a) the decision space $\X$ is a  convex set in $\reals^n$, and (b) the set of reward functions $\F$ is a subset of concave functions over $\X$.  
Formally, OCO operates under the following  assumption:
\begin{assumption}
    Set  $\X \subset \reals^n$ is  \new{convex and compact}. 
     The reward functions in $\F$ are all $L$-Lipschitz concave functions w.r.t. a norm $\norm{\,\cdot\,}$ over $\X$, for some common $L\in \reals_{>0}$. 
     \label{asm:oco}
\end{assumption}
There is a  rich literature on OCO policies \cite{hazan2016introduction}; examples  include Online Gradient Ascent (OGA)~\cite{hazan2016introduction}, Online Mirror Ascent (OMA)~\cite{bubeck2011introduction}, and Follow-The-Regularized-Leader (FTRL)~\cite{mcmahan2017survey}. All three enjoy sublinear regret: 
\begin{theorem} Under Asm.~\ref{asm:oco} OGA, OMA,  and FTRL attain $O(\sqrt{T})$ regret. \label{thm:oco}
\end{theorem}
\new{Details on all three algorithms and the regret they attain are} in \fullversion{Appendix~A of \citet{si2023online}}{Appendix~\ref{appendix:oco}}. 
Most importantly, the OCO framework generalizes to 
the dynamic,  learning-with-optimism, as well as  bandit settings (see  Sec.~\ref{sec:extensions}). 

\noindent\textbf{Weighted Threshold Potentials.} A \emph{threshold potential}~\cite{stobbe2010efficient} $\Psi_{b, \vec{w}, S}: \set{0,1}^n \to \reals_{\geq 0}$, also known as a \emph{budget-additive function} \cite{andelman2004auctions, dobzinski2016breaking,buchfuhrer2010inapproximability}, is defined as: 
\begin{align}
      \Psi_{b,\vec w, S}(\x)&\triangleq  \textstyle \min\set{b, \sum_{j \in S} x_j w_j},\,\text{for $\x \in \set{0,1}^n$}, \label{eq:budget-additive}
\end{align}
where   $b \in \reals_{\geq 0} \cup \set{\infty}$ is a threshold,  $S \subseteq V$ is a subset of $V=[n]$, and  $\vec w = (w_j)_{j \in S}\in [0,b]^{|S|}$ is a weight vector bounded by $b$.\footnote{Assumption $ w_{j} \leq b$, $j \in S$, is w.l.o.g., as  replacing $w_{j}$ with $\min\set{w_{j}, b}$ preserves  values of $f$ over  $\set{0,1}^n$.}  The linear combination of threshold potential functions defines the rich class of \emph{weighted threshold potentials} (WTP)~\cite{stobbe2010efficient}, 
defined as: 
\begin{align}
   f(\x)  &\textstyle \triangleq \sum_{\ell \in C} c_{\ell} \Psi_{b_\ell,\vec w_\ell, S_\ell} (\x), \quad \text{for $\x \in \set{0,1}^n$},
   \label{eq:wtp}
\end{align}
where $C$ is an arbitrary index set and $c_{\ell} \in \reals_{\geq 0}$, for $\ell\in C$.  WTP functions are submodular and monotone (see Appendix~\fullversion{B of \citet{si2023online}}{\ref{appendix:WTP}}). 
We define the \emph{degree} of a WTP function $\Delta_f = \max_{\ell \in C} |S_\ell|$ as the maximum number of variables that a threshold potential $\Psi$ in $f$ depends on.

We give several examples of WTP functions in \fullversion{Appendix~B of \citet{si2023online}}{Appendix~\ref{appendix:WTP}}. In short, classic problems such as 
 influence maximization~\cite{kempe2003maximizing} and facility location~\cite{krause2014submodular,frieze1974cost},  resource allocation problems like cache networks~\cite{ioannidis2016adaptive,li2021online} and similarity caching~\cite{sisalem2022ascent}, as well as demand forecasting~\cite{ito2016large} and team formation~\cite{li2018learning}  can all be expressed using WTP functions. 
 Overall, the WTP class is very broad: there exists a hierarchy among submodular functions, including weighted coverage functions~\cite{karimi2017stochastic}, weighted cardinality truncations~\cite{dolhansky2016deep}, and sums of concave functions composed with non-negative 
modular functions~\cite{stobbe2010efficient}; all of them are strictly dominated by the WTP class (see~\citet{stobbe2010efficient}, as well as \fullversion{Appendix~B of \citet{si2023online}}{Appendix~\ref{appendix:WTP}}).

\section{Problem Formulation}
\label{s:problem}
We consider a combinatorial version of the online learning protocol defined in Eq.~\eqref{eq:onlineprotocol}.  In particular, we focus on the case where (a) the decision set is $\X \subseteq \set{0,1}^n$, i.e., the vectorized representation of subsets of $V=[n]$, and (b) the set $\F$ of reward functions comprises set functions over $\X$. Though some of our results (e.g., Thm.~\ref{theorem:sandwich}) pertain to this general combinatorial setting, we are particularly interested in the case where (a)  $\X$ is a matroid, and (b) $\F$ is the WTP class, i.e., the set of functions whose form is given by Eq.~\eqref{eq:wtp}.

Both in the general combinatorial setting and for WTP functions,  evaluating  the best fixed decision  may be computationally intractable even in hindsight, i.e., when all reward functions were revealed.  
 As is customary (see, e.g.,~\cite{krause2014submodular}), instead of the  regret in Eq.~\eqref{eq:general_regret}, we consider the so-called $\alpha$-regret:
\begin{equation}
    \begin{split}
    \alpha\text{-}\mathrm{regret}_T ({\vec\P}_\X) \triangleq 
    \sup_{\parentheses{f_t}_{t=1}^{T}  \in {\F}^T} \bigg\{ \max_{\x\in \X}  \alpha\sum^T_{t=1} {f}_t(\x) 
     -\underset{{{\vec\P}_\X}}{\E} \interval{\sum^T_{t=1} {f}_t(\x_t)}\bigg\} 
    \label{eq:alpha_regret}.
    \end{split}
\end{equation}
 Intuitively, we compare the performance of the policy $\vec \P_{\X}$ w.r.t. the best {polytime}$(n)$ $\alpha$-approximation of the static offline optimum in hindsight. For example, the approximation ratio would be $\alpha=1-1/e$ in the case of submodular set functions maximized over matroids. 

\section{Online Submodular Optimization via Online Convex Optimization}
\label{sec:osmviaoco}

\subsection{The Case of General Set Functions}
\label{s:osm_genral}
\begin{algorithm}[t]

\caption{Rounding-Augmented OCO (RAOCO) policy}
\begin{algorithmic}[1]
\begin{footnotesize}
\Require OCO policy ${\vec\P}_{\Y}$, randomized rounding $\Xi:\Y \to \X$
\For {$t = 1,2,\dots, T$}
    \State $\y_t \gets \P_{\Y, t}\parentheses{ (\y_s)^{t-1}_{s=1}, (\cfunc_s)^{t-1}_{s=1}}$
\State $\x_t \gets \Xi (\y_t)$
\State Receive reward $f_t(\x_t)$
\State Reward function $f_t$  is revealed 
\State Construct $\cfunc_t$ from $f_t$ satisfying Asm.~\ref{assumption:sandwich}
\EndFor
\end{footnotesize}
\end{algorithmic}
\label{alg:saoco}
\end{algorithm}
First, we show how OCO can be leveraged to tackle online learning in the general combinatorial setting, i.e., when $\X\subseteq \{0,1\}^n$ and $\F$ comprises general functions defined over $\X$. 

\paragraph{Rounding Augmented OCO Policy.} 
We begin by stating a ``sandwich'' property that functions in $\F$ should satisfy, so that the reduction to OCO holds. To do so, we first need to introduce the notion of randomized rounding. Let  $\Y \triangleq \conv\X$ be the convex hull of $\X$. A \emph{randomized rounding}  is a random map ${\Xi:\Y \to \X}$, i.e., a map  from a fractional  $\y \in \Y$ and, possibly, a source of randomness to an integral variable $\x \in \X$.  We  assume that the set $\F$ satisfies the following:
\begin{assumption} (Sandwich Property)\label{assumption:sandwich}
There exists an $\alpha \in (0,1]$, an $L\in \reals_{>0}$, and a randomized rounding  $\Xi:\Y \to \X$ such that, for every $f : \X \to \reals_{\geq 0} \in \F$ there exists a $L$-Lipschitz concave function $\cfunc : \Y \to \reals$ s.t.   
\begin{align}
    \cfunc(\x) &\geq f(\x), \quad\text{for all $ \x \in \X$, and }\label{eq:alpha_approx_upper}\\
    \E_{\Xi} \interval{f(\Xi(\y))} &\geq \alpha \cdot \cfunc(\y) ,\quad \text{ for all}~\y \in \Y. 
    \label{eq:alpha_approx_lower}
\end{align}
\end{assumption}
We refer to $\cfunc$ as the \emph{concave relaxation} of $f$. Intuitively, Asm.~\ref{assumption:sandwich} postulates the existence of such a concave relaxation $\cfunc$ that is not ``far'' from $f$: Eqs.~\eqref{eq:alpha_approx_upper} and~\eqref{eq:alpha_approx_lower} imply that $\cfunc$ bounds $f$ both from above and below, up to the approximation factor $\alpha$. Moreover, the upper bound (Eq.~\eqref{eq:alpha_approx_upper}) needs to only hold for integral values, while the lower bound (Eq.~\eqref{eq:alpha_approx_lower}) needs to only hold in expectation, under an appropriately-defined randomized rounding $\Xi$. 

Armed with this assumption, we can convert any  OCO policy $\vec \P_\Y$ operating over  $\Y=\conv\X$ to a   \emph{Randomized-rounding Augmented OCO} (RAOCO) policy, denoted by $\vec\P_\X$, operating over $\X$. This transformation (see Alg.~\ref{alg:saoco}) uses both the randomized-rounding  $\Xi$, as well as the concave relaxations $(\cfunc_s)_{s=1}^{t-1}$ of the functions $(f_s)_{s=1}^{t-1}$ observed so far. 
At $t \in \T$, the RAOCO policy amounts to:
\begin{subequations}
\label{eq:raoco}
\begin{align}
\y_t &=\P_{\Y, t}\parentheses{ (\y_s)^{t-1}_{s=1}, (\cfunc_s)^{t-1}_{s=1}}, \\
\x_t &=\Xi(\y_t)\in \X.
\end{align}
\end{subequations}
In short,  the OCO policy $\P_{\Y}$ is first  used to generate a new fractional state  $\y_t \in \Y$ by applying $\P_{\Y, t}$ to the  history of concave relaxations.
Then, this fractional decision $\y_t$ is randomly mapped to an integral decision $\x_t\in\X$ according to the rounding scheme $\Xi$. Then, the reward $f(\x_t)$ is received and $f_t$ is revealed, at which point a concave function $\cfunc_t$ is constructed from $f_t$ and added to the history. 
Our first main result is the following: 
 
%

\begin{theorem}
Under Asm.~\ref{assumption:sandwich}, given an OCO policy $\vec\P_\Y$,  the RAOCO policy $\vec\P_\X$ described by Alg.~\ref{alg:saoco} satisfies
   $ \aregret_T \parentheses{\vec\P_\X} \leq \alpha \cdot \regret_T \parentheses{{\vec{\P}}_{\Y}}.$
\label{theorem:sandwich}
\end{theorem}
The proof is in \fullversion{Appendix~D of \citet{si2023online}}{Appendix~\ref{appendix:sandwich}}. As a result, \emph{any regret guarantee obtained by an OCO algorithm over $\Y$, immediately transfers to an $\alpha$-regret for RAOCO, where $\alpha$ is determined by Asm.~\ref{assumption:sandwich}.}
In particular, the decision set $\Y$ is closed, bounded, and convex by construction. Combined with the fact that concave relaxations $\cfunc$ are $L$-Lipschitz (by Asm.~\ref{assumption:sandwich}), Thms.~\ref{thm:oco} and~\ref{theorem:sandwich} yield the following  corollary:
\begin{corollary}\label{cor:raoco}
Under Asm.~\ref{assumption:sandwich},  RAOCO policy $\vec\P_\X$  in  Alg.~\ref{alg:saoco} equipped with OGA, OMA, or FTRL as OCO policy $\vec\P_\Y$  has sublinear $\alpha$-regret. That is, 
     $   \aregret_T\parentheses{\vec\P_\X} = \BigO{\sqrt{T}}.$
\end{corollary}
\new{To use  this result}, Asm.~\ref{assumption:sandwich} \new{should} hold, and both the randomized rounding and the concave relaxations used in RAOCO \new{should be poly-time}: \new{all are} true for WTP functions optimized over matroid constraints, which we turn to next.

\subsection{The Case of Weighted Threshold Potentials} \label{s:osm_wtp}
We now consider the case where the decision set is a matroid, and reward functions belong to the class of WTP functions, defined by Eq.~\eqref{eq:wtp}. We will show that, under appropriate definitions of a randomized rounding and concave relaxations, the class $\F$ satisfies Asm.~\ref{assumption:sandwich} and, thus, online learning via RAOCO comes with the regret guarantees of Corollary~\ref{cor:raoco}.
For $\Y=\conv{\X}$, consider the map $f\mapsto \cfunc$ of WTP functions $f:\X\to\reals$ to concave relaxations $\cfunc:\Y\to\reals$ of the form:
\begin{equation}
    \begin{split}
        \cfunc (\y) \triangleq f(\y) =\textstyle\sum_{\ell \in C} c_{\ell} \min\set{b_\ell, \sum_{j \in S_\ell} y_j w_{\ell ,j}}, 
     \label{eq:budget-additive-concave}
    \end{split}
\end{equation}
for $\y\in \Y$.  In other words, \emph{the relaxation of $f$ is itself}: it has the same functional form, allowing integral variables to become fractional.\footnote{In Appendix~\fullversion{G.IV of \citet{si2023online}}{\ref{appendix:alternate_relaxation}}, \new{we provide an example where the functional form of concave relaxations differs}.} This is clearly concave, as the minimum of affine functions is concave, and the positively weighted sum of concave functions is concave. 
Finally, all such functions are Lipschitz, with a parameter that depends on $c_\ell$, $b_\ell$, $\vec w_\ell$, $\ell\in C$.
 Let $\cfuncSet$ be the image of $\F$ under the map \eqref{eq:budget-additive-concave}.
We make the following assumption, which is readily satisfied if, e.g., all constituent parameters ($c_\ell$, $b_\ell$, $\vec w_\ell$, $\ell\in C$) are uniformly bounded, or the set $\F$ is finite, etc.: 
\begin{assumption}\label{assumption:lip}
    There exists an $L>0$ such that all functions in $\cfuncSet$ are  $L$-Lipschitz.
\end{assumption}


Next, we turn our attention to the randomized rounding $\Xi$. We can in fact characterize the property that $\Xi$ must satisfy for Asm.~\ref{assumption:sandwich} to hold for relaxations given by Eq.~ \eqref{eq:budget-additive-concave}:
\begin{definition}
A randomized rounding  $\Xi:\Y \to \X$ is \emph{negatively correlated} if,  for $\x = \Xi(\y) \in \X$ (a) the coordinates of $\x$ are negatively correlated~\footnote{A set of random variables $x_i\in\set{0,1}, i \in [n]$, are  negatively correlated if  $\E\interval{\prod_{i \in S} x_i }\leq \prod_{i \in S} \E\interval{x_i}$ for all $S \subseteq [n]$.} 
and (b)  $\E_\Xi\interval{\x} = \y$.\label{def:negatively_correlated_sampler}
\end{definition}
Our next result immediately implies that \emph{any negatively correlated rounding can be used in RAOCO}: 
\begin{lemma}\label{proposition:sandwich2}
Let $\Xi:{\Y} \to \X$ be a negatively correlated randomized rounding, and consider the concave relaxations $\cfunc$ constructed from $f\in\F$ via Eq.~\eqref{eq:budget-additive-concave}. Then, if Asm.~\ref{assumption:lip} holds for some $L>0$, the set $\F$ satisfies Asm.~\ref{assumption:sandwich} with 
 $\alpha = \parentheses{1 - \frac{1}{\Delta}}^\Delta$ where $\Delta = \sup_{f\in \F}\Delta_f$.
\end{lemma}
The proof is in \fullversion{Appendix~E of \citet{si2023online}}{Appendix~\ref{appendix:sandwich2}}. As $\Delta\to\infty$, the approximation ratio $\alpha$ approaches $1-1/e$ from above, recovering the usual approximation guarantee. However, for finite $\Delta$, we in fact obtain an improved approximation ratio; an example (quadratic submodular functions) is described in \fullversion{Appendix~B of \citet{si2023online}}{Appendix~\ref{appendix:WTP}}. 
%
%
%
%
%
%
Finally, and most importantly, \emph{a negatively-correlated randomized rounding can always be constructed if $\X$ is  \new{a} matroid}. 
\citet{chekuri2010dependent} provide two polynomial-time randomized rounding algorithms that satisfy this property: 
\begin{lemma} (\citet[Theorem~1.1.]{chekuri2010dependent})\label{theorem:rounding}
    Given a matroid $\X\subset\{0,1\}^n$, let $\y \in \conv{\X}$  and $\Xi$ be either  \emph{swap rounding} or \emph{randomized pipage rounding}. Then, $\Xi$ is negatively correlated.
\end{lemma}
We review swap rounding in \fullversion{Appendix~F of \citet{si2023online}}{Appendix~\ref{appendix:swap}}. Interestingly, the existence of a negatively correlated rounding is inherently linked to  matroids: a negatively-correlated rounding exists \emph{if and only if~$\X$ is a matroid} (see Thm.~I.1. in \citet{chekuri2010dependent}).  Lemma~\ref{proposition:sandwich2} thus implies that the reduction of RAOCO to OCO policies is also  linked to matroids.
Putting everything together, Lemmas~\ref{proposition:sandwich2}--\ref{theorem:rounding} and Corollary~\ref{cor:raoco} yield the following:
\begin{theorem}\label{thm:main2}
Let $\X\!\subseteq\!\{0,1\}^n$ be a matroid, and  $\F$ be a subset of the WTP class for which Asm.~\ref{assumption:lip} holds. Then, the RAOCO policy $\vec \P_\X$ defined by Alg.~\ref{alg:saoco} using swap rounding or randomized pipage rounding as $\Xi$ and OGA, OMA, or FTRL as OCO policy $\vec\P_\Y$, and  concave relaxations in Eq.~\eqref{eq:budget-additive-concave}  
 has sublinear $\alpha$-regret. In particular, 
     $   \aregret\parentheses{\vec\P_\X} = \BigO{\sqrt{T}}.$
\end{theorem}
Note that, though all algorithms yield $O(\sqrt{T})$ regret, the dependence of constants on problem parameters (such as $n$ and the matroid rank $r$), as reported in Table~\ref{table:compare}, is optimized under OMA (see Appendix~\fullversion{C of \citet{si2023online}}{\ref{appendix:setups_omd}}). 

\subsection{Computational Complexity}

\paragraph{OCO Policy.} OCO policies are polytime (see, e.g.,~\cite{hazan2016introduction}). Taking gradient-based OCO policies (e.g., OMA in {Alg.~\fullversion{2 in \citet{si2023online}}{\ref{alg:oma} in the Appendix}}), their computational complexity is  dominated by a projection operation to the convex set $\Y=\conv\X$. 
The exact time complexity of this projection depends on  $\Y$, however, given a membership oracle that decides $\y \in \conv\X$, the projection can be computed efficiently in polynomial time~\cite{hazan2016introduction}. Moreover, the projection problem is a convex problem that can be computed efficiently (e.g., iteratively to an arbitrary precision) and can also be computed in strongly polynomial time~\cite[Theorem~3]{gupta2016solving}.

\begin{table*}[!t]

    \centering
    \resizebox{\textwidth}{!}{%
    \begin{tabular}{|c|c|c|c||c|c|c||c|c|c|c||c|c|c|c||c|c|c|c||c|c|}
    \cline{5-21}
         \multicolumn{4}{c}{} & \multicolumn{3}{|c||}{\RAOCOOGA{}} & \multicolumn{4}{c||}{\RAOCOOMA{}} & \multicolumn{4}{c||}{$\texttt{FSF}^\star$} & \multicolumn{4}{c||}{\texttt{TabularGreedy}} & \multicolumn{2}{c|}{\texttt{Random}} \\
    \hline
        \multicolumn{2}{|c|}{\makecell{Datasets \\ \& Constr.}} & $F^\star$ &$t$ & $\bar{F}_{\X}/F^\star$ & std. dev. & $\eta$ & $\bar{F}_{\X}/F^\star$ & std. dev. & $\eta$ & $\gamma$ & $\bar{F}_{\X}/F^\star$ & std. dev. & $\eta$ & $\gamma$ & $\bar{F}_{\X}/F^\star$ & std. dev. & $\eta$ & $c_{\mathrm p}$ & $\bar{F}_{\X}/F^\star$ & std. dev.  \\
    \hline
         \multirow{6}{*}{\rotatebox{90}{\IMZKC}} & \multirow{3}{*}{\rotatebox{90}{uniform}} & \multirow{3}{*}{$0.234$} & $33$ & $0.902$ & $1.85\times 10^{-2}$ & \multirow{3}{*}{$2.5$} & $\mathbf{0.965}$ & $6.02\times 10^{-3}$ & \multirow{3}{*}{$10$} & \multirow{3}{*}{$0.05$} & $0.839$ & $5.02\times 10^{-3}$ & \multirow{3}{*}{$75$} & \multirow{3}{*}{$0.0$} & $0.833$ & $8.12\times 10^{-3}$ & \multirow{3}{*}{$160$} & \multirow{3}{*}{$1$} & $0.642$ & $3.03 \times 10^{-2}$\\
    \cline{4-6}\cline{8-9}\cline{12-13}\cline{16-17}\cline{20-21}
                                                 &                              &                          & $66$ & $0.924$ & $1.60\times 10^{-2}$ &  & $\mathbf{0.967}$ & $5.65\times 10^{-3}$ & & & $0.896$ & $3.83\times 10^{-3}$ &  & & $0.894$ & $2.55\times 10^{-3}$ & & & $0.624$ & $2.80 \times 10^{-2}$\\
    \cline{4-6}\cline{8-9}\cline{12-13}\cline{16-17}\cline{20-21}
                                                 &                              &                         & $99$ & $0.945$ & $7.70\times 10^{-3}$ &  & $\mathbf{0.982}$ & $5.28\times 10^{-3}$ & & & $0.933$ & $4.17\times 10^{-3}$ &  & & $0.931$ & $1.47\times 10^{-3}$ & & & $0.622$ & $1.91 \times 10^{-2}$\\
    \cline{2-21}
                                                 & \multirow{3}{*}{\rotatebox{90}{partition}}                                      & \multirow{3}{*}{$0.83$} & $33$ & $0.994$ & $2.52\times 10^{-3}$ & \multirow{3}{*}{$8$} & $\mathbf{0.997}$ & $6.95\times 10^{-4}$ & \multirow{3}{*}{$10$} & \multirow{3}{*}{$0.1$}  & \multicolumn{4}{c||}{\multirow{3}{*}{\xmark}} & $0.985$ & $5.65\times 10^{-3}$ & \multirow{3}{*}{$10$} & \multirow{3}{*}{$1$}& $0.953$ & $4.88 \times 10^{-3}$\\
    \cline{4-6}\cline{8-9}\cline{16-17}\cline{20-21}
                                 & & & $66$ & $0.99$ & $8.75\times 10^{-4}$ &  & $\mathbf{0.994}$ & $3.48\times 10^{-4}$ & & & \multicolumn{4}{c||}{} & $0.987$ & $2.93\times 10^{-3}$ & & & $0.950$ & $2.87 \times 10^{-3}$\\
    \cline{4-6}\cline{8-9}\cline{16-17}\cline{20-21}
                                 & & & $99$ & $0.993$ & $9.66\times 10^{-4}$ &  & $\mathbf{0.997}$ & $3.37\times 10^{-4}$ & & &\multicolumn{4}{c||}{}& $0.995$ & $2.70\times 10^{-3}$ & & & $0.953$ & $1.77 \times 10^{-3}$\\
    \hline
         \multirow{6}{*}{\rotatebox{90}{\IMEpinions}} & \multirow{3}{*}{\rotatebox{90}{uniform}} & \multirow{3}{*}{$0.171$} & $50$ & $0.845$ & $1.61\times 10^{-2}$ & \multirow{3}{*}{$4$} & $\mathbf{0.853}$ & $2.46\times 10^{-2}$ & \multirow{3}{*}{$10$} & \multirow{3}{*}{$0.01$} & $0.703$ & $6.03\times 10^{-2}$ & \multirow{3}{*}{$75$} & \multirow{3}{*}{$0.0$} & $0.694$ & $3.44\times 10^{-2}$ & \multirow{3}{*}{$160$} & \multirow{3}{*}{$1$} & $0.632$ & $2.99 \times 10^{-2}$\\
    \cline{4-6}\cline{8-9}\cline{12-13}\cline{16-17}\cline{20-21}
                & & & $100$ & $0.865$ & $1.13\times 10^{-2}$ & & $\mathbf{0.906}$ & $1.25\times 10^{-2}$ & & & $0.776$ & $1.28\times 10^{-2}$ &  & & $0.768$ & $2.87\times 10^{-2}$ & & & $0.615$ & $2.32 \times 10^{-2}$\\
    \cline{4-6}\cline{8-9}\cline{12-13}\cline{16-17}\cline{20-21}
                & & & $149$ & $0.88$ & $7.97\times 10^{-3}$ & & $\mathbf{0.925}$ & $9.09\times 10^{-3}$ & & & $0.807$ & $1.52\times 10^{-2}$ &  & & $0.805$ & $2.48\times 10^{-2}$ & & & $0.629$ & $2.36 \times 10^{-2}$\\
    \cline{2-21}
                & \multirow{3}{*}{\rotatebox{90}{partition}} & \multirow{3}{*}{$0.171$} & $50$ & $0.826$ & $1.35\times 10^{-2}$ & \multirow{3}{*}{$3$} & $\mathbf{0.861}$ & $1.82\times 10^{-2}$ & \multirow{3}{*}{$10$} & \multirow{3}{*}{$0.001$} & \multicolumn{4}{c||}{\multirow{3}{*}{\xmark}} & $0.720$ & $2.28\times 10^{-2}$ & \multirow{3}{*}{$160$} & \multirow{3}{*}{$1$} & $0.620$ & $6.08 \times 10^{-3}$\\
    \cline{4-6}\cline{8-9}\cline{16-17}\cline{20-21}
                                 & & & $100$ & $0.854$ & $6.73\times 10^{-3}$ & & $\mathbf{0.908}$ & $9.42\times 10^{-3}$ & & & \multicolumn{4}{c||}{} & $0.786$ & $1.27\times 10^{-2}$ & & & $0.619$ & $4.70 \times 10^{-3}$\\
    \cline{4-6}\cline{8-9}\cline{16-17}\cline{20-21}
                                 & & & $149$ & $0.88$ & $2.40\times 10^{-3}$ &  & $\mathbf{0.927}$ & $6.34\times 10^{-3}$ & & & \multicolumn{4}{c||}{} & $0.818$ & $1.3\times 10^{-2}$ & & & $0.625$ & $1.10 \times 10^{-2}$\\
    \hline
         \multirow{6}{*}{\rotatebox{90}{\FLMovieLens{}}} & \multirow{3}{*}{\rotatebox{90}{uniform}} & \multirow{3}{*}{$0.407$} & $98$ & $0.749$ & $3.21\times 10^{-2}$ & \multirow{3}{*}{$0.5$} & $\mathbf{0.792}$ & $2.60\times 10^{-2}$ & \multirow{3}{*}{$1.0$} & \multirow{3}{*}{$0.05$} & $0.681$ & $8.41\times 10^{-2}$ & \multirow{3}{*}{$1.0$} & \multirow{3}{*}{$0.001$} & $0.69$ & $1.11\times 10^{-2}$ & \multirow{3}{*}{$160$} & \multirow{3}{*}{$1$} & $0.748$ & $5.26 \times 10^{-2}$\\
    \cline{4-6}\cline{8-9}\cline{12-13}\cline{16-17}\cline{20-21}
                                 & & & $196$ & $\mathbf{0.786}$ & $3.82\times 10^{-2}$ &  & $0.781$ & $1.36\times 10^{-2}$ & & & $0.713$ & $7.44\times 10^{-2}$ & & & $0.676$ & $1.03\times 10^{-2}$ & & & $0.7$ & $4.76 \times 10^{-2}$\\
    \cline{4-6}\cline{8-9}\cline{12-13}\cline{16-17}\cline{20-21}
                                 & & & $293$ & $0.846$ & $3.98\times 10^{-2}$ &  & $\mathbf{0.866}$ & $1.02\times 10^{-2}$ & & & $0.756$ & $6.42\times 10^{-2}$ &  & & $0.769$ & $1.65\times 10^{-2}$ & & & $0.711$ & $3.11 \times 10^{-2}$\\
    \cline{2-21}
                                 & \multirow{3}{*}{\rotatebox{90}{partition}} & \multirow{3}{*}{$0.419$} & $98$ & $0.889$ & $5.20\times 10^{-2}$ & \multirow{3}{*}{$0.5$} & $\mathbf{0.948}$ & $9.70\times 10^{-3}$ & \multirow{3}{*}{$10$} & \multirow{3}{*}{$0.001$} & \multicolumn{4}{c||}{\multirow{3}{*}{\xmark}} & $0.908$ & $2.39\times 10^{-2}$ & \multirow{3}{*}{$160$} & \multirow{3}{*}{$8$} & $0.829$ & $5.54 \times 10^{-2}$\\
    \cline{4-6}\cline{8-9}\cline{16-17}\cline{20-21}
                                 & & & $196$ & $0.872$ & $2.12\times 10^{-2}$ &  & $\mathbf{0.908}$ & $5.14\times 10^{-3}$ & & & \multicolumn{4}{c||}{} & $0.902$ & $1.98\times 10^{-2}$ & & & $0.814$ & $1.51 \times 10^{-2}$\\
    \cline{4-6}\cline{8-9}\cline{16-17}\cline{20-21}
                                 & & & $293$ & $0.926$ & $2.34\times 10^{-2}$ &  & $0.948$ & $3.13\times 10^{-3}$ & & & \multicolumn{4}{c||}{} & $\mathbf{0.964}$ & $2.01\times 10^{-2}$ & & & $0.874$ & $2.84 \times 10^{-3}$\\
    \hline
         \multirow{6}{*}{\rotatebox{90}{\TeamFormation}} & \multirow{3}{*}{\rotatebox{90}{uniform}} & \multirow{3}{*}{$200$}  & $33$ & $0.984$ & $7.05\times 10^{-3}$ & \multirow{3}{*}{$4$} & $\mathbf{0.987}$ & $4.03 \times 10^{-3}$ & \multirow{3}{*}{$0.05$} & \multirow{3}{*}{$0.1$} & $0.845$ & $4.73 \times 10^{-2}$ & \multirow{3}{*}{$1$} & \multirow{3}{*}{$0$} & $0.844$ & $2.21\times 10^{-2}$ & \multirow{3}{*}{$1$} & \multirow{3}{*}{$2$} & $0.605$ & $3.13\times 10^{-2}$\\
    \cline{4-6}\cline{8-9}\cline{12-13}\cline{16-17}\cline{20-21}
                 & & & $66$ & $0.994$ & $3.87 \times 10^{-3}$ &  & $\mathbf{0.994}$ & $2.83 \times 10^{-3}$ & & & $0.868$ & $2.14\times 10^{-2}$ & & & $0.886$ & $1.88\times 10^{-2}$ & & & $0.603$ & $2.01\times 10^{-2}$\\
    \cline{4-6}\cline{8-9}\cline{12-13}\cline{16-17}\cline{20-21}
                 & & & $99$ & $0.995$ & $2.01 \times 10^{-3}$ &  & $\mathbf{0.998}$ & $2.55 \times 10^{-3}$ & & & $0.869$ & $2.95\times 10^{-2}$ & & & $0.902$ & $2.28\times 10^{-2}$ & & & $0.612$ & $1.18\times 10^{-2}$\\
    \cline{2-21}
                 & \multirow{3}{*}{\rotatebox{90}{partition}} & \multirow{3}{*}{$400$} & $33$ & $0.98$ & $3.48 \times 10^{-3}$ & \multirow{3}{*}{$3$} & $\mathbf{0.983}$ & $2.33 \times 10^{-3}$ & \multirow{3}{*}{$0.1$} & \multirow{3}{*}{$0.001$} & \multicolumn{4}{c||}{\multirow{3}{*}{\xmark}} & $0.834$ & $4.01\times 10^{-3}$ & \multirow{3}{*}{$1$} & \multirow{3}{*}{$1$} & $0.611$ & $9.15\times 10^{-3}$\\
    \cline{4-6}\cline{8-9}\cline{16-17}\cline{20-21}
                & &  & $66$ & $0.990$ & $1.75 \times 10^{-3}$ &  & $\mathbf{0.991}$ & $1.18 \times 10^{-3}$ & & & \multicolumn{4}{c||}{} & $0.852$ & $2.04\times 10^{-3}$ & & & $0.607$ & $5.44\times 10^{-3}$ \\
    \cline{4-6}\cline{8-9}\cline{16-17}\cline{20-21}
                & &  & $99$ & $0.993$ & $1.11 \times 10^{-3}$ &  & $\mathbf{0.994}$ & $7.75\times 10^{-4}$ & & & \multicolumn{4}{c||}{} & $0.857$ & $1.36\times 10^{-3}$ & & & $0.601$ & $7.71\times 10^{-3}$\\
    \hline
    \end{tabular}       
    }
    \caption{Average cumulative reward $\bar{F}_{\X}$ ($t = T/3, 2T/3, T$), normalized by fractional optimal $F^\star$, of integral policies across different datasets and constraints. Optimal hyperparameters ($\eta, \gamma, c_{\mathrm p}$) are reported with each algorithm (see Appendix~\fullversion{H of \citet{si2023online}}{\ref{appendix:experiments}} for value ranges explored). \Algorithm{RAOCO} combined with \Algorithm{OGA} or \Algorithm{OMA} outperforms competitors almost reaching one, with the exception of \texttt{MovieLens}, where \Algorithm{TabularGreedy} does better. As \Algorithm{Random} also performs well on \texttt{MovieLens}, this indicates that the (static) offline optimal is quite poor for this reward sequence.   By Property~\ref{assumption:sandwich}, fractional solutions strictly dominate the integral optimal, which implies that in all cases \Algorithm{RAOCO} outperformed the $1-1/e$ approximation, attaining an almost optimal value.  }
    \label{tab:res}
\end{table*}

\paragraph{Concave Relaxations and Randomized Rounding.}  Concave relaxations are linear in the representation of the function $f$, but in practice come ``for free'', once parameters in Eq.~\eqref{eq:wtp} are provided.   Swap rounding over a general matroid is  $\BigO{n r^2}$, where 
$r$ is the rank of the matroid~\cite{chekuri2010dependent}. This dominates the remaining operations (including OCO), and thus determines the overall complexity of our algorithm.
This $\BigO{n r^2}$ term assumes access to \new{the} decomposition of a fractional point $\y \in \Y$ to bases of the matroid. Carathéodory's theorem implies the existence of such decomposition of at most $n+1$ points in $\X$; moreover, there exists a decomposition algorithm~\cite{cunningham1984testing} for general matroids with running time $\BigO{n^6}$. 
However, in all practical cases we consider (including uniform and partition matroids) the complexity is significantly lower. More specifically,
for partition matroids,  swap rounding reduces to an algorithm with linear time complexity, namely, $\BigO{m n}$ for $m$ partitions~\cite{Srinivasan2001}. 



\section{Extensions}\label{sec:extensions}

\paragraph{Dynamic Setting.} In the dynamic setting, the decision maker compares its performance to the best sequence of decisions $(\y^\star_t)_{t \in \T} $ with a path length {regularity} condition~\cite{zinkevich2003,cesa2012mirror}. I.e., let
$\Lambda_{\X} (T,P_T) \triangleq \set{\parentheses{\x_t}^T_{t=1} \in \X^T: \sum^T_{t=1} \norm{\x_{t+1} - \x_t} \leq P_T} \subset \X^T$
be the set of sequences of decision in a set $\X$ with path length less than $P_T \in \reals_{\geq 0}$ over a time horizon $T$. We extend the definition of the regret (Eq.~\eqref{eq:alpha_regret}) to that of  \emph{dynamic $\alpha$-regret}: 
\begin{align*}
     &\alpha\text{-}\mathrm{regret}_{T, P_T} ({\vec\P}_\X) \triangleq  \sup_{\parentheses{{f}_t}^T_{t=1} \in {\F}^T} \bigg\{\max_{\parentheses{\x^\star_t}^T_{t=1}\in    \Lambda_{\X} (T,P_T)}  \alpha\sum^T_{t=1} {f}_t(\x^\star_t) -\sum^T_{t=1} {f}_t(\x_t)\bigg\} .
\end{align*}
When $\alpha\text{-}\mathrm{regret}_{T, P_T}({\vec\P}_\X)$  is sublinear in $T$, the policy attains average rewards that asymptotically compete with the optimal decisions \emph{of bounded path length}, in hindsight.

Through our reduction to OCO, we can  leverage  OGA \cite{zinkevich2003} or meta-learning algorithms over OGA~\cite{zhao2020dynamic} to obtain dynamic regret guarantees in OSM. As an additional technical contribution, we provide  the first sufficient and necessary conditions for OMA to admit a dynamic regret guarantee (see Appendix~\fullversion{A of \citet{si2023online}}{\ref{appendix:oco}}). This allows to extend a specific instance of OMA operating on the simplex,  the so-called fixed-share algorithm~\cite{cesa2012mirror,herbster1998tracking},  to  matroid polytopes.  This yields a tighter regret guarantee than OGA~\cite{zinkevich2003,zhao2020dynamic} (see Theorem~\fullversion{7}{\ref{thm:general_regret_dynamic+optimistic}} in Appendix~\fullversion{A of \citet{si2023online}}{\ref{appendix:oco}}).
Putting everything together,  we get:
\begin{theorem}\label{thm:dynamic}
Under Asm.~\ref{assumption:sandwich},  RAOCO policy $\vec\P_\X$  described by  Alg.~\ref{alg:saoco} equipped with an OMA policy in Appendix~\fullversion{A of \citet{si2023online}}{\ref{appendix:oco}} as OCO policy $\vec\P_{\Y}$ 
 has sublinear dynamic $\alpha$-regret, i.e.,     
 $    \alpha\text{-}\mathrm{regret}_{T,P_T} ({\vec\P}_\X) = \BigO{\sqrt{P_T T} }.$
\end{theorem}
This follows from Thm.~\ref{theorem:sandwich} and the dynamic regret guarantee for OMA    in Appendix~\fullversion{A of \citet{si2023online}}{\ref{appendix:oco}}.

\paragraph{Optimistic Setting.} 
\label{s:optimism}
In the optimistic setting, the decision maker has access to additional information available in the form of predictions: a  function $\pi_{t+1}:\interval{0,1}^n \to \reals_{\geq 0}$,  serving as a prediction of the reward function $f_{t+1}(\x)$ is made available  before committing to a decision $\x_{t+1}$ at timeslot $t \in \T$.  The prediction $\pi_t$  encodes prior information available to the decision maker at timeslot $t$.\footnote{Function $\pi_t$ can extend over fractional values, e.g.,  be the multi-linear relaxation of a set function.}  Let  $\g_t$ and $\pg_t$ be supergradients of $\cfunc_t$ and $\pi_t$ at point $\y_t$, respectively.  We can define an optimistic OMA policy 
(see Alg.~\fullversion{3}{\ref{alg:ooma}} in Appendix~\fullversion{A of \citet{si2023online}}{\ref{appendix:oco}}) that leverages both the prediction and the observed rewards. Applying again our reduction of OSM to this setting we get:
\begin{theorem}\label{thm:osm_optimistic}
Under Asm.~\ref{assumption:sandwich},  RAOCO policy $\vec\P_\X$  in Alg.~\ref{alg:saoco} equipped with  OOMA in Appendix~\fullversion{A of \citet{si2023online}}{\ref{appendix:oco}} as policy $\vec\P_{\Y}$ yields
  $   \alpha\text{-}\mathrm{regret}_{T, P_T} ({\vec\P}_\X)=\BigO{ \sqrt{P_T\sum^T_{t=1}\norm{\g_t - \pg_t}^2_\infty}}.$ 
\end{theorem}
This theorem follows from Theorem~\ref{theorem:sandwich} and the optimistic regret guarantee established for OMA policies in Theorem~\fullversion{7}{\ref{thm:general_regret_dynamic+optimistic}} in Appendix~\fullversion{A of \citet{si2023online}}{\ref{appendix:oco}}.
The optimistic regret guarantee in Theorem~\ref{thm:osm_optimistic} shows that the regret of a policy can be reduced to $0$ when the predictions are perfect, while providing $\BigO{\sqrt{T}}$ guarantee in Thm.~\ref{thm:main2} when the predictions are arbitrarily bad (with bounded gradients).
To the best of our knowledge, ours is the first work to provide guarantees for  optimistic OSM.

\paragraph{Bandit Setting.}
\label{s:bandit}
Recall that in the bandit setting the decision maker  only has  access to the reward $f_t(\x_t)$ after committing to a decision $\x_t \in \X$ at $t\in\T$; i.e., the reward function is \emph{not} revealed. 
Our reduction to OCO does not readily apply to the bandit setting; however, we show that the bandit algorithm by \citet{wan2023bandit} can be used to construct such a reduction. The main challenge is to estimate gradients of inputs in $\Y$ only from bandit feedback; this can be done via a perturbation method (see also~\citet{hazan2014bandit}). This approach, described in Appendix~\fullversion{G of \citet{si2023online}}{\ref{appendix:extensions}}, yields the following theorem:
 

\begin{figure}[t]
    \centering
    \subcaptionbox{Stationary}{\includegraphics[height = .19\textwidth,trim={0 0 .5cm 0}]{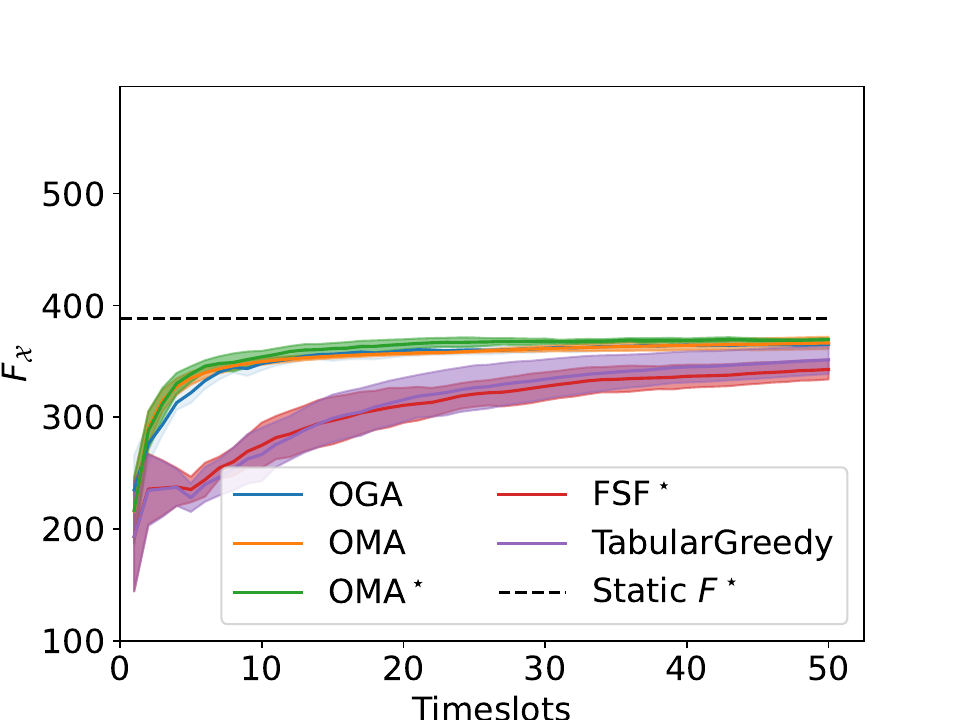}}
    \subcaptionbox{Non-Stationary}{\hspace*{-10pt}\includegraphics[height = .19\textwidth, trim={0.54cm 0 .7cm 0},clip]{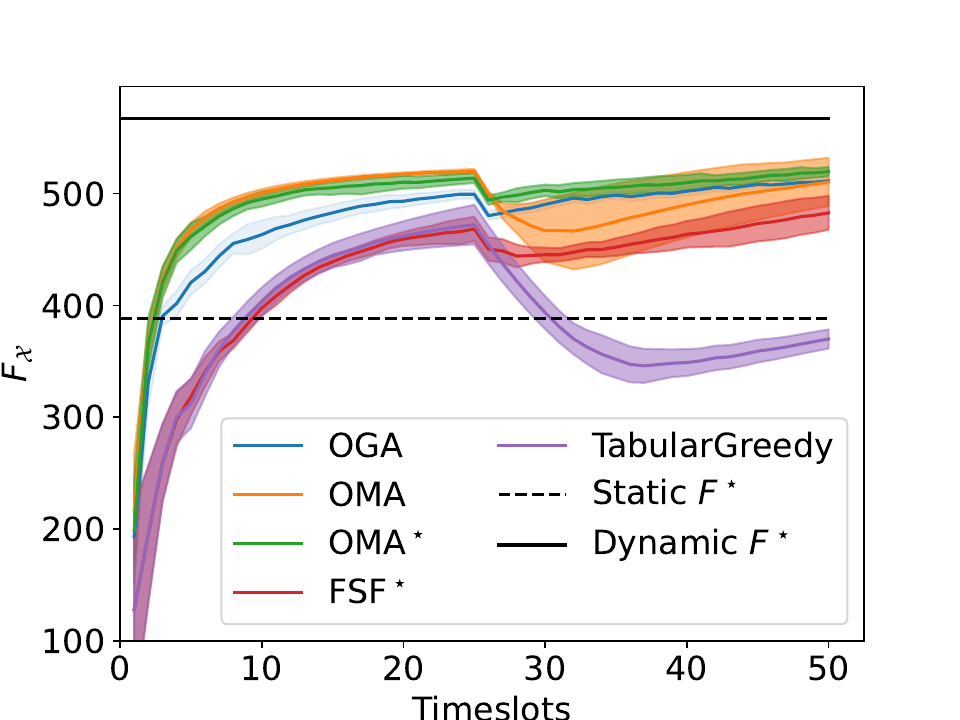}}
    \caption{ Average cumulative reward $\bar F_\X$ of the different policies under \WeightedCoverage{} dataset under different setups: stationary in (a) and non-stationary in (b).  Non-stationarity in (b) is applied by changing the objective at $t=25$ (see Appendix~\fullversion{H of \citet{si2023online}}{\ref{appendix:experiments}}). The area depicts the standard deviation over 5 runs.}\label{fig:dynamicregret}
\end{figure}


\begin{theorem} Under bounded submodular monotone rewards and partition matroid constraint sets, LIRAOCO policy $\vec\P_\X$  in Alg.~\fullversion{5}{\ref{alg:lisaoco}} in Appendix~\fullversion{G of \citet{si2023online}}{\ref{appendix:extensions}} equipped with  an OCO policy $\vec\P_{\Y_\delta}$  \label{theorem:bandit}
yields 
 $   \aregret_{T, P_T, W} \parentheses{\vec\P_\X} \leq  W \cdot \mathrm{regret}_{T/W, P_T} \parentheses{\vec{\P}_{\Y_\delta}} + \frac{T}{W} + 2 \alpha \delta r^2 n   T,$
 where $\delta, W$, are tuneable parameters of the algorithm and  $\mathrm{regret}_{T/W, P_T} \parentheses{\vec{\P}_{\Y_\delta}}$ is the regret of an OCO policy executed for $T/W$ timeslots.
\end{theorem}
\new{Thm.~\ref{theorem:bandit}} applies to general submodular functions, but is restricted to partition matroids.  \new{The theorem} also extends to dynamic and optimistic settings: we provide the full description in  Appendix~\fullversion{G of \citet{si2023online}}{\ref{appendix:extensions}}. Our analysis generalizes  \citet{wan2023bandit} in that (a) we show that a reduction can be performed to any OCO algorithm, rather than just FTRL, as well as  (b) in extending it to the dynamic and optimistic settings 
(see also Table~\fullversion{4}{\ref{table:compare2}} in \fullversion{\citet{si2023online}}{the Appendix}).

\section{Experiments}\label{sec:experiments}

\paragraph{Datasets and Problem Instances.} We consider five different OSM problems: two \old{are}  influence maximization problems, over the \IMZKC{} \cite{zachary1977information} and \IMEpinions{} 
\cite{epinions} graphs, respectively, a facility location problem over the \FLMovieLens{} dataset \cite{movielens}, a team formation, and a weighted coverage problem over synthetic datasets. Reward functions are generated over a finite horizon and optimized online over both uniform and partition matroid constraints. 
\new{Details are provided in Appendix~\fullversion{H of \citet{si2023online}}{\ref{appendix:experiments}}.}

\paragraph{Algorithms.} We implement the policy in Alg.~\ref{alg:saoco} (\Algorithm{RAOCO}), coupled with \Algorithm{OGA} (\RAOCOOGA{}) and \Algorithm{OMA} (\RAOCOOMA{}) as OCO policies. As competitors, we also implement the \emph{fixed share forecaster} (\Algorithm{FSF$^\star$}) policy by~\citet{matsuoka2021tracking}, which only applies to uniform matroids, and the 
\Algorithm{TabularGreedy} policy by~\citet{onlineassignement}, as well as a \Algorithm{Random} algorithm, which selects a decision u.a.r.~from the bases of $\X$.
\new{Details and hyperparameters explored are reported in Appendix~\fullversion{H of \citet{si2023online}}{\ref{appendix:experiments}}}. Our code is publicly available.\footnote{\url{https://github.com/neu-spiral/OSMviaOCO}}

\paragraph{Metrics.} For each setting, we compute $F^\star = \max_{\y\in \Y}\frac{1}{T}\sum_{t=1}^T\cfunc(\y)$, the value of the optimal fractional solution in hindsight, assuming rewards are replaced by their concave relaxations. 
Note that this involves solving a convex optimization problem, and overestimates the (intractable) offline optimal, i.e., $F^\star\geq \max_{\x\in \X}\frac{1}{T}\sum_{t=1}^Tf(\x)$. 
For each online policy, we compute both the integral and fractional average cumulative reward at different timeslots $t$, given by $\bar{F}_{\X} =\frac{1}{t}\sum_{s=1}^tf(\x_s)$,  $\bar{F}_{\Y} =\frac{1}{t}\sum_{s=1}^t\cfunc(\y_s)$, respectively.
We repeat each experiment for $5$ different random seeds, and use this to report standard deviations.


\paragraph{OSM Policy Comparison.} A comparison between the two  versions of \Algorithm{RAOCO} (\Algorithm{OGA} and \Algorithm{OMA}) with the three competitors is shown in Table~\ref{tab:res}. 
We observe that both \Algorithm{OGA} and \Algorithm{OMA} significantly outperform competitors, with the exception of \texttt{MovieLens}, where \Algorithm{TabularGreedy} does better. \Algorithm{OMA} is almost always better than \Algorithm{OGA}. Most importantly, we significantly outperform both \Algorithm{TabularGreedy} and \Algorithm{FSF$^\star$} w.r.t.~running time, being $2.15-250$ times faster (see Appendix~\fullversion{H of \citet{si2023online}}{\ref{appendix:experiments}}).

\paragraph{Dynamic Regret and  Optimistic Learning.} Fig.~\ref{fig:dynamicregret} shows the performance of the different policies in a dynamic scenario. All the policies have similar performance in the stationary setting, however in the non-stationary setting only \Algorithm{OGA}, \Algorithm{OMA}, and \Algorithm{FSF$^\star$} show robustness. We further experiment with an optimistic setting in Fig.~\fullversion{4 in \citet{si2023online}}{\ref{fig:optimistic_figure}} under a non-stationary setting where we provide additional information about future rewards to the optimistic policies, which can leverage this information and yield better results (see Appendix~\fullversion{H of \citet{si2023online}}{\ref{appendix:experiments}}).

\section{Conclusion}
\label{sec:conclusions}
We provide a reduction of  online maximization of a large class of submodular
functions to online convex optimization and  show that our reduction extends to many different versions of the online learning problem.
There are many possible extensions. As our framework does not directly apply to general submodular functions, it would be interesting to derive a  general method to construct the concave relaxations which are the building block of our reduction in Sec.~\ref{s:osm_genral}. It is also meaningful to investigate the applicability of our reduction to monotone submodular functions with bounded curvature~\cite{feldman2021guess}, and non-monotone submodular functions~\cite{krause2014submodular}.

\section{Acknowledgements}
We gratefully acknowledge support from the NSF (grants 1750539, 2112471, 1813406, and 1908510).

\bibliographystyle{unsrtnat}
\bibliography{references}

\input{paper_appendix}
\end{document}

%% file: notations_macros.tex
\usepackage{blindtext}
\usepackage{enumitem}
\usepackage{xpatch}
\usepackage{bm}
\usepackage[framemethod=TikZ]{mdframed}
\usepackage{ulem}
\newcommand{\IMZKC}{\texttt{ZKC}}
\newcommand{\IMEpinions}{\texttt{Epinions}}
\newcommand{\FLMovieLens}{\texttt{MovieLens}}
\newcommand{\TeamFormation}{\texttt{SynthTF}}
\newcommand{\WeightedCoverage}{\texttt{SynthWC}}
\newcommand{\Algorithm}[1]{\texttt{#1}}
\newcommand{\RAOCOOGA}{\Algorithm{RAOCO-OGA}}
\newcommand{\RAOCOOMA}{\Algorithm{RAOCO-OMA}}

\newmdtheoremenv{btheorem}{Theorem}
\newmdtheoremenv{blemma}{Lemma}
\newmdtheoremenv{bassumption}{Assumption}
\newmdtheoremenv{bcorollary}{Corollary}

\newtheorem{observation}{Observation}
\newtheorem{assumption}{Assumption}
\newtheorem{definition}{Definition}
\renewcommand{\vec}[1]{\bm{\mathbf{#1}}}
\newcommand{\x}{\vec{x}}
\newcommand{\y}{\vec{y}}
\newcommand{\z}{\vec{z}}
\newcommand{\h}{\vec{h}}
\newcommand{\w}{\vec{w}}

\newcommand{\interval}[1]{\left[#1\right]}
\newcommand\norm[1]{\lVert#1\rVert}
\newcommand{\E}{\mathbb{E}}
\newcommand{\I}{\mathcal{I}}
\newcommand{\X}{\mathcal{X}}
\newcommand{\reals}{\mathbb{R}}
\newcommand{\naturals}{\mathbb{N}}
\renewcommand{\H}{\vec{H}}
\newcommand{\set}[1]{\left\{#1\right\}}
\newcommand{\parentheses}[1]{\left(#1\right)}
\newcommand{\conv}[1]{\mathrm{conv}\parentheses{#1}}
\newcommand{\card}[1]{\left|#1\right|}
\renewcommand{\P}{{\mathcal{P}}}
\newcommand{\supp}[1]{\mathrm{supp}\parentheses{#1}}
\newcommand{\F}{\mathcal{F}}

\newcommand{\T}{[T]}
\newcommand{\M}{\mathcal{M}}

\newcommand{\charvec}{\vec{1}}

\renewcommand{\F}{\mathcal{F}}
\newcommand{\D}{\mathcal{D}}

\renewcommand{\w}{\vec{\w}}

\newcommand{\BigO}[1]{\mathcal{O}\parentheses{#1}}

\newcommand{\Y}{\mathcal{Y}}

\newcommand{\regret}{\mathrm{regret}}
\newcommand{\aregret}{\alpha \text{-}\mathrm{regret}}

\newcommand{\cfunc}{\tilde{f}}

\newcommand{\cfuncSet}{\tilde{\F}}
\newcommand{\argmax}{\mathrm{argmax}}

\newcommand{\g}{\vec g}
\usepackage{multirow}
\usepackage{pifont}
\newcommand{\xmark}{\ding{53}}%

\newcommand{\paintbg}{\cellcolor{royalazure!7}}

\newcommand{\inner}[2]{ \parentheses{#1}\cdot\parentheses{#2}}

\newcommand{\py}{\y^\pi}
\newcommand{\pg}{\g^\pi}

\newcommand{\argmin}{\mathrm{argmin}}
\usepackage{amsmath,eqparbox,xparse}
\makeatletter
\NewDocumentCommand{\eqmathbox}{o O{c} m}{%
  \IfValueTF{#1}
    {\def\eqmathbox@##1##2{\eqmakebox[#1][#2]{$##1##2$}}}
    {\def\eqmathbox@##1##2{\eqmakebox{$##1##2$}}}
  \mathpalette\eqmathbox@{#3}
}
\makeatother

\newcommand{\rank}{r}

\renewcommand{\paragraph}[1]{\par\noindent\textbf{#1}}

\newenvironment{packed_itemize}{
\begin{list}{\labelitemi}{\leftmargin=2em}
\vspace{-5pt}
 \setlength{\itemsep}{0pt}
 \setlength{\parskip}{0pt}
 \setlength{\parsep}{0pt}
}{\end{list}}

\newcommand{\catalog}{\mathcal{C}}
\newcommand{\servers}{\mathcal{S}}
\newcommand{\capacity}{c}
\newcommand{\ppath}{p}
\definecolor{royalazure}{rgb}{0.0, 0.22, 0.66}
\usepackage{array}
\makeatletter
\newcommand{\thickhline}{%
    \noalign {\ifnum 0=`}\fi \hrule height 1pt
    \futurelet \reserved@a \@xhline
}
\newcolumntype{"}{@{\hskip\tabcolsep\vrule width 1pt\hskip\tabcolsep}}
\makeatother

\newcommand{\efunc}{\hat{f}}
\newcommand{\eFunc}{\hat{F}}

\DeclareFontFamily{U}{mathx}{\hyphenchar\font45}
\DeclareFontShape{U}{mathx}{m}{n}{
      <5> <6> <7> <8> <9> <10>
      <10.95> <12> <14.4> <17.28> <20.74> <24.88>
      mathx10
      }{}
\DeclareSymbolFont{mathx}{U}{mathx}{m}{n}
\DeclareMathSymbol{\bigtimes}{1}{mathx}{"91}

\newcommand{\forcednewline}{$$ $$}
\usepackage{mathtools}
\DeclarePairedDelimiter\ceil{\lceil}{\rceil}

\newcommand{\avg}{\mathrm{avg}}
\newcommand{\given}{\,\big|\,}

%% file: paper_appendix.tex
\clearpage
\onecolumn
\appendix
\addcontentsline{toc}{section}{Appendices}
\renewcommand{\thesubsection}{\Alph{section}.\Roman{subsection}}

\begin{center}
{\Large  \textbf{Appendices}}
\end{center}

\begin{algorithm}[htbp]
\caption{Online Mirror Ascent: $\mathrm{OMA}_{ \eta,\Phi, \Y}(\y_t, \cfunc_t)$\label{alg:oma}}
    \begin{algorithmic}[1]
        \Require Learning rate $\eta \in \reals_{\geq 0}$, mirror map $\Phi: \mathcal{D} \to \reals$, decision set $\Y$, decision $\y_t$, reward function $\cfunc_t: \Y \to \reals$, 
\State
		$\hat{\vec{y}}_t\gets \nabla \Phi (\vec{y}_t)$\Comment{{Map the primal point to a dual point}}
		\State
		$\hat{\vec{z}}_{t+1}   \gets \hat{\vec{y}}_t + \eta \vec g_t$\Comment{{ Take supergradient step in the dual space ($\vec g_t \in \partial \cfunc_t(\y_t)$)}}
		\State
		$\vec{z}_{t+1}   \gets \left(\nabla \Phi\right)^{-1}(\hat{\vec{z}}_{t+1})$\Comment{{ Map the dual point to a primal point}}
		\State
		$\vec{y}_{t+1}\gets \Pi_{\Y \cap \D}^\Phi(\vec{z}_{t+1})$\Comment{{ Project new point onto feasible region $\mathcal{X}$}}
        \State \Return $\y_{t+1}$
    \end{algorithmic}
\end{algorithm}

\begin{figure}[htbp]
    \centering
    \includegraphics[width=.5\linewidth]{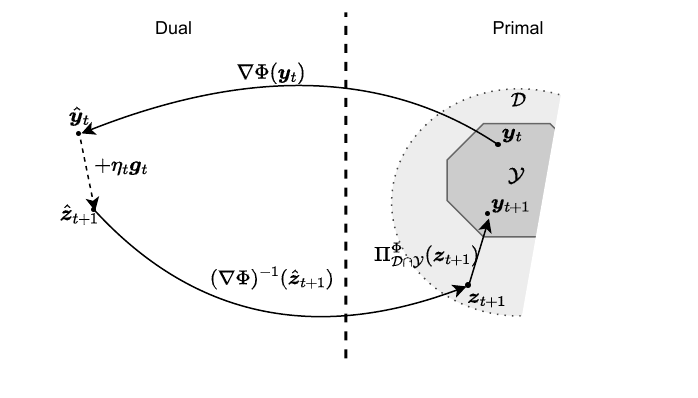}
    \caption{Online Mirror Ascent in Alg.~\ref{alg:oma}.\label{fig:oma}}
    \label{fig:my_label}
\end{figure}

\begin{algorithm}[t]
\caption{Optimistic Online Mirror Ascent: $\mathrm{OOMA}_{ \eta,\Phi, \Y}(\y_t, \cfunc_t, \pi_{t+1})$}\label{alg:ooma}
    \begin{algorithmic}[1]
        \Require Learning rate $\vec \eta \in \reals_{\geq 0}$, mirror map $\Phi: \mathcal{D} \to \reals$, decision set $\Y$, decision $\y_t$, reward function $\cfunc_t: \Y \to \reals$, prediction function $\pi_{t+1}$
        \State $\y^{\pi}_{t+1} \gets \mathrm{OMA}_{ \eta,\Phi, \Y}(\y_t, \cfunc_t)$ \Comment{Adapt secondary decision according to the reward $f_t$ and previous primary decision $\y_t$}
        \State $\y_{t+1} \gets \mathrm{OMA}_{ \eta,\Phi, \Y}(\y^{\pi}_t, \pi_{t+1})$ \Comment{Adapt primary decision according to the precision $\pi_{t+1}$ of the future reward $\cfunc_{t+1}$ and previous secondary decision $\py_t$}
    \State \Return $\y_{t+1}$ 
    \end{algorithmic}
\end{algorithm}

\section{Online Convex Optimization}
\label{appendix:oco}
In this section, we  describe  two powerful families of OCO policies: online mirror ascent (OMA, in Appendix~\ref{s:desc_omd}) and follow the regularized leader (FTRL, in Appendix~\ref{s:desc_ftrl}). Online gradient ascent (OGA) can be described as a special case of OMA, under a Euclidean mirror map. Our exposition below also simultaneously covers  the guarantees that come with OMA in the dynamic and optimistic settings; both follow from Theorem~\ref{thm:general_regret_dynamic+optimistic} (see Appendix~\ref{appendix:extensions} for the derivation). Results for the static setting (e.g., Theorem~\ref{thm:oco}) also follow from Theorem~\ref{thm:general_regret_dynamic+optimistic} by appropriately restricting the corresponding parameters (see Appendix~\ref{app:dervofthm:oco}).

The setting, algorithms, and analysis we present here are a combination of the ones provided by \citet{rakhlin2013online}, \citet{zinkevich2003}, and \citet{dynamicregret2020}. In particular, \citet{rakhlin2013online} present OMA in the optimistic setting, while \citet{zinkevich2003} and \citet{dynamicregret2020} studied OGA and OMA, respectively, in the dynamic setting. In the ``master theorem'' (Theorem~\ref{thm:general_regret_dynamic+optimistic}) that we present below, we combine the proof techniques by all three, to give an algorithm that comes with simultaneous guarantees in all  settings (static or dynamic, optimistic or non-optimistic, and combinations thereof).

\subsection{Online Mirror Ascent} 
\label{s:desc_omd}
Online Mirror Ascent~\cite[Sec. 5.3]{bubeck2011introduction} is the online version of the mirror ascent (MA) algorithm~\cite{beck2003mirror} for convex optimization of a fixed, known function; OMA is a class of policies that generalize OGA. The main premise behind mirror ascent is that variables and supergradients live in two distinct spaces: the primal space, for variables, and the dual space, for supergradients. The two are linked via a function known as a mirror map. Contrary to standard supergradient ascent, updates using the gradient occur on the dual space; the mirror map is used to invert this update to a change on the primal variables.  For several constrained optimization problems of
interest, mirror ascent leads to faster convergence compared to gradient ascent~\cite[Sec.~4.3]{bubeck2011introduction}. 

The primal and dual spaces are $\mathbb{R}^n$. To disambiguate between the two, we denote primal points by $\y, \z \in \mathbb{R}^n$   and dual points by  $\hat{\y},\hat{\z} \in \mathbb{R}^n$, respectively.   Formally, OMD is parameterized by (1) a learning rate $\eta \in \mathbb{R}_+$, and (2) a  differentiable map $\Phi : \mathcal{D} \to \mathbb{R}$ where $\D$ is its domain. 
\begin{assumption}\label{asm:mirror_map}
    The map $\Phi: \D \to \reals$ satisfies the following properties:
    \begin{enumerate}
        \item The domain $\D$ of $\Phi$ is a convex and open set such that the decision set $\Y$ is included in its closure, i.e., $\Y \subseteq  \mathrm{closure}(\mathcal{D})$, and their intersection is nonempty $\Y \cap \D \neq \emptyset$. 
        \item The map $\Phi$ is $\rho$ strongly-convex over $\D$ w.r.t. a norm $\norm{\,\cdot\,}$ and differentiable over $\D$.
        \item The map $\nabla \Phi(\x): \D \to \reals^n$ is surjective.
        \item The gradient of $\Phi$ diverges on the boundary of $\D$, i.e., $\lim_{\x \to \partial \D} \norm{\nabla \Phi(\x)} = +\infty$, where  $\partial \D = \mathrm{closure}(\D) \setminus \D$.

    \end{enumerate}
\end{assumption}
When the map $\Phi$ satisfies the above properties is said to be a \emph{mirror map}. 

OMA takes the form described in Alg.~\ref{alg:oma}. Given $\eta$ and $\Phi$, an OMA iteration proceeds as follows: after receiving the reward $\cfunc_t(\y_t)$ and the reward function $\cfunc_t(\,\cdot\,)$ is revealed, the current state $\y_t$ is first mapped from the primal  to the dual space via: 
\begin{align}\label{eq:mapstep}\hat{\y}_{t} = \nabla \Phi(\y_t).\end{align} Then, a regular supergradient ascent step is performed in the \emph{dual space} to obtain an updated dual point: \begin{align}\hat{\z}_{t+1}   = \hat{\y}_t +  \eta \g_t, \end{align} where $\g_t \in \partial \cfunc_t(\y_t)$ is the supergradient of $\cfunc_t$ at point $\y_t$. This updated dual point is then mapped back to the primal space using the inverse  of mapping $\nabla \Phi$, i.e.:
\begin{align}\label{eq:invstep}\z_{t+1} = \left(\nabla \Phi\right)^{-1}\!(\hat{\z}_{t+1}).\end{align}
The resulting primal point $\z_{t+1}$ may lie outside the decision set $\Y$. To obtain the final feasible point $\y_{t+1}\in \Y$, a projection is made using the Bregman divergence associated with the mirror map~$\Phi$. The final result becomes:
\begin{align}\label{eq:bregprojstep}
\y_{t+1}= \Pi_{\Y \cap \D}^\Phi(\z_{t+1}),
\end{align}
where $\Pi_{\Y \cap \D}^\Phi(\,\cdot\,)$ is the Bregman projection~\eqref{eq:Bregman_projection} associated to $\Phi$ onto the set $\Y \cap \D$.  The steps described Eqs.~\eqref{eq:mapstep}--\eqref{eq:bregprojstep} define OMA. The  illustration in Fig.~\ref{fig:oma} summarizes these steps. Note that $\vec{y}_{t+1}$ is a function of $\{(\vec{y}_{t}, \cfunc_{t})\} \subseteq \{(\y_{s}, \cfunc_s)\}^{t}_{s=1} $, hence OMA is indeed an OCO policy. 

We conclude our description of OMA by providing a formal definition of supergradients and the Bregman projection.
\begin{definition}
    The superdifferential of a concave function $\cfunc : \reals^n \to \reals$ at point $\y \in \reals^n$ is defined as the set $\partial \cfunc (\y) \subset \reals^n$ of supergradients s.t.  every $\vec g(\y) \in \partial \cfunc(\y)$ satisfies the following inequality
\begin{align}
\label{eq:def_supergradient}
    \cfunc(\y) - \cfunc(\y') \geq \vec g(\y) \cdot (\y - \y' ), \quad \text{for every $\y' \in \reals^n$.}
\end{align}
When the function $\cfunc$ is differentiable, then by definition $\partial \cfunc(\y) = \{\nabla \cfunc(\y)\}$, where $\nabla \cfunc(\y)$ is the gradient of $\cfunc$ at point $\y \in \reals^n$.
\end{definition}
\begin{definition}
The Bregman projection~\cite{kiwiel1997proximal} associated to a map $\Phi$ onto a convex set $\mathcal S$ is  denoted by $\Pi^{\Phi}_{\mathcal{S}}:\mathbb{R}^n\to\mathcal{S}$, is defined as
    \begin{align}
        \Pi^{\Phi}_{\mathcal{S}}(\y') &= \underset{\y \in {\mathcal{S}}}{\arg\min}\,D_\Phi (\y,\y'), & \text{where}& &  D_\Phi(\y,\y') = \Phi(\y) - \Phi(\y') - \nabla {\Phi(\y')} \cdot (\y - \y')\label{eq:Bregman_projection}.
    \end{align}
\end{definition}
OMA considers that the decision set is bounded w.r.t. the Bregman divergence associated to $\Phi$. Formally, we assume the following. 
\begin{assumption}\label{asm:bounded_bregman_divergence}
Consider a decision set $\Y$ and a mirror map $\Phi$ and $\y_1  \in \Y$. The Bregman divergence $D_\Phi$  associated to $\Phi$ is bounded over $\Y$, i.e., there exits $D< \infty$ s.t.
\begin{align}
{D_\Phi(\y, \y_1)} &\leq D^2,  &\text{for any $\y \in \Y$.} \label{eq:bounded_diam_mirror_map}
\end{align}   
\end{assumption}
\begin{assumption}\label{asm:bounded_gradient_mirror_map}
    Consider a decision set $\Y$ and a mirror map $\Phi$. The dual norm $\norm{\,\cdot\,}_\star$ of the gradient of the mirror map $\Phi$ is bounded by $L_\Phi  \in \reals_{\geq 0}$, i.e., 
    \begin{align}
        \norm{\nabla \Phi (\y)}_\star &\leq L_\Phi, &\text{ for every $\y \in \Y$.} \label{eq:bounded_grad_map}
    \end{align}
\end{assumption}

Note that Asm.~\ref{asm:bounded_gradient_mirror_map} is typically not required to establish static regret guarantees~\eqref{eq:general_regret} (see, e.g., \citet{hazan2016introduction}). However,  as we  remark in Appendix~\ref{appendix:setups_omd}, policies with sublinear static regret do not necessarily perform well in dynamic settings, e.g., this the case for OMA configured with the standard negative entropy mirror map~\cite{bubeck2011introduction}.  In Theorem~\ref{thm:general_regret_dynamic+optimistic}, we show that Asm.~\ref{asm:bounded_gradient_mirror_map} is sufficient to establish sublinear dynamic regret against a comparator sequence (see Sec.~\ref{sec:extensions}) with path length $P_T >0$.

\paragraph{Regret Guarantee.} Here, we provide a general regret bound for \emph{Optimistic} OMA in Alg.~\ref{alg:ooma} configured with a general mirror map $\Phi$; note that the same algorithm applies to both the optimistic and dynamic settings, and so does the formal guarantee that we provide below. The non-optimistic regret bounds follow from the optimistic regret bounds by setting the predictions to $0$ ($\pi_{t} = 0$), and Alg.~\ref{alg:ooma} reduces to Alg.~\ref{alg:oma}. Static (i.e., non-dynamic) bounds follow by setting $P_T=0$.


\begin{theorem}\label{thm:general_regret_dynamic+optimistic}
     Consider an OCO setting with a convex decision space $\Y$,  concave rewards $\cfunc \in \hat{\F}$ satisfying Asm.~\ref{asm:oco} from a set $\hat \F$ selected by an adversary, and predictions $\pi_t: \Y \to \reals$ for $t \in \T$.   The regret of OOMA in Alg.~\ref{alg:ooma} configured with $\Phi$ satisfying Asms.~\ref{asm:mirror_map}--\ref{asm:bounded_gradient_mirror_map} against a sequence  of decision $\parentheses{\y_t^\star}_{t \in \T} \in  \Lambda_{\Y} (T,P_T) $ is upper bounded as follows
    \begin{align}
         \mathrm{regret}_{T, P_T} (\vec \P_\Y) & \leq {\eta {\sum^T_{t =1} \norm{\vec g_t - \pg_t}^2_\star} }/{2\rho} + \parentheses{D^2 + 2 L_\Phi P_T}/{\eta}. 
    \end{align}
When  the learning rate $\eta_\star  = \sqrt{{2\rho \parentheses{D^2 + 2 L_\Phi P_T}} / {\sum_{t \in \T} \norm{\vec g_t - \pg_t}^2_\star}}$ is selected, the following holds
\begin{align}
      \mathrm{regret}_{T, P_T} (\vec \P_\Y) \leq \sqrt{{2/\rho\parentheses{D^2 + 2 L_\Phi P_T} \sum^T_{t=1} \norm{\vec g_t - \pg_t}^2_\star}}
\end{align}
\end{theorem}
\begin{proof}
    We adapt the proof of~\citet[Lemma~1]{rakhlin2013optimization} to support dynamic regret. 
\begin{align}
    &\sum_{t \in [T]}\inner{\g_t}{ \y^\star_t - \y_t} \\
    &\leq \sum_{t \in [T]} \norm{\g_t - \pg_t}_\star \norm{\y_t - \py_t}_\star + \frac{1}{\eta} \parentheses{D_\Phi(\y^\star_t, \py_{t-1}) - D_\Phi(\y^\star_t, \py_{t}) - \frac{\rho}{2} \norm{\py_t - \y_t}^2 }\label{eq:proof_dynamic_optimistic_p0}
\end{align}
In order to upper bound the r.h.s. term in the above inequality. We first bound the term $\sum_{t \in \T}\norm{\g_t - \pg_t}_\star \norm{\y_t - \py_t}$, and then, secondly, we bound the term $\sum_{t\in [T]}\parentheses{D_\Phi(\y^\star_t, \py_{t-1}) - D_\Phi(\y^\star_t, \py_{t})}$.

First, note that it holds that $ a b = \inf_{\eta'} \frac{\eta'}{2} a^2 + \frac{1}{2\eta'} b^2$, so we have
\begin{align}
    \norm{\g_t - \pg_t}_\star \norm{\y_t - \py_t} &=\inf_{\eta'} \set{\frac{\eta'}{2}   \norm{\g_t - \pg_t}^2_\star  + \frac{1}{2\eta'} \norm{\y_t - \py_t}^2 }  \\
    &\leq \frac{\eta}{2\rho}   \norm{\g_t - \pg_t}^2_\star  + \frac{\rho}{2\eta} \norm{\py_{t}- \y_t}^2 & &\text{(for some $\eta /\rho$)}.\label{eq:proof_dynamic_optimistic_p1}
\end{align}

Second, we employ the Cauchy-Schwarz inequality and the  bounds on $D_\Phi$ and $\nabla\Phi$ from the assumptions, so we have the following:
\begin{align}
   &\sum_{t\in [T]}\parentheses{D_\Phi(\y^\star_t, \py_{t-1}) - D_\Phi(\y^\star_t, \py_{t})}  \\
   &\leq  \underbrace{D_\Phi(\y^\star_1,\y^\pi_0)}_{\text{$\leq D^2$}} +  \sum_{t\in [T]}\parentheses{D_\Phi(\y^\star_{t+1}, \py_{t}) - D_\Phi(\y^\star_t, \py_{t})} \\
   & \leq D^2 + \sum_{t \in [T]} \underbrace{\inner{\nabla \Phi(\y^\star_{t+1}) - \nabla \Phi(\y^\pi_{t})}{ \y^\star_{t+1} - \y^\star_t}}_{\text{Cauchy-Schwarz's Ineq.}}  - \underbrace{D_\Phi(\y^\star_{t} , \y^\star_{t+1})}_{\text{$\geq 0$}}  \\
   &\leq D^2 + \sum_{t \in [T]} \underbrace{\norm{\nabla \Phi(\y^\star_{t+1}) - \nabla \Phi(\y^\pi_{t})}_\star}_{\text{$\leq L_{\Phi}$~\eqref{eq:bounded_grad_map}}}\norm{ \y^\star_{t+1} - \y^\star_t} \\
   & \leq  D^2 + 2 L_{\Phi} \sum_{t \in [T]} \norm{ \y^\star_{t+1} - \y^\star_t}  .\label{eq:proof_dynamic_optimistic_p2}
\end{align}

Combine Eqs.~\eqref{eq:proof_dynamic_optimistic_p0}--\eqref{eq:proof_dynamic_optimistic_p2} to obtain the final upper bound.
\begin{align}
\sum_{t\in \T} \cfunc_t(\y^\star_t) - \cfunc_t(\y_t)& \leq \sum_{t \in [T]}{\g_t} \cdot { \y^\star_t - \y_t } \\
&\leq \frac{1}{\eta}\Big({D^2 + 2 L_{\Phi} \underbrace{\sum_{t \in [T]}  \norm{ \y^\star_{t+1} - \y^\star_t} }_{\text{$P_T$}} }\Big) + \frac{\eta }{2\rho} \sum_{t \in \T}  \norm{\g_t - \pg_t}^2_\star\\
&= \frac{1}{\eta}\Big({D^2 + 2 L_{\Phi} P_T}\Big) + \frac{\eta }{2\rho} \sum_{t \in \T}  \norm{\g_t - \pg_t}^2_\star.
\end{align}

The learning rate $\eta_\star  = \sqrt{{2\rho \parentheses{D^2 + 2 L_\Phi P_T}}/{\sum_{t \in \T} \norm{\vec g_t - \pg_t}^2_\star}}$  yields the tightest upper bound given by the theorem.  
\end{proof}

Note that, to obtain the tightest upper bound on the regret,  we assume that when setting the learning rate we have access to quantities $\rho$, $D$, $L_\Phi$, $P_T$, as well as quantity  $l_T \triangleq \sqrt{\sum_{t \in \T} \norm{\vec g_t - \pg_t}^2_\star}$. The first three are problem-specific bounds and are usually easy to determine given the problem instance (see also Appendix~\ref{appendix:setups_omd}). Parameter $P_T$ is a usual problem input, as it bounds the power of the (dynamic) adversary; assuming that it is known is standard (see, e.g., \citet{zinkevich2003,besbes2015non}). Guarantees however can be still provided even if it is not a priori known using the same meta-learning approach as the one discussed next (see, e.g., \citet{zhao2020dynamic}). 

Quantity $l_T$ can readily be bounded by $L \sqrt{T} $ in the non-optimistic setting, by Assumption~\ref{asm:oco}. 
However, to obtain a tighter bound, but also when operating in the optimistic setting, this parameter can be learned through a meta-algorithm (see also \citet{hazan2016introduction,shalev2012online, mcmahan2017survey}). In short, one can execute in parallel multiple instances of OMA algorithms configured for different $l_T$ with an appropriately chosen range and resolution,
 and  frame an expert problem to learn the best expert (policy). The meta problem is a standard prediction with expert advice problem, and can be tackled with well understood parameter-free and computationally efficient learning algorithms~\cite{hazan2016introduction,shalev2012online, mcmahan2017survey}. Alternatively, one could consider adaptive online algorithms~\cite{shalev2012online} that adapt to the observed gradients with a slight degradation in performance. 

Note that in our implementation and experiments, we perform a grid search for $\eta$ values (see Sec.~\ref{sec:experiments}). We also implemented the meta-learning algorithm for the optimistic setting (see Sec.~\ref{appendix:experiments}), also with a grid search for $\eta$. 


\subsection{Follow the Regularized Leader}
\begin{algorithm}[t]
\caption{Follow the Regularized Leader: $\mathrm{FTRL}_{ \eta,R, \Y}(({\y_s})^t_{s=1}, ({\cfunc_s})^{t}_{s=1})$\label{alg:ftrl}}
    \begin{algorithmic}[1]
        \Require Learning rate $\eta \in \reals_{\geq 0}$, regularization function $R: \X \to \reals$, decision set $\Y$, decisions $\y_s$, reward functions $\cfunc_s: \Y \to \reals$ for  $s \in [t]$ 
        \State $\vec g_s \gets \partial_{\y}  \cfunc_s(\y_s)$  
\State $\vec{y}_{t+1}\gets \argmin_{\y\in \Y } \, \,\eta\, \sum^t_{s=1}\g_s \cdot  \y + R(\y)$ \Comment{{ Project new point onto feasible region $\mathcal{X}$}}
        \State \Return $\y_{t+1}$
    \end{algorithmic}
\end{algorithm}

\label{s:desc_ftrl}
It is well known that \emph{follow the leader} (FTL) policy, also known as \emph{fictitious play} in economics,  fails to provide sublinear regret guarantee~\cite{hazan2016introduction}. The FTL policy selects na\"ively the decision that would minimize the past costs, i.e., $ \x_{t+1} \in \argmin_{\y \in \Y}\sum^{t}_{s=1} f_s(\y)$. This policy fails against an adversary; an adversary can drive the policy to change severely its decision from one timeslot to another. The policy can be modified to exhibit more \emph{stability}~\cite{mcmahan2017survey} through regularization, by the introduction of regularization function $R: \X \to \reals$. The update rule is then modified to the following: 
\begin{align}
    \x_{t+1} = \argmin_{\y \in \Y}\,\, \eta \sum^{t}_{s=1} \g_s \cdot \y + R(\y).
\end{align}
We assume that the
regularization function satisfies the  following properties:
\begin{assumption}\label{asm:regularizer}
    The map $R : \X \to \reals$  is 1-strongly convex w.r.t. a norm $\norm{\,\cdot\,}$, smooth, and twice differentiable.
\end{assumption}
The pseudocode of the algorithm is provided in Alg.~\ref{alg:ftrl}. 
\paragraph{Regret Guarantee.} The regret of the FTRL policy is given by:
\begin{theorem}(\cite[Theorem~5.2.]{hazan2016introduction})
     Consider an OCO setting with a convex decision space $\Y$,  concave rewards $\cfunc \in \hat{\F}$ satisfying Asm.~\ref{asm:oco} from a set $\hat \F$ selected by an adversary. The regret of FTRL in Alg.~\ref{alg:ftrl} against a fixed decision $\y_\star \in \Y$ is upper bounded as follows
\begin{align}
\sum^T_{t=1} \cfunc_t (\y^\star) - \sum^T_{t=1} \cfunc_t (\y_t) \leq 2\eta \sum_{t \in [T]} \norm{\g_t}^2_\star + \frac{R(\y^\star) - R(\y_1)}{\eta}.
\end{align}
When  the learning rate rate $\eta_\star  = \sqrt{\frac{2\parentheses{R(\y^\star) - R(\y_1)}}{\sum_{t \in [T]} \norm{\g_t}^2_\star}} $ is selected, the following holds:
\begin{align}
    \sum^T_{t=1} \cfunc_t (\y^\star) - \sum^T_{t=1} \cfunc_t (\y_t)  \leq 2\sqrt{2 \parentheses{R(\y^\star) - R(\y_1)} \sum_{t \in [T]} \norm{\g_t}^2_\star}.
\end{align}
\end{theorem}
\subsection{Derivation of Theorem~\ref{thm:oco}}\label{app:dervofthm:oco}
Assumption~\ref{asm:oco} implies that the supergradients of the reward functions are bounded. Note that when the reward functions are Lipschitz continuous their supergradients are bounded at any point~\cite{shalev2012online}. Moreover, the diameter of the constraint set is bounded by a constant. Thus, Theorem~\ref{thm:general_regret_dynamic+optimistic}  implies that in the case of OMA in Alg.~\ref{alg:oma} configured with an appropriately selected mirror map satisfying  Asms.~\ref{asm:mirror_map}--\ref{asm:bounded_gradient_mirror_map}  (e.g. OGA by selecting $\Phi(\x) = \frac{1}{2} \norm{\x}_2^2$) the regret bound is given by
\begin{align}
   \mathrm{regret}_{T} (\vec \P_\Y) \leq D \sqrt{{2/\rho \sum^T_{t=1} \norm{\vec g_t }^2_\star}} \leq D L {\frac{2 T}{\rho} }  = \BigO{\sqrt T}.
\end{align}
Furthermore, the regret of FTRL for an appropriately selected regularizer satisfying Asm~\ref{asm:regularizer} (e.g., Euclidean regularizer $R(\x) = \frac{1}{2} \norm{\x}_2^2$) in Alg.~\ref{alg:ftrl} is given by
\begin{align}
    \mathrm{regret}_{T} (\vec \P_\Y) \leq  2\sqrt{2 \parentheses{R(\y^\star) - R(\y_1)} \sum_{t \in [T]} \norm{\g_t}^2_\star} \leq  2 L \sqrt{2 \parentheses{R(\y^\star) - R(\y_1)} T} = \BigO{\sqrt{T}}.    
\end{align}

\section{Weighted Threshold Potential Functions}
\label{appendix:WTP}

\subsection{Submodularity and Monotonicity} 
 Consider a threshold potential $\Psi_{ {b,\vec w, S}}(A)$ defined in Eq.~\eqref{eq:budget-additive} for $A \subseteq V = [n]$.  We omit the parameters ${b, \vec{w}, S}$ when they are clear from the context. We split our proof into two parts, by first proving submodularity and then monotonicity.

\emph{Submodularity.}  Given a finite set $V$, the function is said to be submodular if and only if for every $A \subseteq V$ and $i,j \in V \setminus A$ it holds
    \begin{align}
       f(A \cup \set{i, j}) - f(A \cup \set{j})  &\leq  f(A \cup \set{i}) - f(A).
    \end{align}   
We have the following 
\begin{align}
    \Psi(A \cup \set{i}) - \Psi(A)&= \begin{cases}
        0 &\text{for $\sum_{k \in S \cap A } w_k \geq b$},\\
        \min\set{b - \sum_{k \in S \cap A } w_k , w_i} &\text{otherwise.}
    \end{cases}
\label{eq:p1_sub}
\end{align}
It also holds
\begin{align}
    \Psi(A \cup \set{i, j}) - \Psi(A \cup \set{j})&= \begin{cases}
        0 &\text{for $\sum_{k \in S \cap A } w_k \geq b - w_j$},\\
        \min\set{b - \sum_{k \in S \cap A } w_k - w_j, w_i} &\text{otherwise.}
    \end{cases}
\label{eq:p2_sub}
\end{align}
Combine Eqs.~\eqref{eq:p1_sub} and \eqref{eq:p2_sub} to obtain
\begin{align}
    \Psi(A \cup \set{i, j}) - \Psi(A \cup \set{j}) \leq  \Psi(A \cup \set{i}) - \Psi(A). 
\end{align}
Thus, we proved that the set function $g$ is submodular. The function of interest  in Eq.~\eqref{eq:wtp} can be expressed for every $S \subseteq V$as
\begin{align}
    f(A) = \sum_{\ell \in C} c_{\ell} \Psi_{b_{\ell}, \vec {w}_{\ell}, S_{\ell}} (A).\label{eq:mono_p1}
\end{align}
The function $f$ is a weighted sum with positive weights of submodular function; therefore, from the definition $f$ is also submodular.

\emph{Monotonicity.} Consider the sets $A \subseteq B \subseteq V$. We have the following 
\begin{align}
    \Psi(B) = \min\set{b, \sum_{k \in S \cap B} w_k} = \min\set{b, \sum_{k \in S \cap A} w_k +\sum_{k \in S \cap \parentheses{B \setminus A}} w_k}. 
\end{align}
Note that $\sum_{k \in S \cap \parentheses{B \setminus A} } w_k \geq 0$. Thus, it holds
\begin{align}
    \Psi(B) \geq \min\set{b, \sum_{k \in S \cap A} w_k} = \Psi(A).\label{eq:mono_p2}
\end{align}
We conclude the proof by noting that $f$ is a weighted sum with positive weights of monotone functions~\eqref{eq:mono_p2} as expressed in Eq.~\eqref{eq:mono_p1}.

\subsection{Applications} \label{sec:examples}

\subsubsection{Weighted Coverage Functions}\label{appendix:wcf}

 For $V =[n]$, let $\{S_\ell
\}_{\ell \in C}$ be a collection of subsets of $V$, and $c_\ell \in \reals_{\geq 0}$ be a weight associated to each subset $S_\ell$. A weighted set coverage function $f: \{0, 1\}^n \rightarrow \reals_{\geq 0}$ receives as input an $\x\in\{0,1\}^n$ representing a subset of $V$ and accrues reward $c_\ell$ if $S_\ell$ is ``covered'' by an element in $\x$: formally,  \begin{align} f(\x) = \sum_{\ell \in C} c_\ell \left(1 - \prod_{j \in 
    S_\ell} (1-x_j)\right).\label{eq:wcf}\end{align} 
  \begin{observation}
    The weighted coverage function $f$ 
   belongs to the class of WTP functions described by Eq.~\eqref{eq:wtp}; in particular:
  \begin{align}f(\x) &= \sum_{\ell \in C} c_{\ell} \Psi_{1, \vec 1, S_\ell}(\x) &&\text{for every $\x \in \set{0,1}^n$.}\end{align}
  \end{observation}
 This follows from the simple fact that
\begin{align}
     1-\prod_{i\in S}(1-x_i) = \min\set{1,\sum_{i\in S} x_i}, \quad \text{for all}~\x\in\{0,1\}^n, S\subseteq V.
\end{align}

Classic problems such as 
 influence maximization~\cite{kempe2003maximizing} and facility location~\cite{krause2014submodular,frieze1974cost} can all be expressed via the maximization of weighted coverage functions subject to matroid constraints (see also~\cite{karimi2017stochastic}). Another such example is cache networks~\cite{ioannidis2016adaptive}, which have also been studied in the online setting~\cite{li2021online}. 
Our online setting naturally captures different epidemics over the same graph in influence maximization and dynamic tasks in facility location~\cite{karimi2017stochastic}, as well as time-variant content item requests in cache networks~\cite{li2021online}. We describe each of these settings in more detail below.

\paragraph{Influence Maximization \cite{kempe2003maximizing, karimi2017stochastic}.} In the classic paper by Kempe et al.~\cite{kempe2003maximizing}, a network is represented as a graph $G(V,E)$ and a probabilistic model like Independent Cascades~\cite{goldenberg2001talk} or Linear Threshold~\cite{granovetter1978threshold} is used to emulate the propagation of an epidemic over $G$, starting from a seed set $S\subseteq V$. The objective function is then the total number of nodes in $V$ reached by the epidemic. This is hard to compute so, in practice, the objective is generated by sampling different \emph{reachability graphs} from the underlying probabilistic epidemic model, and estimating the objective in expectation. This estimation is indeed a WC function: in particular, $C=V$, and there is one set $S_i$ for every node $i$, comprising the nodes that would infect $i$ if they are placed in the seed set, and $c_i=1$.

The influence maximization setting maps very well to the online submodular maximization problem we consider, particularly in the bandit setting. Different reachability graphs/coverage functions correspond to different instances the epidemic propagation model at each timeslot or, more broadly, different epidemics, potentially with time-varying statistics. The decision maker gets to pick the seeds subject to, e.g., a cardinality/uniform matroid constraint, and observes either the entire reachability graph (in the full information setting), or just the final number of nodes infected (in the bandit setting).

\paragraph{Facility Location \cite{krause2014submodular, frieze1974cost, karimi2017stochastic}.} In the classic facility location problem, we are given a complete weighted bipartite graph $G(V\cup V',E)$, where $V=[n]$, $E= V\times V'$, with non-negative weights $w_{v,v'}\geq 0$, $v\in V$,$v\in V'$. 
The decision maker selects a subset $S \subset V$ and each $v'\in V$ responds by selecting the $v \in S$ with the highest weight $w_{v,v'}.$ The goal is to
maximize the average weight of these selected edges, i.e. to maximize
\begin{align}
f(S) = \frac{1}{|V'|}
\sum_{v'\in V}\max_{v\in S} w_{v,v'}
\end{align}
given some matroid constraints on $S$. Set 
$V$ can be considered a set of facilities and $V'$ a set of customers or tasks to be served, while $w_{v,v'}$ is the utility accrued if facility $v$ serves tasks $v'$. Karimi et al.~\cite{karimi2017stochastic} note that this is also an instance of the \emph{Exemplar-based Clustering} problem, in which
$V = V'$ is a set of objects and $w_{v,v}$ is the similarity (or inverted distance) between objects $v$ and $v$,
and one tries to find a ``summary'' subset $S$ of exemplars/centroids highly similar to all elements in $V$. Karimi et al.~also show that this objective is a weighted coverage function, of the form \eqref{eq:wcf}, because, for $\x\in\{0,1\}^n$ the characteristic vector of $S$, we have 
\begin{align}
\begin{split}
 \max_{v\in S} w_{v,v'} &=  \sum_{i=1}^{n-1} \left(w_{\pi_j,v'}-w_{\pi_{i+1},v'}\right)\left(1-\prod_{j=1}^i(1- x_{\pi_j})\right) +  w_{\pi_n,v'}\left(1-\prod_{j=1}^n (1-x_{\pi_n})\right)\nonumber
\end{split}
\end{align}
where $\vec \pi:V\to V$ is a permutation such that:
$w_{\pi_1,v'}\geq w_{\pi_2,v'} \geq \ldots \geq w_{\pi_n,v'}.$

This also maps well to our online setting. In the context of facility location, we wish to choose facilities per timeslot in the presence of time-varying, online arrivals of customers/tasks, and in the context of exemplar clustering we wish to choose summaries/exemplars that are good across a potentially time-varying sequence of sets of points. Both fit very well in the full-information setting, as all distances/similarities are immediately observable/computable whenever the tasks/points are revealed.

\paragraph{Cache Networks \cite{ioannidis2016adaptive,li2021online}.} Ioannidis and Yeh \cite{ioannidis2016adaptive} consider a network of caches, each capable of storing at most a constant number of content items of equal size. Item requests are routed over paths in the graph, and are satisfied upon hitting the first cache that contains the requested item. The objective is to determine a mapping of items to caches that minimizes the aggregate retrieval cost. 

Formally, the network is represented as a directed graph $G(V,E)$, with $V=[n]$,
 Let $\catalog$ be the set of content items available.
Each node $v\in V$ has capacity $\capacity_v\in \naturals$:  exactly $\capacity_v$ content items in $\catalog$ can be stored in $v$. 
Let 
$x_{vi}  \in \{0,1\}$ 
the variable indicating whether $v\in V$ stores  item $i \in \catalog$, and let $\x =[x_{vi}]_{v\in V,i \in \catalog}\in \{0,1\}^{|V|\times |\catalog|}$. 
Note that the capacity constraints imply that
$\textstyle\sum_{i\in \catalog} x_{vi} \leq c_v$, for all $v\in V.$ This set of constraints defines a partition matroid \cite{ioannidis2016adaptive}.
Each item $i$ in the catalog $\catalog$ is associated with a fixed set of \emph{designated servers} $\servers_i\subseteq V$, that always store $i$; w.l.o.g., it is assumed that memory used is outside the cache capacity (as it is fixed and outside the optimization problem). A request \emph{request} $r$ is a pair $(i,p)$ where $i\in \catalog$ is the item requested, and $\ppath$ is the path traversed to serve this request. The terminal node in the path is always a designated source node for $i$, i.e., if $|p|=K$, 
$p_K \in S_i.$ An incoming request $(i,p)$ is routed over the network $G$  following path $p$, until it reaches a cache that stores $i$. The objective is to find an allocation $\x$ that maximizes the \emph{caching gain}, i.e., the reduction in routing costs compared to serving the request from the final designated server. Given request $r=(i,p)$, this is captured by the objective:
\begin{align}
f(\x) = \sum_{k=1}^{|p|-1} \!\!w_{p_{k+1}p_k}\left(1\!-\! \prod_{k'=1}^k (1\!-\!x_{p_{k'}i}) \right),
\label{gain}
\end{align}
where $w_{u,v}\geq 0$ is the cost of transfering an item over edge $(u,v)$. This objective is clearly of the form in Eq.~\eqref{eq:wcf}.


This problem also maps very well to the online submodular maximization setting: different functions, varying through time, naturally correspond to different requests $r=(i,p)$ for different contents, traversing different paths across the network. Again, the full information (rather than bandit) setting is natural here, as the entire function is revealed once the item and the path are revealed; both are needed at execution time, to identify what is requested and how the request should be routed. The online submodular maximization setting was recently studied by Li et al.~\cite{li2021online}, who used a distributed variant of the algorithm by Streeter et al.~\cite{streeter2008online} to obtain $\alpha$-regret  guarantees (see~Table~\ref{table:compare} for a comparison of \cite{streeter2008online} with RAOCO).

\subsubsection{Similarity Caching} Similarity caching~\cite{sisalem2022ascent} is an extension of the so-called paging problem~\cite{king72,flajolet92}.
The objective is to efficiently use local storage memory and provide approximate (similar) answers to a query in order to reduce retrieval costs from some distant server. Formally,
consider a catalog of files  represented by the set $V = [n]$, and an augmented catalog represented by the set $\hat{V} = [2n]$.  Item $i \in V$ is the local copy of a remote file $i+n \in \hat V$. Let $c(q,i) \in \reals$ be the cost of responding to query $q \in V$ with a similar file $i\in \hat V$: this can be determined, e.g., by the similarity or distance between the files, in some appropriately defined feature space. Si Salem et al.~\cite{sisalem2022ascent} show that the reward received under this setting is given by 
\begin{align}
  f(\x) = \sum_{l=1}^{C_q } \parentheses{c(q, \pi_{q,l+1}) - c(q, \pi_{q,l})} \min \left\{k - \card{S_l}, \sum_{j \in [l] \setminus S_l} x_{\pi_{q,j}} \right\},
\end{align}
for some $C_q \in \naturals$, subset $S_l \subset V$ where $\card {S_l}< k$, and permutation $\vec\pi_q: \hat V \to \hat V$ satisfying $c(q, \pi_{q,l+1}) - c(q, \pi_{q,l}) \geq 0 $ for $l \in [C_q]$. This is indeed a WTP function:
\begin{observation}
The similarity caching objective belongs to the class of WTP functions defined by Eq.~\eqref{eq:wtp}. In particular, $f(\x) = \sum^{C_q}_{l=1}  \parentheses{c(q, \pi_{q,l+1}) - c(q, \pi_{q,l})}  \Psi_{ k - \card{S_l},\vec 1, S_l} (\x)$. 
\end{observation}
Note that the above functions \emph{cannot} be described as weighted coverage functions (Eq.~\eqref{eq:wcf}), and serve as an example of how WTP generalizes the WCF class (see also the discussion in Section~\ref{s:generality_wtp} below). Moreover, this problem also fits the full information OSM scenario very well: queries $q$ arrive online and, once revealed, the full reward function is determined as well.

\subsubsection{Quadratic Set Functions} \label{appendix: quadratic}
Motivated by applications such as 
price optimization on the basis of demand forecasting~\cite{ito2016large}, scheduling~\cite{Skutella2001}, detection
in multiple-input multiple-output channels in communication systems~\cite{MIMO} and team formation~\cite{lappas2009finding,boon2003team}, we consider quadratic monotone submodular 
functions $f: \{0, 1\}^n \rightarrow \reals_{\geq 0}$ of the form 
\begin{align}
f(\x) = \vec{h}^\intercal \x + \frac{1}{2}\x^\intercal \vec{H}\x,
\end{align}
where, w.l.o.g., $\H$ is symmetric and has zeros in the diagonal. Note that $f$ is monotone and submodular if and only if : (a) $\H \leq \vec{0}_{n \times n}$, and (b) $\h + \H \x \geq \vec{0}_{n \times 1}$. Thinking of this as the performance of a team, the linear/modular part of the function captures the strength of each team member, and the quadratic part captures the complementarity/overlaps between their skills.

\begin{observation}
The quadratic function $f$  belongs to the class of WTP functions~\eqref{eq:wtp}, i.e.,
 $f(\x) =  \Psi_{\infty,  \h + \H \vec{1}, [n]}(x) +
    \sum_{i=1}^{n-1} \sum_{j=i+1}^n
        (-H_{i,j}) \Psi_{1, \vec{1}, \set{i,j}}(\x)
$, for every $\x\in\set{0,1}^n$.
\end{observation}

To see this, observe that the following equality holds for all $x, y \in \{0,1\}$
    \begin{align}
    \label{lemma: quadratic}
        1 - (1-x)(1-y) = x + y - xy = \min \{1, x+y\}
    \end{align}
    We have that
    \begin{align*}
        \x^\intercal \H \x 
            &= \sum_{i=1}^n \sum_{j=1}^n x_i x_j H_{i,j} \overset{\eqref{lemma: quadratic}}{=} \sum_{i=1}^n \sum_{j=1}^n (x_i + x_j - \min \{1, x_i + x_j\}) H_{i,j} \\
            &= 2\x^\intercal \H \vec{1} + 2\sum_{i=1}^{n-1} \sum_{j=i+1}^n 
            (-H_{i,j})\min \{1, x_i + x_j\}
    \end{align*}
    The last equality holds because $\H$ is symmetric and hollow. Substituting the above expression for $\x^\intercal \H \x$ into 
    $
        f(\x) = \vec{h}^\intercal \x + \frac{1}{2}\x^\intercal \vec{H}\x
    $
    yields the observation. 

This also matches  the OSM setting well. For example, different team performance functions can correspond to the score/quality attained by the team at a sequence of progressively revealed tasks. In the bandit setting, only the final performance score is revealed, while in the full information setting, this is decomposed into the individual contributions and their pairwise interactions.

Finally, note that, by definition, the degree $\Delta_f$ of quadratic set functions is at most 2. As a result, the $\alpha$-regret we attain in this setting has $\alpha=\frac{3}{4}$, by Lemma~\ref{proposition:sandwich2}.

\subsection{Generality of the WTP Class}
\label{s:generality_wtp}
\paragraph{SCMM $\subset$ WTP.} Stobbe and Krause~\cite{stobbe2010efficient} defined the class of sums of concave functions composed of non-negative modular functions plus an arbitrary modular function (SCMM) taking the form
\begin{align}
    f(\x) = \vec c \cdot \x + \sum_{j\in J} \phi_j(\vec w_j\cdot\x) \label{eq:scmm}
\end{align}
where $\vec c$, $\vec w_j \in \reals^n$ and $\vec 0\leq \vec w_j \leq \vec1$ (element-wise) and $\phi_j:[0,\vec w_j \cdot \vec 1]\to \reals $ are arbitrary concave functions. They then show that the SCMM class can be expressed in terms of threshold potentials.
\begin{proposition} (Stobbe and Krause~\cite{stobbe2010efficient}) For $\phi \in C^2([0,M])$,
    \begin{align}\label{eq:threshold_relation}
        \phi(v) = \phi(v) + \phi'(M) v -\int^M_{0} \min \set{v, y} \phi''(y) dy\quad \text{for $v \in [0,M]$}.
    \end{align}
\end{proposition}
Discretized versions of Eq.~\eqref{eq:threshold_relation} can be efficiently computed when $v$ belongs to a finite set, e.g., 
\begin{align}
\phi\parentheses{\sum_{i \in S} x_i}&=\phi(0)+(\phi(|S|)-\phi(|S|-1)) \parentheses{\sum_{i \in S} x_i}\nonumber\\&+ \sum_{k=1}^{|S|-1}(2 \phi(k)-\phi(k-1)-\phi(k+1)) \min\set{k, \sum_{i \in S} x_i}.
\end{align}

\paragraph{A Hierarchy of Submodular Functions.} Consider now a finite set $V = [n]$, $c_{i} \in \reals_{\geq 0}$, $b_{i} \in \reals_{\geq 0} \cup\set{\infty}$, $w_{i,j} \in \reals_{\geq 0}$, and $S_{i}\subseteq V$ for $i \in C$, $j \in V$. Consider the following classes of submodular functions:
\begin{enumerate}

 \item Weighted cardinality truncations (WCT) functions which take the form $f(\x) = \sum_{i \in C} c_i \Psi_{b_i,\vec 1, S_i}(\x)$.
\item Weighted coverage (WC) functions which take the form $f(\x) = \sum_{i \in C} c_i \Psi_{1,\vec 1, S_i}(\x)$.
\item Facility location (FL) functions. $f(\x) =\sum_{i \in C} \max\set{x_j w_{i,j}: j \in [n]}$. It is a subclass of weighted coverage functions~\cite{karimi2017stochastic}.
\end{enumerate}

The following chain  relationship between these classes of functions holds: 
\begin{align}
    \mathrm{FL} \subset \mathrm{WC} \subset \mathrm{WCT} \subset \mathrm{SCMM} \subset \mathrm{WTP}~\eqref{eq:budget-additive} \subset \text{Submodular-Functions}.
\end{align}
The strict inclusions $\mathrm{WC} \subset \mathrm{WCT} \subset \mathrm{SCMM}$ was shown by Dolhansky and Bilmes~\cite{dolhansky2016deep}, and the strict inclusion $\mathrm{SCMM} \subset \mathrm{WTP}$ was shown by Stobbe and Krause~\cite{stobbe2010efficient}.

\section{OMA under WTP Functions and Matroid Polytopes and Derivation of Bounds in Table~\ref{table:compare}}
\label{appendix:setups_omd}
For completeness, in this section, we consider set functions in the class of WTP functions~\eqref{eq:wtp}, and discuss implementation details (i.e., how OMA is instantiated) over the concave relaxations of functions in this class. We note that, to obtain the most favorable dependence of regret constants on problem parameters, we need to use an OMA policy coupled with a negative entropy mirror map. However, we prove a negative result: OMA with negative entropy has linear dynamic regret when applied ``out-of-the-box'' (see Proposition~\ref{proposition:linear_dynamic_regret} in Section~\ref{s:negative_result}).
Nevertheless, we are able to overcome this issue by extending the so-called fixed share update rule, originally proposed for optimization over the simplex~\cite{cesa2012mirror,herbster1998tracking}, to matroid polytopes (see Proposition~\ref{proposition:negative_entropy}).

We provide a summary of notation specific to this section in Sec.~\ref{s:notation_summary},  the negative result in Section~\ref{s:negative_result},  the two setups of OMA in Alg.~\ref{alg:oma}: the shifted negative entropy in Sec.~\ref{s:shifted_oma} and the Euclidean in Sec.~\ref{s:euclidean} (OGA) setup, the characterization of the supergradient of the concave relaxation of WTP functions in Sec.~\ref{appendix:additional_related_work}, and the derivation of the regret bounds and a comparison with related work in Sec.~\ref{appendix:supergradient_computation}.

\subsection{Notation Summary}\label{s:notation_summary}
We denote by $\charvec_S = \parentheses{\mathds{1} \parentheses{i \in S}}_{i \in V} \in \set{0,1}^n$ the characteristic vector of a subset $S \subseteq V$ for $\mathds{1}\parentheses{\chi} \in \set{0,1}$ takes the value $1$ when $\chi$ is true, and $0$ otherwise. We consider  the convex set $\Y$ as the matroid basis polytope of some matroid $\M (V, \I)$ with rank $r \in \naturals$, i.e., the set $\Y$ is given by
\begin{align}
    \mathcal Y = \conv{\set{\vec 1_S: S \in \mathcal{I}}} \subset \set{\y \in [0,1]^n: \norm{\y}_1 = r}.
\end{align}  
We denote by $d_\infty$ the upper bound on discrete derivatives of $f$, which is given by
\begin{align}
      d_\infty \triangleq \max_{i \in [n]} \set{ f(\set{i}) - f(\emptyset)} =\max_{i \in [n]} \set{ f(\set{i})}.
\end{align}

\subsection{A Negative Result}\label{s:negative_result}
OMA configured with the negative entropy mirror map $\Phi(\y) = \sum_{y \in [n]} y_i \log(y_i)$  which enjoys sublinear static regret over the simplex $\Y = \Delta_n \subset [0,1]^n$ does not maintain such a  guarantee for dynamic regret. We prove this negative result in Proposition~\ref{proposition:linear_dynamic_regret} through a simple construction; since $L_\Phi = \infty$ in this setting, then Asm.~\ref{asm:bounded_gradient_mirror_map} is also necessary for OMA to yield a sublinear dynamic regret policy. Formally: 
\begin{proposition}\label{proposition:linear_dynamic_regret}
    OMA in Alg.~\ref{alg:oma} configured with negative entropy mirror map $\Phi(\y) = \sum_{y \in [n]} y_i \log(y_i)$ has linear dynamic regret ($P_T > 0$) under simplex constraint set $\Y = \Delta_n \subset [0,1]^n$ for any $\eta \in \reals_{> 0}$. 
\end{proposition}
\begin{proof}
    Consider $n=2$, and the following rewards $f_t(\x) = x_1$ for $t \leq T/2$, and $f_t(\x) = x_2$ for $T/2\leq t \leq T$. The dynamic optimum selects $\x_t = (1,0)$ for $t \leq T/2$, and $\x_t = (0,1)$ for $T/2\leq t \leq T$ incurring a total reward of $T$. It is easy to verify that the policy's state is given by $\x_t = \parentheses{e^{(t-1)\eta}/(e^{(t-1)\eta} +1),1/(e^{(t-1)\eta} +1) }$ for $t \leq T/2$ and  $\x_t = \big(e^{(T/2 - 1)\eta} / (e^{(t-1-T/2)\eta} + e^{(T/2-1)\eta})$, $e^{(t-1-T/2)\eta} / (e^{(t-1-T/2)\eta} + e^{(T/2-1)\eta})\big)$ for $T/2+1 \leq t \leq T$. Thus, the regret is given by \begin{align}\textstyle\mathrm{regret}_{T, 2} (\vec \P_{\Delta_2})= T - \sum^{\frac{T}{2}}_{t=1} \underbrace{\frac{e^{(t-1)\eta}}{e^{(t-1)\eta} + 1}}_{\text{$\leq 1$}} + \underbrace{\frac{e^{(t-1)\eta}}{e^{\left({\frac{T}{2}}-1\right)\eta} + e^{(t-1)\eta}} }_{\text{$\leq 1/2$}} \geq T/4 = \Omega (T).
    \end{align}
    This concludes the proof.
\end{proof}

We overcome this negative result next, by extending the so-called fixed share update rule, originally proposed for optimization over the simplex~\cite{cesa2012mirror,herbster1998tracking}, to matroid polytopes (see Proposition~\ref{proposition:negative_entropy}).

\subsection{Shifted Negative Entropy Setup}\label{s:shifted_oma}
Let  $\gamma \in[0, \sqrt{e^{-2} + 1/4} - 1/2]$ be a shifting parameter where $\sqrt{e^{-2} + 1/4} - 1/2\approx 0.12$. The shifted negative entropy setup of OMA, defined as 
\begin{align}
 \Phi (\y) &\triangleq   \eqmathbox[left][l]{ \sum_{i \in [n]}  (y_i +\gamma) \log(y_i + \gamma)} & &\eqmathbox[right][l]{\text{for $\y \in \mathcal{D} \triangleq (-\gamma, \infty)^n$}}.\label{eq:shifted_entropy}
\end{align}
Note that it holds $\nabla \Phi(\y) = \parentheses{1 + \log(y_i + \gamma)}_{i \in [n]}$ and $\parentheses{\nabla \Phi}^{-1}(\hat\y) = \parentheses{\exp\parentheses{\hat y_i - 1} - \gamma}$ for $\y \in \mathcal{D}$ and $\hat \y \in \reals^n$. Thus, at timeslot $t$, the steps in lines~2--4 in Alg.~\ref{alg:oma} correspond to the following update rule
\begin{align}
    z_{t+1, i } &= y_{t,i} e^{\eta_t g_{t,i}} + \gamma \parentheses{e^{\eta_t g_{t,i}} -1} & &\text{for $i \in [n]$}.
\end{align}
Note that this update rule has the same form as the fixed share update rule~\cite{cesa2012mirror,herbster1998tracking}. However, this policy extends beyond the simplex.

It is easy to see why the \emph{pure} multiplicative update rule obtained by OMA in Alg.~\ref{alg:oma} configured with the negative-entropy mirror map fails to provide sublinear dynamic regret guarantees through the constructed scenario in the proof of Proposition~\ref{proposition:linear_dynamic_regret}. The adversary can drive the state of the algorithm to arbitrary close to $(1,0)$, but when the best decision changes to $(0,1)$ the algorithm fails to adapt quickly purely from a multiplicative update. This is why the additive term controlled by $\gamma$ is needed, which also allows us to control $L_\Phi \leq \log(1/\gamma)$.

We summarize the properties satisfied by this  setup in the following proposition.

\begin{proposition}\label{proposition:negative_entropy}
Consider a matroid basis polytope $\Y$ with rank $r$ defined in Sec.~\ref{appendix:setups_omd} and $\y_1 = \Pi^\Phi_{\Y \cap \D} ({\rank}/{n} \vec 1)$. Then,  Asms.~\ref{asm:mirror_map}--\ref{asm:bounded_bregman_divergence} are satisfied under the mirror map defined in Eq.~\eqref{eq:shifted_entropy} ($\gamma \in [0,\sqrt{e^{-2} + 1/4} - 1/2]$) under the bounds in Table~\ref{tab:bounds}.
\end{proposition}
\begin{proof}

\noindent\emph{{Strong Convexity.}}
Since $\Phi$ is twice differentiable, it is sufficient~\cite{shalev2012online} to show that $\y^\intercal \nabla^2 \Phi(\vec w) \y  \geq \frac{1}{r + \gamma n} \norm {\y}^2_1$ for every $\vec w \in \Y$ and $\y \in \reals^n$, so that  $\Phi$ is $\frac{1}{r + \gamma n}$ strongly convex w.r.t. the norm $\norm{\,\cdot\,}_{1}$ over the set $\Y$.
\begin{align}
        \y^\intercal \nabla^2 \Phi(\vec w) \y &= \sum_{i \in [n]}\frac{ y^2_i }{w_i + \gamma} =\textstyle \frac{1}{\norm{\vec w + \gamma \vec 1}_1}\parentheses{\sum_{i \in [n]} w_i + \gamma } \parentheses{\sum_{i \in [n]} \frac{y^2_i}{w_i + \gamma }} \\
        &\geq  \frac{1}{\norm{\vec w + \gamma \vec 1}_1} \parentheses{\sum_{i \in  [n]} \sqrt{w_i + \gamma } \frac{y_i}{\sqrt{w_i + \gamma}}}^2 \geq \frac{1}{r + \gamma n} \norm{\y}^2_1. 
    \end{align}
The inequality follows from Cauchy–Schwarz inequality.

\noindent\emph{{Bregman Divergence Bound.}} Consider a point $\y_1 = \Pi^\Phi_{\Y \cap \D} ({\rank}/{n} \vec 1) \in \Y$, thus for $\y \in \Y$ the Bregman divergence $D_\Phi(\y,\y_1)$ is given by 
\begin{align}
    D_\Phi(\y,\y_1) &\leq D_\Phi(\y,{\rank}/{n} \vec 1)  = \sum_{i \in [n]} (y_i + \gamma) \log\parentheses{\frac{y_i + \gamma}{\rank/n + \gamma}}  + \rank - y_i
    \leq \sum_{i \in [n]} (y_i + \gamma) \log\parentheses{\frac{y_i + \gamma}{\rank/n + \gamma}}\nonumber  \\
    &\leq r (1+\gamma) \log\parentheses{\frac{(1+\gamma)n}{\rank + \gamma n}}.
\end{align}
The upper bound is obtained by setting $\y \in \Y \cap \set{0,1}^n$.

\noindent \emph{{Mirror Map's  Lipschitzness.}} The gradient of the mirror map $\Phi$ is upper bounded by $L_\Phi =\log(1/\gamma)-1$ under the $\norm{\,\cdot\,}_\infty$ norm.

It is straightforward to check that all the properties in Asm~\ref{asm:mirror_map} are satisfied. 
\end{proof}

\begin{corollary}  \label{corollary:1} Under WTP reward functions in Eq.~\eqref{eq:wtp}  and concave functions in Eq.~\eqref{eq:budget-additive-concave},  RAOCO policy in Alg.~\ref{alg:saoco} $\vec\P_\X$ equipped with an OOMA policy in  Alg.~\ref{alg:ooma} configured with shifted negative entropy mirror map~\eqref{eq:shifted_entropy} for $\gamma \in[0, \sqrt{e^{-2} + 1/4} - 1/2]$ and a fixed learning rate $\eta =\sqrt{\frac{2 \parentheses{r (1+\gamma) \log\parentheses{\frac{(1+\gamma)n}{\rank + \gamma n}} + 2 (\log(1/\gamma)-1) P_T }}{r \sum_{t \in \T} \norm{\vec g_t - \pg_t}^2_\star}} $ has the following $\alpha$-regret
\begin{align}
    \alpha\mathrm{-regret} (\vec \P_\X) &\leq   \sqrt{2 (r + \gamma n) \parentheses{r (1+\gamma) \log\parentheses{\frac{(1+\gamma)n}{\rank + \gamma n}} + 2 (\log(1/\gamma)-1) P_T } \sum_{t \in [T]} \norm{\g_t - \g^\pi_t}^2_\infty}.
\end{align}
\end{corollary}
The corollary follows directly from Theorem~\ref{thm:general_regret_dynamic+optimistic} and Proposition~\ref{proposition:negative_entropy}.  Note that when $P_T = 0$ and $\gamma=0$, the regret is $\alpha\mathrm{-regret} (\vec \P_\X) = \BigO{ r \sqrt{ \log(n/r)T}}$. When $P_T > 0$ and $\gamma = 1/n$, the regret is  $\alpha\mathrm{-regret} (\vec \P_\X) = \BigO {\sqrt{r (r \log(n/r) + \log(n) P_T)T}}$.

\subsection{Euclidean Setup}\label{s:euclidean}
The Euclidean setup is defined for the squared Euclidean norm mirror map, i.e., 
\begin{align}
    \Phi (\y) &\triangleq  \eqmathbox[left][l]{\frac{1}{2}\norm{\y}^2_2}  & &\eqmathbox[right][l]{\text{for $\y \in \mathcal{D} \triangleq \reals^n$}.}\label{eq:euclidean}
\end{align}
Note that it holds $\nabla \Phi(\y) = \y$ and $\parentheses{\nabla \Phi}^{-1}(\hat\y) =\hat \y$ for $\y \in \mathcal{D}$ and $\hat \y \in \reals^n$. Thus, at timeslot $t$, the steps in lines~2--4 in Alg.~\ref{alg:oma} correspond to the Online Gradient Ascent update rule
\begin{align}
   \z_{t+1} &= \y_t  + \eta \g_{t} 
\end{align}
The Bregman projection step in line~5 Alg.~\ref{alg:oma} corresponds to the Euclidean projection.
\begin{align}
    \y_{t+1} &= \Pi^{\Phi}_{\Y\cap \D} (\z_{t+1})= \Pi^{\Phi}_{\Y} (\z_{t+1})  =  \argmin_{\y \in \Y} \norm{\z_{t+1}  -\y}^2_2.
\end{align}
We summarize the properties satisfied by this  setup in the following proposition.
\begin{proposition}
Consider a matroid basis polytope $\Y$ with rank $r$ defined in Sec.~\ref{appendix:setups_omd} and $\y_1 = \Pi^{\Phi}_{\Y}{(\vec 0)}$. Then,  Asms.~\ref{asm:mirror_map}--\ref{asm:bounded_gradient_mirror_map} are satisfied for the mirror map defined in Eq.~\eqref{eq:euclidean}  under the bounds in Table~\ref{tab:bounds}.
\end{proposition}
\begin{proof}
    Under this setup the Bregman divergence is given by $D_{\Phi} (\y, \y) = \frac{1}{2}\norm{\y' -\y}^2_2$ and thus $\Phi$ is 1-strongly convex w.r.t. $\norm{\,\cdot\,}_2$ (from the definition of strong convexity). The upper bound on $D_{\Phi} (\y, \y_1)$ for $\y_1  = \Pi^{\Phi}_{\Y}{(\vec 0)}$ is given by $ \frac{1}{2} \norm{\y}^2 \leq \frac{r}{2}$. The gradient of the mirror map $\Phi$ is upper bounded by $L_\Phi = \max_{\y \in \Y} \norm{\y}_\infty = 1$. It is straightforward to check that all the properties in Asm~\ref{asm:mirror_map} are satisfied. 
\end{proof}

\begin{corollary}\label{corollary:2} Under WTP reward functions in Eq.~\eqref{eq:wtp}  and concave functions in Eq.~\eqref{eq:budget-additive-concave},  RAOCO policy in Alg.~\ref{alg:saoco} $\vec\P_\X$ equipped with an OMA policy in  Alg.~\ref{alg:oma} configured with Euclidean mirrr map~\eqref{eq:euclidean} and a fixed learning rate $\eta =\sqrt{\frac{r + 2 P_T}{\sum_{t \in \T} \norm{\vec g_t - \pg_t}_2^2}} $ has the following $\alpha$-regret
\begin{align}
    \alpha\mathrm{-regret} (\vec \P_\X) &\leq  \sqrt{  (r + 4 P_T)\sum^T_{t=1} \norm{\g_t - \g^\pi_t}_2^2}.
\end{align}
\end{corollary}
The corollary follows directly from Theorem~\ref{thm:general_regret_dynamic+optimistic} and Proposition~\ref{proposition:negative_entropy}.

Note that all the OCO regret results in this paper follow from Corollaries~\ref{corollary:1} and~\ref{corollary:2}. In particular, the static non-optimistic regret guarantees are obtained for $P_T = 0$ and $\pg_t = \vec 0$. The dynamic non-optimistic regret guarantees are obtained for $P_T > 0$ and  $\pg_t = \vec 0$. Finally, the non-dynamic optimistic regret guarantees are obtained for $P_T = 0$ and  $\pg_t \neq 0$. 

\subsection{Supergradient Computation}\label{appendix:supergradient_computation}
We conclude our description of the different setups of OMA by providing an explicit expression of the supergradient of the concave relaxation of WTP functions. 
\begin{proposition}
    A supergradient  of $\cfunc: \interval{0,1}^n \to \reals_{\geq 0}$ in Eq.~\eqref{eq:budget-additive-concave} at point $\y \in \interval{0,1}^n$ is given as
    \begin{align}
        \vec g(\y) \triangleq  \parentheses{\sum_{\ell \in C}c_l w_{\ell,j} \mathds{1}\parentheses{j \in S_\ell \land \sum_{j' \in S_\ell} y_{j'} w_{j'} \leq b_\ell}}_{j \in V},\label{eq:subgradient}
    \end{align}\label{proposition:supergradient}
    where $\land$ is a logical \emph{and} operator.
\end{proposition}
\begin{proof}
Consider the concave relaxation $\tilde{\Psi}$ of a threshold potential $\Psi$ in Eq.~\eqref{eq:budget-additive} given by
   \begin{align}
       {\tilde{\Psi}_{b, \vec w, S}}(\y) &\triangleq \min\set{b, \sum_{j \in S } w_j y_j }, &\text{for $\y \in [0,1]^n$.}
   \end{align}
    We omit the parameters ${b,\vec w, S}$ when they are clear from the context.
   We verify that the vector $\vec g'(\y) =\parentheses{w_{j} \mathds{1}\parentheses{j \in S \land \sum_{j' \in S} y_{j'} w_{j'} \leq b}}_{j\in V}$ is a supergradient of $\tilde\Psi (\,\cdot\,)$ at point $\y$, i.e., it holds~\eqref{eq:def_supergradient}
  \begin{align}
    {\tilde{\Psi}}(\y) - {\tilde{\Psi}}(\y') &\geq \vec g'(\y) \cdot (\y - \y'), \quad \text{for every $\y' \in \reals^n$.} \label{eq:supergrdient_ineq}
\end{align}
Consider $\sum_{j \in S} w_{j} y_{j} > b$, then it holds 
\begin{align}
    {\tilde{\Psi}}(\y) - {\tilde{\Psi}}(\y') = b - \min\set{b , \sum_{j \in S } w_k j'_j} \geq 0.
\end{align}
\begin{align}
    \vec g'(\y) \cdot (\y - \y') &= \sum_{j\in V}  g_j(\y) (y_j - y'_j) = \sum_{j\in S}  w_j (y_j - y'_j) \mathds{1} \parentheses{\sum_{j' \in S} w_{j'} y_{j'} \leq b} = 0\\
    &\leq {\tilde{\Psi}}(\y) - {\tilde{\Psi}}(\y').
\end{align}
Thus, when $\sum_{j \in S} w_j y_j > b$ the inequality in Eq.~\eqref{eq:supergrdient_ineq} holds. It remains to check if it holds when $\sum_{j \in S} w_j y_j \leq b$. We have 
\begin{align}
    \vec g'(\y) \cdot (\y - \y')& = \sum_{j\in V}  g_j(\y) (y_j - y'_j) = \sum_{j\in S}  w_j (y_j - y'_j) =  \sum_{j\in S}  w_jy_j - \sum_{j\in S} w_j {y'_j}\\
    &\leq \sum_{j\in S}  w_j y_j - \min\set{b,\sum_{j\in S} w_j {y'_j}}\leq \min\set{b, \sum_{j\in S}  w_j y_j} - \min\set{b,\sum_{j\in S} w_j {y'_j}} \\
    &= {\tilde{\Psi}}(\y) - {\tilde{\Psi}}(\y').
\end{align}
The vector $\vec g'(\y)$ is a supergradient of $\tilde \Psi (\,\cdot\,)$ at point $\y$. We conclude the proof that by noting that $\cfunc(\y) = \sum_{\ell \in C} c_\ell \tilde{\Psi}_{b_\ell, \vec w_\ell, S_\ell}(\y)$, and it is straightforward to check that $\partial \parentheses{\sum_{\ell \in C} c_\ell \tilde{\Psi}_{b_\ell, \vec w_\ell, S_\ell}(\y)}= \sum_{\ell \in C}c_\ell  \partial \parentheses{ \tilde{\Psi}_{b_\ell, \vec w_\ell, S_\ell}(\y) }$ from the definition of superdifferentials.
\end{proof}
\begin{table}[t]
    \centering
    \begin{tiny}
    \begin{tabular}{|c|c|c|c|c|c|c|c|}
    \toprule
    \hline
         \textbf{Mirror Map}~$\Phi$& \textbf{Domain} $\mathcal{D}$ &  \textbf{Dual-primal update step} & \textbf{Str. convexity}~{$\rho$} & \textbf{Bregman div. bound}~$D^2$ &  \textbf{Lip. const.}~ $L_\Phi$\\
         \hline
         $\sum^n_{i=1} (y_i+\gamma) \log(y_i + \gamma)$ & $(-\gamma, \infty)^n$ & $z_{t+1, i } = (y_{t,i} + \gamma) e^{\eta_t g_{t,i}}-\gamma, i \in [n]$ & $1/(\rank + \gamma n)$ & $ r (1+\gamma) \log\parentheses{\frac{(1+\gamma)n}{\rank + \gamma n }}$& $\log(1/\gamma)  - 1$ \\
         
         \hline
         $\sum^n_{i=1} y_i\log(y_i)$ & $\reals^n_{>0}$ & $z_{t+1, i } = y_{t,i} e^{\eta_t g_{t,i}}, i \in [n]$ & $1/\rank$ & $ r \log\parentheses{\frac{n}{\rank}}$& $\infty$\\
\hline
        $\frac{1}{2}\norm{\y}^2_2$ & $\reals^n$ & $\z_{t+1} = \y_t + \eta_t \g_t$ & $1$ & $\min\set{r, n-r}/2$ & $1/2$ \\
         \bottomrule
    \end{tabular}
        
    \end{tiny}
    \caption{Properties of the different mirror map choices of OMA in Alg.~\ref{alg:oma}. }
    \label{tab:bounds}
\end{table}

\subsection{Regret Bounds Derivation and Comparison with Related Work}
\label{appendix:additional_related_work}
In this section, we first discuss the derivation of the regret bounds attained by this work in Table~\ref{table:compare}. Then, we discuss algorithm-specific parameters of related work.

\paragraph{Regret Bounds.} We discuss the regret bound appearing in Table~\ref{table:compare}. We obtain the tightest regret bounds by employing the negative entropy setup of OMA in Sec.~\ref{appendix:setups_omd},  allowing us to optimize the regret constants. In particular, the regret of  OMA when configured with the negative-entropy mirror map grows logarithmically w.r.t. the problem dimension (i.e., $\BigO{r\sqrt{\log(n/r) T}}$ in Corollary~\ref{corollary:1}) instead of sublinearly as in the case of OGA (e.g., $\BigO{\sqrt{rnT}}$ in Corollary~\ref{corollary:2}). We note that the related work assumes that the reward functions are bounded by $1$. We make the same assumption to better compare our bounds in Table~\ref{tab:bounds}.  However, note that our bounds are further scaled by $d_\infty = \max_{i \in [n]} \set{ f(\set{i}) - f(\emptyset)} =\max_{i \in [n]} \set{ f(\set{i})} \in [1/n, 1]$ which can further improve our bounds under the static and dynamic full-information settings.

\paragraph{Algorithm-specific Parameters in Related Work.} \citet{niazadeh2021online} require a Blackwell approachability policy. We denote by  $O_{a}$ its time complexity. \citet{onlineassignement} require a color palette size  parameter $c_{\mathrm{p}} \in \naturals$. This parameter creates $c_{\mathrm{p}} $ expert algorithms (OCO algorithms operating over the simplex) for a single slot, e.g., over partition matroids the algorithm instantiates $r c_{\mathrm{p}}$ experts.  Harvey et al.~\cite{harvey2020improved} require a discretization parameter $\epsilon > 0$ to convert their continuous-time algorithm to a discrete algorithm, an oracle that gives the multi-linear extension with time-complexity $O_{\mathrm{m}}$, and swap rounding~\ref{appendix:swap}. Works that operate on continuous DR-submodular objectives~\cite{chen2018online,zhang2019online,zhang2022stochastic} require OCO policies with a time complexity $O_{\mathrm{oco}}$; moreover, these algorithms operate over the multilinear extension with time complexity  $O_{\mathrm{m}}$. The decisions can be rounded through a randomized-rounding algorithm to operate over matroids (e.g., swap rounding~\ref{appendix:swap}). We assume that these algorithms have access to a decomposition oracle; then the rounding step has an additional time complexity of $\BigO{r^2 n}$.  \citet{kakade2007playing} operate on Linearly Weighted Decomposable (LWD) functions, i.e., reward functions of the form $f_t(\x) = \vec\Phi(\x) \cdot \vec w_t$, where  $\vec \Phi(\x) \in \reals^s$ is a  vector-valued set function known to the decision maker, and $\vec w_t\in \reals^s$ are time-varying vector-valued weights selected by an adversary for some $s \in \naturals$. This algorithm requires access to an $\alpha$-approximation oracle with time-complexity~$O_\alpha$.

\section{Proof of Theorem~\ref{theorem:sandwich}}
\label{appendix:sandwich}
\begin{proof}
Consider the sequence of reward functions $\set{f_1, f_2, \dots, f_T} \in \F^T$ and the associated sequence of concave relaxations  $\set{\cfunc_1, \cfunc_2, \dots, \cfunc_T} \in \hat{\F}^T$. Then,
\begin{align}
    \max_{\x \in \X} \sum_{t\in\T} f_t(\x) &\leq \max_{\x \in \X} \sum_{t\in\T} \cfunc_t(\x) \leq \max_{\y \in \Y} \sum_{t\in\T} \cfunc_t(\y).\label{eq:proof:piece1}
\end{align}
The first inequality is implied by Eq.~\eqref{eq:alpha_approx_upper} in Assumption~\ref{assumption:sandwich}, and the second inequality holds because maximizing over a superset $\Y =\conv\X \supseteq \X$ can only increase the objective value attained. The total expected reward obtained by the RAOCO policy  is given by 
\begin{align}
    \E_\Xi \interval{\sum_{t\in\T} f_t(\x_t)} = \sum_{t\in\T}\E_\Xi  \interval{ f_t(\x_t)} =   \sum_{t\in\T}\E_\Xi  \interval{ f_t(\Xi(\y_t))}\overset{\eqref{eq:alpha_approx_lower}}{\geq} \alpha \sum_{t\in\T} \cfunc_t(\y_t).\label{eq:proof:piece2}
\end{align}
From Eqs.~\eqref{eq:proof:piece1} and \eqref{eq:proof:piece2}, it follows that:
\begin{align}
    \alpha \max_{\x \in \X}\sum_{t\in\T} f_t(\x) -  \E_\Xi \interval{\sum_{t\in\T} f_t(\x_t)} \leq  \alpha \max_{\y \in \Y}\sum_{t\in\T} \cfunc_t(\y) -  \alpha \sum_{t\in\T} \cfunc_t(\y_t).
\end{align}
We conclude the proof by noting that the above inequality holds for arbitrary sequences of reward functions $\set{f_1, f_2, \dots, f_T} \in \F^T$.
\end{proof}

\section{Proof of Lemma~\ref{proposition:sandwich2}}
\label{appendix:sandwich2}

We begin by proving a series of auxiliary lemmas. 
\begin{lemma}
Consider  $n \in \naturals$, $\y \in [0,1]^n$, $b \in \reals_{>0}$, and  $\vec {w} \in   [0,b]^n$.  The following holds
\begin{align}
  \min\left\{b , \sum_{i \in V} y_i w_i \right\} \geq b - b \prod_{i \in V} (1 - y_i  w_i /b).
  \label{eq:lower_bound_piece}
\end{align}
\label{lemma:util_bounds_lower}
\end{lemma}
\begin{proof}
We define $A_n \triangleq b - b \prod_{i \in V} (1 - y_i  w_i/b)$ and $B_n \triangleq \min\left\{b, \sum_{i \in V} y_i  w_i\right\}$. 

{We first show by induction that, if $A_n \leq B_n$, then this inequality holds also for $n+1$.}

\noindent\emph{Base Case ($n=1$).}
\begin{align}
    A_1 &= b - b + y_1 w_1 = y_1 w_1 = \min\{b, w_1 y_1\} = B_1.
\end{align}
\noindent\emph{Induction Step.}
\begin{align}
    A_{n+1} &= b - b \prod_{i \in [n+1]} (1 - y_i  w_i/b )
    \label{eq:an+1}
    \\
    &=b - b \prod_{i \in [n]} (1 - y_i  w_i/b ) (1 - y_{n+1} w_{n+1}/b) \\
    &= b - b \prod_{i \in [n]} (1 - y_i w_i/b) +  (b y_{n+1} w_{n+1}/b) \prod_{i \in [n]} (1 - y_i w_i/b)\\
    &= A_n +  y_{n+1} w_{n+1}\prod_{i \in [n]} (1 - y_i w_i/b) \\
    &\leq A_n +  y_{n+1} w_{n+1}.
  \end{align}
The last inequality holds since by construction $w_i\le b$ and thus $0\le y_i w_i /b \le 1$, and $0 \leq \prod_{i \in V} \left(1- {y_i w_i}/{b}\right)\leq 1$. For the same reason, $0\le\prod_{i \in [n+1]} (1 - y_i  w_i/b )\le 1$ and thus, by~\eqref{eq:an+1}, we have
$A_{n+1} \leq b$. Moreover, note that if $B_n = b$ then $B_{n+1} = b$. Therefore:
\begin{align}
    A_{n+1} &\leq \min\left\{b, A_n +  y_{n+1}  w_{n+1}\right\} \\
    &\leq \min\left\{b, B_n +  y_{n+1}  w_{n+1} \right\}\\
    &= \begin{cases}
    \min\left\{b, \sum^{n+1}_{i=1} y_i w_i\right\} = B_{n+1}, &\text{ if}\quad B_n \leq b,  \\
    \min\left\{b, b + y_{n+1}  w_{n+1} \right\} = b = B_{n+1}, &\text{ if}\quad B_n = b,
    \end{cases}
\end{align}
and the proof by induction is completed.
\end{proof}

\begin{lemma}\label{lemma:util_bounds_upper}
Consider $n \in \naturals$, $\y \in [0,1]^n$, $b \in \reals_{>0}$, and $\vec{w} \in  [0,b]^n$. The following holds
\begin{align}
    b - b\prod_{i \in [n]}  (1- y_i w_i / b)  \geq (1-1/n)^{n}  \min\left\{b,  \sum_{i \in [n]} y_i  w_i\right\},
     \label{eq:upper_bound_piece}
\end{align}
\end{lemma}
\begin{proof}
    Consider $\y' = (y_i w_i/b)_{i \in [n]} \in [0,1]^n$. Then, it follows from \cite[Lemma~3.1]{goemans1994new}.
    \begin{align}
        1 - 1\prod_{i \in [n]}  (1- y'_i)  \geq (1-1/n)^n  \min\left\{1,  \sum_{i \in [n]} y'_i\right\}.
    \end{align}
Replacing $\y'$ by its value and multiplying both sides by $b$ yields Eq.~\eqref{eq:upper_bound_piece}.
\end{proof}

\begin{lemma}
Consider $n \in \naturals$, $\y \in \interval{0,1}^n$, and $\vec w \in [0,1]^n$. The following holds 
\begin{align}
  \prod_{i \in [n]} (1-w_i y_i) =  \sum_{S \in 2^{[n]}} \prod_{i \in S} (1-y_i) \prod_{i \in S} w_i \prod_{i \in [n]\setminus S} (1-w_i).\label{eq:tech_equiv_expression}
\end{align}\label{lemma:tech_equiv_expression}
\end{lemma}
\begin{proof}

We show by induction that, if Eq.~\eqref{eq:tech_equiv_expression} holds for some $n \in \naturals$, then this equality holds also for $n+1$. Define $A_n(\y) \triangleq \prod_{i \in [n]} (1-w_i y_i)$ and $B_n(\y) \triangleq \sum_{S \in 2^{[n]}} \prod_{i \in S} (1-y_i) \prod_{i \in S} w_i \prod_{i \in [n]\setminus S} (1-w_i)$ for $n \in \naturals$ and $y \in \interval{0,1}^n$.

\emph{Base case $(n=1)$.} In this setting the powerset of $\set{1}$ is  $2^{\set{1}} =\set{\emptyset, \set{1}} $.
\begin{align}
    1- w_1 y_1 &=  \sum_{S \in \set{\emptyset, \set{1}}}  \prod_{i \in S} (1-y_i) \prod_{i \in S} w_i \prod_{i \in [n]\setminus S} (1-w_i)\\
    &= 1-w_1 + (1-y_1) w_1 = 1- w_1 y_1. 
\end{align}

\emph{Induction Step.}
\begin{align}
    B_{n+1}(\y) &= \sum_{\substack{S \in 2^{[n+1]} \\ n+1 \notin S}} \prod_{i \in S} (1-y_i) \prod_{i \in S} w_i \prod_{i \in [n+1]\setminus S} (1-w_i) \nonumber \\
    &+  \sum_{\substack{S \in 2^{[n+1]} \\ n+1 \in S}} \prod_{i \in S} (1-y_i) \prod_{i \in S} w_i \prod_{i \in [n]\setminus S} (1-w_i)  \label{eq:induction_tech_p1} \\
     &= \sum_{S \in 2^{[n]} } \prod_{i \in S} (1-y_i) \prod_{i \in S} w_i \prod_{i \in [n]\setminus S} (1-w_i) (1-w_{n+1})  \nonumber \\
     &+ \sum_{S \in 2^{[n]} } \prod_{i \in S} (1-y_i) \prod_{i \in S} w_i \prod_{i \in [n]\setminus S} (1-w_i)  w_{n+1} (1-y_{n+1}) \label{eq:induction_tech_p2}\\
     &= A_n(\y) (1- w_{n+1}) (1-w_{n+1} + w_{n+1} (1 - y_{n+1})) = A_{n+1} (\y) \label{eq:induction_tech_p3}.
\end{align}
The first equality Eq.~\eqref{eq:induction_tech_p1} is obtained from the definition of $B_n(\y)$ and splitting the summation over the set $2^{[n+1]}$ into a summation over two subsets $\set{S \subseteq 2^{[n+1]}: n+1 \notin S}$ and $\set{S \subseteq 2^{[n+1]}: n+1 \in S}$.  Equation~\eqref{eq:induction_tech_p2} is obtained by factoring the terms that depend on $n+1$. Equation~\eqref{eq:induction_tech_p2} follows from the definition of $A_n(\y)$.  This concludes the proof.

\end{proof}


\begin{lemma}\label{lemma:weighted_negative_correlation}
Consider random variables $x_i \in \set{0,1}, i \in V$ with expected values $\E[x_i] = y_i, i \in V$ that are {negatively correlated}. The random variables $w_i x_i \in \set{0, w_i}$ are also negatively correlated for non-negative weights $w_i\in [0,1] ,i\in V$. The following holds for every $S \subseteq V$ 
\begin{align}                                                
  \E\interval{  \prod_{i \in S} (1-w_i x_i)} &\leq \prod_{i \in S} (1-w_iy_i), & \E\interval{  \prod_{i \in S} (1+w_i -x_i)} &\leq \prod_{i \in S} (1+w_i - y_i).\label{eq:weighted_negative_correlation}
\end{align}
\end{lemma}

\begin{proof} We exploit the alternative expression of $ \prod_{i \in B} (1-w_i x_i)$ developed in Lemma~\ref{lemma:tech_equiv_expression} and the negative correlation property to prove this proposition. Consider l.h.s. term in Eq.~\eqref{eq:weighted_negative_correlation}
    \begin{align}
        \E\interval{  \prod_{i \in S} (1-w_i x_i)} &\overset{\eqref{eq:tech_equiv_expression}}{=} \E\interval{  \sum_{S' \in 2^{S}} \prod_{i \in S'} (1-x_i) \prod_{i \in S'} w_i \prod_{i \in S\setminus S'} (1-w_i)}\\
        &= \sum_{S' \in 2^{S}} \E\interval{ \prod_{i \in S'} (1-x_i)} \prod_{i \in S'} w_i \prod_{i \in S \setminus S'} (1-w_i)\\
        &\leq \sum_{S' \in 2^{S}}  \prod_{i \in S'} (1-y_i) \prod_{i \in S'} w_i \prod_{i \in S \setminus S'} (1-w_i)\\
        &=  \prod_{i \in S} (1-w_i y_i).\label{eq:exp_piece_induction}
    \end{align}  
\end{proof}

Putting everything together,
    consider $\y \in \interval{0,1}^n$ and $\x = \Xi (\y)$ stated in the Proposition. From Lemmas~\ref{lemma:util_bounds_lower}--\ref{lemma:weighted_negative_correlation} we have:
    \begin{align}
        \E_\Xi\interval{f(\x)} &\overset{\eqref{eq:budget-additive}}{=}  \E_\Xi\interval{\sum_{i \in C} c_{i} \min\set{b_{i}, \sum_{j \in S_{i}}w_{i,j} x_{j}}} \\&\overset{\eqref{eq:lower_bound_piece}}{\geq}  \E_\Xi\interval{\sum_{i \in C} c_{i} \parentheses{ b_i - b_i \prod_{j \in S_i} \parentheses{1-\frac{x_i w_{i,j}}{b_i}}}}\\
        & \overset{\eqref{eq:exp_piece_induction}}{\geq}\sum_{i \in C} c_{i} \parentheses{ b_i - b_i \prod_{j \in S_i} \parentheses{1-\frac{y_i w_{i,j}}{b_i}}} \\& \overset{\eqref{eq:upper_bound_piece}}{\geq} \parentheses{1 - \frac{1}{\Delta}}^{\Delta}\sum_{i \in C} c_{i} \min\set{b_{i}, \sum_{j \in S_{i}}w_{i,j} y_{j}}.
    \end{align}
Note that when $b_i = \infty$, then $\E_\Xi\interval{\min\set{b_{i}, \sum_{j \in S_{i}}w_{i,j} x_{j}}} =\min\set{b_{i}, \sum_{j \in S_{i}}w_{i,j} y_{j}}$ due to linearity. This concludes the proof of Lemma~\ref{proposition:sandwich2}. \qed

\section{Swap Rounding}\label{appendix:swap}

In swap rounding~\cite{chekuri2010dependent}, we are given a fractional solution $\y \in \Y=\conv{\X}$. We assume that $\y$ is decomposed as a convex combination of bases (i.e., maximal elements) of $\X$, i.e., \begin{align}\mathbf{y} &= \sum_{k=1}^K \gamma_k \z_k, & &\text{where} & \gamma_k&\in [0,1],  \sum_{k=1}^K\gamma_k=1,~\text{and}~\z_k\in \X,k\in[K],~\text{are bases of}~\X.\end{align} 
We note again that Carathéodory's theorem implies the existence of such decomposition of at most $n+1$ points in $\X$; moreover, there exists a decomposition algorithm~\cite{cunningham1984testing} for general matroids with a running time $\BigO{n^6}$. 
However, in all practical cases we consider (including uniform and partition matroids) the complexity is significantly lower. More specifically,
on partition matroids, swap rounding reduces to an algorithm with linear time complexity, namely, $\BigO{m n}$ for $m$ partitions~\cite{Srinivasan2001}.

In short, swap rounding randomly picks between pairs of basis vectors over several iterations, producing a random final integral solution.  Let $\beta_k \triangleq \sum_{k'=1}^k \gamma_k$,  $k\in [K]$. Then:
\begin{itemize} 
\item Set $\x_1=\z_1$.
\item Then, at iteration,  a random selection is made between vector $\x_{k}\in \X$ and basis $\z_{k+1}\in\X$ to produce a new integral  $\x_{k+1}\in \X$ as follows: 
\begin{align}
 \x_{k+1} = \begin{cases} \x_{k}, &\text{with probability}  ~\frac{\beta_k}{\beta_{k+1}}, \\
 \z_{k+1},& \text{with probability}\frac{\gamma_{k+1}}{\beta_{k+1}}.
 \end{cases}
 \end{align}
 The authors refer to this randomized selection as a ``merge'' between $\x_{k}$ and $\z_{k+1}$.
\end{itemize}
The final vector $\x_K$ is the output of the algorithm and satisfies, by construction:
\begin{align}
    \E[\x_K] = \sum_{k=1}^K \gamma_k \z_k = \y.
\end{align}

\section{Extensions}
\label{appendix:extensions}

\subsection{Dynamic Setting} The derivation of Theorem~\ref{thm:dynamic} follows the same steps as the proof of Theorem~\ref{theorem:sandwich}.  Recall that 
$\Lambda_{\X} (T,P_T) \subset \X^T$ is the set of sequences of decision in a set $\X$ with path length less than $P_T \in \reals_{\geq 0}$ over a time horizon $T$ defined in Sec.\ref{sec:extensions}. Consider the sequence of reward functions $\set{f_1, f_2, \dots, f_T} \in \F^T$ and the associated sequence of concave relaxations  $\set{\cfunc_1, \cfunc_2, \dots, \cfunc_T} \in \hat{\F}^T$. Then, Asm.~\ref{assumption:sandwich} implies 
\begin{align}
    \max_{\parentheses{\x_t}^T_{t=1} \in \Lambda_{\X} (T,P_T)} \sum_{t\in\T} f_t(\x_t) &\leq \max_{\parentheses{\x_t}^T_{t=1} \in \Lambda_{\X} (T,P_T)} \sum_{t\in\T} \cfunc_t(\x_t) \leq \max_{\parentheses{\y_t}^T_{t=1} \in \Lambda_{\Y} (T,P_T)} \sum_{t\in\T} \cfunc_t(\y_t).\label{eq:proof:dynamic_piece1}
\end{align}
The first inequality is implied by Eq.~\eqref{eq:alpha_approx_upper} in Asm.~\ref{assumption:sandwich}, and the second inequality holds because maximizing over a superset $\Lambda_{\Y} (T,P_T)\supseteq \Lambda_{\X} (T,P_T)$ can only increase the attained objective value. The total expected reward obtained by the RAOCO policy  is given by 
\begin{align}
    \E_\Xi \interval{\sum_{t\in\T} f_t(\x_t)}\overset{\eqref{eq:proof:piece2}}{\geq} \alpha \sum_{t\in\T} \cfunc_t(\y_t).\label{eq:proof:dynamic_piece2}
\end{align}
Eqs.~\eqref{eq:proof:dynamic_piece1} and~\eqref{eq:proof:dynamic_piece2} imply that
\begin{align}
    \aregret{}_{T,P_T} \parentheses{\vec\P_\X} \leq \alpha \cdot \regret{}_{T,P_T} \parentheses{{\vec{\P}}_{\Y}}.\label{e:general_sandwitch}
\end{align}
The OCO policy ${\vec{\P}}_{\Y}$ is OMA in Alg.~\ref{alg:oma} configured with an appropriately selected mirror map satisfying  Asms.~\ref{asm:mirror_map}--\ref{asm:bounded_gradient_mirror_map} (e.g. OGA by selecting $\Phi(\x) = \frac{1}{2} \norm{\x}_2^2$). Then under Asm.~\ref{asm:oco}, we conclude from Theorem~\ref{thm:general_regret_dynamic+optimistic} the following 
\begin{align}
    \aregret{}_{T,P_T} \parentheses{\vec\P_\X} &\leq \alpha \cdot \regret{}_{T,P_T} \parentheses{{\vec{\P}}_{\Y}} \leq \sqrt{2/\rho\parentheses{D^2 + 2 L_\Phi P_T} L^2 T}= \BigO{\sqrt {P_T T}},
\end{align}
where $L, \rho^{-1}, D$ are all bounded from above by a finite quantity.

\subsection{Optimistic Setting} 
Under Asm.~\ref{assumption:sandwich}, the following inequality holds from Eq.~\eqref{e:general_sandwitch} in the dynamic setting.
\begin{align}
     \aregret{}_{T,P_T} \parentheses{\vec\P_\X} \leq \alpha \cdot \regret{}_{T,P_T} \parentheses{{\vec{\P}}_{\Y}}.
\end{align}
Combining the above inequality and Theorem~\ref{thm:general_regret_dynamic+optimistic} yields the following upper bound: 
\begin{align}
         \aregret{}_{T,P_T} \parentheses{\vec\P_\X} \leq \alpha \cdot \sqrt{{2/\rho\parentheses{D^2 + 2 L_\Phi P_T} \sum^T_{t=1} \norm{\vec g_t - \pg_t}^2_\star}} = \BigO{\sqrt{ P_T \sum^T_{t=1} \norm{\vec g_t - \pg_t}^2_\star}}
\end{align}
where $L, \rho^{-1}, D, P_T$ are all bounded from above by a finite quantity.

\subsection{Bandit setting}\label{appendix:bandit:setting}

\begin{table*}[!t]
\resizebox{\textwidth}{!}{
\begin{tiny}
\begin{tabular}{||c||c||c|c|c||c||c||c||}
\toprule
 & &\multicolumn{5}{c||}{\textbf{$(1-1/e)$-regret (Bandit)}} & \\
\cline{3-8}
\textbf{Paper}  & \textbf{Prob. Class} &\multicolumn{3}{c||}{\textbf{Static}} &\multicolumn{1}{c||}{\textbf{Dynamic}} & \multicolumn{1}{c||}{\textbf{Optimistic}} &  \textbf{Time} \\
\cline{3-8}
&&\textbf{Uni.}&\textbf{Part.}&\textbf{Gen.}&\textbf{Part.}&\textbf{Part.} & \\
\cline{1-8}
\cite{kakade2007playing}  & LWD&\multicolumn{3}{c||}{$n^{\frac{5}{3}} T^{\frac{2}{3}}$}& \xmark & \xmark &  $ T O_{\mathrm{\alpha}}$\\ 
\hline
\hline
\cite{niazadeh2021online} & GS& $rn^{\frac{2}{3}} \log^{\frac{1}{3}} (n) T^{\frac{2}{3}} $&\xmark&\xmark&\xmark&\xmark& $T^4 O_b$\\
\hline
\cite{onlineassignement}  & GS &\multicolumn{2}{c|}{$\log^{\frac{1}{3}}\!(n)(n  c_{\mathrm p} r)^{\frac{2}{3}}\!T^{\frac{2}{3}}$}& \xmark & \xmark & \xmark & ${n^2 c_{\mathrm p} }$\\ 
\hline
\cite{wan2023bandit}  & GS &\multicolumn{2}{c|}{$r^{\frac{4}{3}} n^{\frac{4}{3}}T^{\frac{2}{3}}$}& \xmark & \xmark& \xmark & $ n r$ \\
\hline
\hline
\paintbg \parbox{.25cm}{\centering This work}& \parbox{1cm}{\centering\paintbg  GS}&  \multicolumn{2}{c|}{\parbox{3cm}{\paintbg \centering $ \alpha r^{4/3} n^{4/3} T^{2/3}$}} & \paintbg \xmark & \paintbg $ \sqrt{P_T} T^{4/5}$& \parbox{3cm}{\centering\paintbg $W \sqrt{P_T \sum^{T/W}_{l=1} \norm{\tilde{\vec g}_{l} - \vec g^\pi_l}^2_2}$ \\ $+ T (\delta + 1/W)$ }&\paintbg ${n r}$\\
\bottomrule
\end{tabular}
\end{tiny}
}
\caption{Order ($\BigO{\,\cdot\,}$) comparison of regrets and time complexities attained by different online submodular optimization algorithms for general submodular  (GS) under the bandit setting. We also include Kakade et al.~\cite{kakade2007playing}, who operate on Linearly Weighted Decomposable (LWD) functions. Competing algorithms that outperform us either only operate on limited constraint sets \cite{niazadeh2021online,onlineassignement} or the LWD class \cite{kakade2007playing}; with the exception of the work of~\cite{wan2023bandit} which attains a tighter regret in the static setting over GS class, however when we employ a specialized static OCO algorithm we attain the same bound (see the discussion in Sec.~\ref{appendix:extensions}).   Our work also generalizes to the dynamic and optimistic settings in the bandit setting.}
\vspace*{-15pt}
\label{table:compare2}
\end{table*}

\begin{algorithm}[t]
\caption{Limited-Information Rounding-augmented OCO (LIRAOCO) policy}
\begin{algorithmic}[1]
\Require Exploration parameter $\delta \in (0,1/2n]$, OCO policy $\vec{\P}_{\Y_\delta}$, $\vec H_l, l \in [T/W]$
\For {$\mathcal{W}_l = \mathcal{W}_1, \mathcal{W}_2, \dots, \mathcal{W}_{T/W}$} \Comment{$\mathcal W_l$ is a sequence of $|\mathcal W_l| = W$ contiguous timeslots s.t. $[T] = \bigcup_{l \in [T/W]}\mathcal{W}_l$}
\State Draw $t_l$ u.a.r. from $\mathcal {W}_l$ 
\For{$t \in \mathcal {W}_l$}
\If {$t  = t_l$}
\State  $\y^\delta_l \gets \P_{\Y_\delta, l}\parentheses{ (\y^\delta_s)^{l-1}_{s=1}, (r_s)^{l-1}_{s=1}}$
\State Sample $z_l$ from $Z$ where $\mathbb P (Z < z) = \int^z_0\frac{\exp(s-1)}{1 - \exp(-1)} \mathds{1} \parentheses{s \in [0,1]} ds$
\State Sample $\vec v_l$ u.a.r. from $\mathbb S_{n-1}$
\If {$z_l \geq \frac{1}{2} $}
\State Sample $\vec u_l$ from $\set{\vec 0, \vec e_1, \vec e_2, \dots, \vec e_n}$ w.p. $\mathbb P (\vec u_l = \vec 0) = 1/2$ and  $\mathbb P (\vec u_l = \vec e_i) = 1/(2n)$
\State $\y_l \gets z_l \cdot  (\y^\delta_{l} + \parentheses{(\vec H_l \vec v_l)^{\intercal}\vec u_l} \vec u_l)  $ 
\State Play $\x_t \gets \hat{\Xi}(\y_l)$
\State Receive reward $f_t(\x_t)$ 
\State $ \tilde l_l \gets 2 \alpha \frac{n}{z_l} f_t(\x_t) (1 -2 \mathds{1} \parentheses{\vec u_l = \vec 0})$ 
\Else
\State Sample $\vec u_l$ u.a.r. from $\set{\vec e_1, \vec e_2, \dots, \vec e_n}$ 
\State Assign u.a.r. $\y_l \gets z_l \y^\delta_l + \frac{1}{2} \vec u_l$ or $\y_l \gets z_l \y^\delta_l$ 
\State Play $\x_t \gets \hat{\Xi}(\y_l)$
\State Receive reward $f_t(\x_t)$ 
\State $ \tilde l_l \gets 4 \alpha n \parentheses{(\vec H_l \vec v_l)^{\intercal}\vec u_l} f_t(\x_t) (1 -2 \mathds{1} \parentheses{\vec y_l = z_l \vec \y^\delta_l})$ 
\EndIf
\State Construct an estimate of supergradient $\tilde{\vec{g}}_l \gets  n \tilde l_l \vec H^{-1}_{l} \vec v_l$
\State Construct a linear function $r_l(\y) =  \y \cdot \tilde{\vec {g}}_l$
\Else 
\State Play $\x_t \gets \hat{\Xi}(\y^\delta_{l})$
\State Receive reward $f_t(\x_{t})$ 
\EndIf
\EndFor
\EndFor
\end{algorithmic}
\label{alg:lisaoco}
\end{algorithm}

In this section, we provide a full description of the Limited-Information Rounding-augmented OCO (LIRAOCO) policy in Alg.~\ref{alg:lisaoco}. The policy decomposes the timeslots $[T]$  to $T/W$ (assuming w.l.o.g. $T/W \in \naturals$) sequences of equally-sized contiguous timeslots denoted by $\mathcal{W}_l \triangleq \set{(l-1) W  +1, (l-1) W + 2, \dots, l W}$ for $l\in[T/W]$. The algorithm takes as input exploration matrices  $\vec H_l$ and shrinkage parameter $\delta >0$. In what follows we consider $\vec H_l = \delta \vec I$. Moreover, the algorithm utilizes an OCO algorithm $\P_{\Y_\delta}$ operating over a  properly designed convex subset $\Y_\delta$ of a set  $\Y$. The decisions are sampled through a rounding scheme $\hat \Xi: \Y \to \X$ at a given timeslot. The policy computes gradient estimates to explore the best actions by an OCO algorithm. The regret bound of the OCO policy is exploited by freezing the decision obtained by the algorithm for a time window of size $W \in \naturals$. Overall, the algorithm attains the bound specified in Theorem~\ref{theorem:bandit}. In what follows, we describe the decision set $\X$, the construction of the subset $\Y_\delta \subseteq \Y$, and the rounding scheme $\tilde \Xi$. Moreover, we introduce additional definitions that include the continuous extension of the reward function induced by the randomized rounding scheme, and its auxiliary function. We finally show that the gradient estimates in Alg.~\ref{alg:lisaoco} (line~21) are unbiased estimates of the gradients of the auxiliary function; this reduction relies on the reduction of \citet{wan2023bandit}. We restate Theorem~\ref{theorem:bandit}, and discuss how the regret bounds under dynamic and optimistic settings are obtained when OOMA in Alg.~\ref{alg:ooma} is selected as the OCO policy.

Under this setting, we restrict the decision set $\X$ to partition matroids.  Given a ground set $V = [n]$, partitions $V_i \subseteq V$, cardinalities $r_i \in \naturals$ for $i \in [m]$, the set $\X$ is given by  
\begin{align}
\textstyle \X \triangleq \set{\x \in \set{0,1}^{n}: \sum_{j \in V_i} x_j \leq r_i, i \in [m]}.\label{eq:parition_matroid}
\end{align} 
Let $n' \triangleq \sum_{i \in [m]}|V_i| \times r_j $, we define the alternative representation of the partition matroid $\X$ given as 
\begin{align}\label{eq:equiv_parition_matroid}
 \textstyle \X' \triangleq \set{\x \in \set{0,1}^{n'}: \sum_{k\in V_i} x_{i,j, k} \leq 1, j \in [r_i], i \in [m]},
\end{align} 
where $x_{i,j,k}$  indicates whether item $k \in V$ is selected from partition $V_i$ at slot $j$. Consider the mapping $T: \X' \to \X$ defined by
\begin{align}
\textstyle T( \x') &\triangleq \textstyle \parentheses{\min\set{1, \sum_{i \in [m]: k \in V_i} \sum_{j \in [r_i]} \x'_{i,j,k}}}_{k \in V}, \qquad \text{for $\vec x \in \X'$}. \label{eq:Tmapping}
\end{align}
It is easy to see that maximizing a function $f: \X \to \reals_{\geq 0}$ over the sets $\X$ or $\set{ T(\x'): \x' \in \X'}$ attains the same value, i.e., 
\begin{align}
\textstyle \max_{\x' \in \X'} f(T(\x')) =  \max_{\x \in \X} f(\x).
\end{align}
Thus, we could simply optimize over the set $\X'$ and map to points in $\X$ through the mapping $T$.

We define the fractional decision set for the OCO policy $\P_{\Y_\delta}$ employed by LIRAOCO policy in Alg.~\ref{alg:lisaoco} to be the convex set $\Y \triangleq \conv{\X'}$.  We define an appropriate subset of $\Y$ denoted by $\Y_\delta$ s.t. for every point $\y \in \Y_\delta$ we can form a ball of radius $\delta \in (0,1)$ contained in $\Y$. This set permits the decision maker to explore (perturb) without exiting the original decision set $\Y$. We provide an illustration in Fig.~\ref{fig:illustration}. Formally, 
\begin{proposition} \label{proposition:shrunken_set}
    Given the convex set $\Y = \conv{\X'}$ in Eq.~\eqref{eq:equiv_parition_matroid} and $\delta \in [0, 1/\parentheses{2 \max\set{|V_i|:i\in [m]}}]$, every ball formed around points in  
    \begin{align}
        \Y_\delta \triangleq  \set{\parentheses{(1- 2 \delta |V_i| ) \y_{i,j}+ \vec 1_{|V_i|} \delta}_{ i \in [m], j \in [r_i]}: \y \in \Y}
    \end{align}
    are contained in $\Y$, i.e., 
    \begin{align}
        \y + \delta \vec v \in \Y,\qquad  \text{for $\y \in \Y_\delta$, $\vec v \in \mathbb B_{n}$.}
    \end{align}
\end{proposition}
\begin{proof}
Note that the set $\Y$ is a product of simplexes, i.e., $ \Y = \bigtimes_{i \in [m] }  \parentheses{\Delta_{V_{i}}}^{r_i}$, where $\Delta_{V_{i}}$ is a simplex with support $V_i$. We  consider the subset $\Delta_{\delta, V_i} = \set{ \delta \vec{1}_{|V_i|} + (1 - 2|V_i|\delta) \y: \y \in \Delta_{V_i}}$ for $\delta \in (0, 1/(2 |V_i|))$. Let $\vec y_{i,j} \in \Delta_{\delta,|V_i|}$ for $j \in [r_i]$ and $i \in [m]$. It holds that $\vec y_{i,j} + \delta \vec v_{i,j} = \delta \parentheses{\vec{1}_{|V_i|} +\vec v_{i,j}} + (1- 2 |V_i| \delta ) \y'_{i,j}$ for some point $\y'_{i,j}$ in $\Delta_{V_i}$. The point $ \frac{1}{2|V_i|} (\vec{1}_{|V_i|} + \vec v_{i,j})$ satisfies  $ \vec 0 \leq   \frac{1}{2|V_i|} (\vec v_{i,j} + \vec{1}_{|V_i|}) \leq \frac{1}{|V_i|} \vec{1}_{|V_i|}$ (element-wise). Thus, it holds $ \frac{1}{2|V_i|} (\vec v_{i,j} + \vec{1}_{|V_i|}) \in \conv{\set{\vec 0, \vec e_1, \vec e_2, \dots, \vec e_{|V_i|}}} \subseteq \Delta_{V_i}$.
Thus, we have 
\begin{align}
       \y_{i,j} + \delta \vec v_{i,j}  =(1- 2|V_i|\delta)\underbrace{\y'_{i,j}}_\text{$\in \Delta_{V_i}$}  + 2|V_i| \delta  \underbrace{ (\vec{1}_{|V_i|} +\vec v_{i,j}) /2|V_i|}_\text{$\in \Delta_{V_i}$}.
\end{align}
It follows that  $\y_{i,j} + \delta \vec v_{i,j}  \in  \Delta_{V_i}$ because   $\Delta_{V_i}$ is a convex set, and $\y_{i,j}$ is a convex combinations of two points in $\Delta_{V_i}$. Since $\Y$ is a product of simplexes, it holds that $$\Y_\delta =\set{\parentheses{(1- 2 \delta |V_i| ) \y_{k,l}+ \vec 1_{|V_i|} \delta}_{i \in [m],j \in [r_i]}: \y \in \Y}.$$
\end{proof}
\begin{figure}
    \centering
    \includegraphics[width=.2\linewidth]{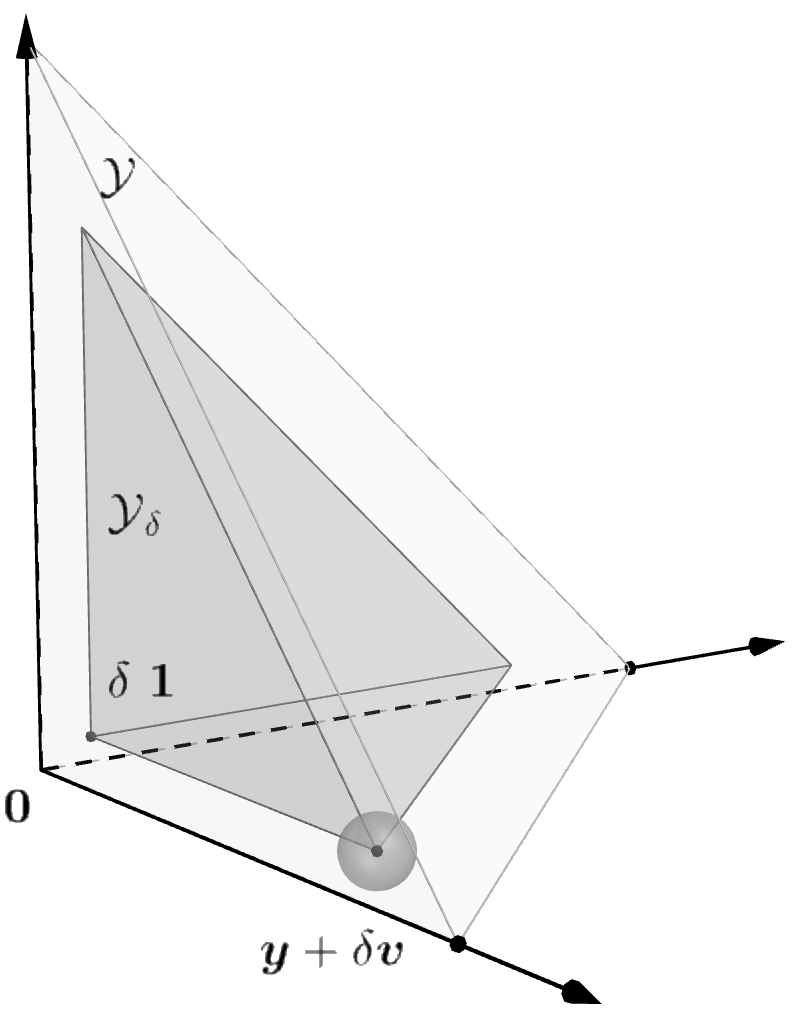}
    \caption{Illustration of Proposition~\ref{proposition:shrunken_set}: under the shrunken set $\Y_\delta$, any ball formed around points $\y \in \Y_\delta$ of radius $\delta$ is contained in $\Y$.\label{fig:illustration}}
\end{figure}

We provide a definition of the rounding scheme $\hat \Xi:\Y \to \X$ employed by LIRAOCO in Alg.~\ref{alg:lisaoco}.
\begin{definition} \cite[Extention Mapping]{wan2023bandit}
\label{def:ext_rounding_scheme}
      The randomized rounding scheme $\hat \Xi : \Y \to \X$ defined for sets  $\Y = \conv{\X'}$, and $\Y$ in Eqs.~\eqref{eq:equiv_parition_matroid} and \eqref{eq:parition_matroid}.  Given a point $\y \in \Y$, an intermediate randomized rounding scheme $\xi: \Y \to \X'$ outputs a point $\vec \omega \in \X'$, by sampling a single item $k \in V_i$ from the distribution $\vec y_{i,j}/\norm{\y_{i,j}}_1$ with support $V_i$ w.p. $\norm{\y_{i,j}}_1$ and assigns $\vec e_k$ to $\vec \omega_{i,j}$, otherwise it assigns  $\vec 0$ w.p. $1 - \norm{\y_{i,j}}_1$ to $\vec \omega_{i,j}$, for every  $i \in [m]$ and $j \in [r_i]$. The rounding scheme $\hat \Xi$ is then defined as the composition mapping $T \circ \xi:\Y \to \X $ where $T$ is provided in Eq.~\eqref{eq:Tmapping}.
\end{definition}

\begin{definition}
    A function $\efunc:\Y\to \reals_{\geq 0}$ is a DR-submodular function if for any $\y, \y'\in \Y$ 
    \begin{align}
       \efunc(\y + t \vec e_i) - \efunc(\y) \geq \efunc(\y'+ t \vec e_i) - \efunc(\y'), \label{eq:dr-submodularity_1}
    \end{align}
    where $\y \leq \y'$, $\y + t \vec e_i, \y' + t \vec e_i \in \Y$. 
    When $f$ is differentiable we have $\nabla \efunc(\y) \geq \nabla \efunc(\y')$ for $\y \leq \y'$. When $f$ is twice differentiable, the DR-submodularity is equivalent to
    \begin{align}
        \frac{\partial^2 \efunc(\y)}{\partial y_i \partial y_j} &\leq 0, &\text{ for all $i,j \in [n]$ and $\y \in \Y$}. \label{eq:dr-submodularity_2}
    \end{align}
    Moreover, $f$ is monotone if $\efunc(\y) \leq \efunc(\y')$ when $\y \leq \y'$.
\end{definition}
\begin{definition}\cite[Auxillary Function]{zhang2022stochastic}
    For a monotone DR-submodular function $\efunc : \Y \to \reals_{\geq 0}$ satisfying $\efunc(\vec 0) = 0$, its auxiliary function $\eFunc: \Y \to \reals_{\geq 0 }$ is defined as follows
\begin{align}
    \eFunc (\vec y) \triangleq \int^1_{0} \frac{\exp (z - 1)}{z} \efunc(z \cdot \vec y) dz.\label{def:aux} , \qquad \text{for $\y \in \Y.$}
\end{align}
\end{definition}

\begin{definition} \label{def:smoothed_func}
    For continuous function $l : \reals^n \to \reals$ and invertible matrix $\vec H \in \reals^{d\times d}$, an $\vec H$-smoothed version of $l$ denoted by  $l^{\vec H}$  is defined as
    \begin{align}
        l^{\vec H}( \x) \triangleq  \E_{\vec v \sim \mathbb B_d} \interval{l (\vec x + \vec H \vec v)},
    \end{align}
     where $\vec v \sim \mathbb B_n$ indicates that $\vec v$ is sampled from the $n$-dimensional unit ball $\mathbb B_d$ uniformly at random. 
\end{definition}
Remark that when $l: \reals^n \to \reals$ is a linear function, then $l = l^{\vec H}$.

\begin{lemma}\cite[Lemma~5.1]{wan2023bandit}\label{lemma:bandit_extention}
    Given a submodular monotone function $f:\X \to \reals_{\geq 0 }$ s.t. $f(\x) < L$ for $\x \in \X$ and $f(\vec 0) = 0$.   The randomized rounding scheme $\hat \Xi : \Y \to \X$ in Def.~\ref{def:ext_rounding_scheme} defines a multi-linear, monotone-increasing, DR-submodular, and $L$-lipschitz (w.r.t. $\norm{\,\cdot\,}_\infty$) function $\efunc: \Y \to \reals_{\geq 0}$ given by  \begin{align}
         \efunc (\y) &\triangleq \E_{\x \sim \hat\Xi (\y)} \interval{ f(\x)},  \qquad\text{for $\y \in \Y$},
    \end{align}
    
\end{lemma}
\begin{proof} The proof follows the same lines as  the proof of ~\cite[Lemma 5.3]{wan2023bandit}.

\noindent\emph{Multilinearity.} Given a point $\y \in \Y$, we have 
    \begin{align}
         \efunc (\y) &= \E_{\x \sim \hat\Xi (\y)} \interval{ f(\x)}= \E_{\vec \omega \sim \xi (\y)} \interval{ f(T(\omega))} = \sum_{\x' \in \X'} \mathbb P (\vec \omega = \x') f(T(\x')). 
    \end{align}
    Note that $\mathbb P (\vec \omega = \x') = \prod_{i \in [m]} \prod_{j \in [r_i]} \mathbb{P} (\vec \omega_{i,j} = \x'_{i,j})$; moreover,  $\mathbb{P} (\vec \omega_{i,j} = \x'_{i,j}) = y_{i,j,k}$ when $\x'_{i,j} = \vec e_k$, otherwise $\mathbb{P} (\vec \omega_{i,j} = \x'_{i,j}) = 1 - \sum_{k \in V_i} y_{i,j,k}$ when $\x'_{i,j} = \vec 0$, for $i \in [m]$ and $j \in [r_i]$. The factor $f(T(\x'))$ is a constant independent from $\y$, so we conclude $\efunc$ is a multilinear function.

    \noindent\emph{Monotonicity and lipschitzness.} Let $\vec e_{i,j,k} = (\mathds{1} \parentheses{(i',j',k') = (i,j,k)})_{i'\in [m],j'\in[r_i],k'\in V_i}$ and $\bar {\vec {e}}_{i,j,k} = \vec 1 - \vec e_{i,j,k}$. Since the function $\efunc$ is multilinear, it holds that
    \begin{align}
      \frac{  \partial \efunc(\y)}{  \partial y_{i,j,k}} = \frac{ \efunc(\y \oplus \lambda \vec e_{i,j,k}) -  \efunc(\y \ominus \bar{\vec{e}}_{i,j,k})}{\lambda}.
    \end{align}
    where $\lambda = 1 - \sum_{k' \in V_i: k' \neq k} y_{i,j,k'} > 0$, and the binary operators $\oplus$ and $\ominus$ are component-wise maximum and minimum operations, respectively, i.e., $\y\oplus\y' = \parentheses{\max\set{y_{i,j,k}, y'_{i,j,k}}}_{i \in [m], j \in [r_i], k \in V_i}$ and $\y\ominus\y' = \parentheses{\min\set{y_{i,j,k}, y'_{i,j,k}}}_{i \in [m], j \in [r_i], k \in V_i}$. Consider the two mapping $L^{\oplus}_{i,j,k}: \X' \to \X' $ and  $L^{\ominus}_{i,j,k}: \X' \to \X' $ defined as follows:
    \begin{align}
       \parentheses{ L^{\oplus}_{i,j,k} (\x)}_{i',j'}  &= \begin{cases}
           \x_{i',j'} + \vec e_k \mathds{1} \parentheses{\x_{i',j'}=\vec 0} &\text{if $(i',j') = (i,j)$},\\
           \x_{i',j'} & \text{otherwise},
       \end{cases}, \\
       \parentheses{ L^{\ominus}_{i,j,k} (\x)}_{i',j'}  &= \begin{cases}
           \x_{i',j'} \mathds{1} \parentheses{\x_{i',j'}\neq \vec e_k} &\text{if $(i',j') = (i,j)$},\\
           \x_{i',j'} & \text{otherwise},
       \end{cases} 
    \end{align}
    for $i' \in [m], j' \in [r_i]$ and $\x \in \X'$. One could check that 
    \begin{align}\label{eq:mapping_expectation_relation}
        \efunc(\y \oplus \lambda \vec e_{i,j,k}) = \E_{\vec \omega \sim \xi (\y) } \interval{f (T( L^{\oplus}_{i,j,k} (\vec \omega)))},\quad \efunc(\y \ominus \lambda \vec e_{i,j,k}) = \E_{\vec \omega \sim \xi (\y) } \interval{f (T( L^{\ominus}_{i,j,k} (\vec \omega)))}.
    \end{align}
Thus, it holds 
\begin{align}
    \frac{  \partial \efunc(\y)}{  \partial y_{i,j,k}} = \frac{ \efunc(\y \oplus \lambda \vec e_{i,j,k}) -  \efunc(\y \ominus \bar{\vec{e}}_{i,j,k})}{\lambda} = \frac{\E_{\vec \omega \sim \xi (\y) } \interval{f (T( L^{\oplus}_{i,j,k} (\vec\omega))) - f (T( L^{\ominus}_{i,j,k} (\vec\omega)))}}{\lambda} \geq 0.
\end{align}
Because $\lambda \geq 0$, $L^{\ominus}_{i,j,k} (\vec\omega) \leq L^{\oplus}_{i,j,k} (\vec\omega)$ (element-wise) by construction, and $f$ is monotone increasing. Note that  $L^{\oplus}_{i,j,k} (\vec\omega) \neq L^{\ominus}_{i,j,k} (\vec\omega)$ only when $\vec \omega_{i,j} = \vec e_k$ or $\vec \omega_{i,j} = \vec 0$. So, we have 
\begin{align}
    \frac{  \partial \efunc(\y)}{  \partial y_{i,j,k}} \leq L \frac{\mathbb P (\vec \omega_{i,j} \in \set{\vec e_k, \vec 0})}{ 1 - \sum_{k' \in V_i: k' \neq k} y_{i,j,k'} } =L \frac{1 - \sum_{k' \in V_i: k' \neq k} y_{i,j,k'} }{ 1 - \sum_{k' \in V_i: k' \neq k} y_{i,j,k'} } =  L.
\end{align}
The function $\efunc$ is $M$-lipschitz w.r.t. $\norm{\,\cdot\,}_\infty$.

\noindent \emph{DR-submodularity.}  The partial derivative of a multilinear function is also multilinear. So, we have the following
\begin{align}
    &\frac{  \partial^2 \efunc(\y)}{  \partial y_{i,j,k} \partial y_{i',j',k'}} = \frac{  \partial}{  \partial y_{i',j',k'}} \parentheses{  \frac{ \efunc(\y \oplus \lambda \vec e_{i,j,k}) -  \efunc(\y \ominus \bar{\vec{e}}_{i,j,k})}{\lambda} }\\
    &=  \frac{ \efunc(\y \oplus \lambda \vec e_{i,j,k} \oplus \lambda' \vec e_{i',j',k'}) -  \efunc(\y \ominus \bar{\vec{e}}_{i',j',k'} \oplus \lambda \vec{e}_{i,j,k}) - f(\y \oplus \lambda' \vec {e}_{i',j',k'} \ominus \vec e_{i,j,k}) + f(\y \ominus \vec e_{i,j,k} \ominus \vec e_{i',j',k'})}{\lambda \lambda'}.
\end{align}
As established in Eq.~\eqref{eq:mapping_expectation_relation}, it holds
\begin{align}
    & \efunc(\y \oplus \lambda \vec e_{i,j,k} \oplus \lambda' \vec e_{i',j',k'}) -  \efunc(\y \ominus \bar{\vec{e}}_{i',j',k'} \oplus \lambda \vec{e}_{i,j,k}) - f(\y \oplus \lambda' \vec {e}_{i',j',k'} \ominus \vec e_{i,j,k}) + f(\y \ominus \vec e_{i,j,k} \ominus \vec e_{i',j',k'})=\\
    &\E_{\vec\omega \sim \xi(\y)} \bigg[ f (T (L^{\oplus}_{i,j,k} (L^{\oplus}_{i',j',k'}(\vec{\omega})))) - f (T (L^{\oplus}_{i,j,k} (L^{\ominus}_{i',j',k'}(\vec{\omega})))) - f (T (L^{\ominus}_{i,j,k} (L^{\oplus}_{i',j',k'}(\vec{\omega})))) \\
    &+f (T (L^{\ominus}_{i,j,k} (L^{\ominus}_{i',j',k'}(\vec{\omega}))))\bigg] .
\end{align}
Considering that $f$ is submodular and the second-order differences definition of submodularity~\cite{bach2013learning}, we obtain $f (T (L^{\oplus}_{i,j,k} (L^{\oplus}_{i',j',k'}(\vec\omega)))) - f (T (L^{\oplus}_{i,j,k} (L^{\ominus}_{i',j',k'}(\vec\omega)))) - f (T (L^{\ominus}_{i,j,k} (L^{\ominus}_{i',j',k'}(\vec\omega)))) \geq 0 $ for any $\vec \omega \in \X'$. It follows that the second-partial derivative is non-positive
\begin{align}
    \frac{  \partial^2 \efunc(\y)}{  \partial y_{i,j,k} \partial y_{i',j',k'}} \leq  0.
\end{align}
Thus,  $\efunc$ is DR-submodular. This concludes the proof.
\end{proof}

\begin{lemma}\cite[Corollary~6.8]{hazan2016introduction}
Let $\vec H \in \reals^{n\times n}$ be an invertible matrix, $l: \reals^n \to \reals$ be a continuous function. Then
\begin{align}
    \nabla l^{\vec H} (\x) = n \E_{\vec v \sim \mathbb S_{n-1}} \interval{l(\x + \vec H \vec v) \vec{H}^{-1} \vec v},
\end{align}
where $\vec v \sim \mathbb S_{n-1}$ indicates that $\vec v$ is sampled from the $(n-1)$-dimensional unit sphere $\mathbb S_{n-1}$ u.a.r.
\end{lemma}
\begin{proof} This proof follows the same lines as the proof of [Lemma 6.7]\cite{hazan2016introduction}.

\noindent \emph{Part I.} Consider the case when $\vec H = \delta \vec I$. 
Using Calculus Stoke's theorem, we have 
\begin{align}
    \nabla_{\x}\parentheses{ \int_{\mathbb B_\delta} l(\x + \vec v) d \vec v} = \int_{\mathbb S_\delta} l(\x + \vec u) \frac{\vec u}{\norm{\vec u}} d \vec u.\label{eq:stokes}
\end{align}
We have the following 
\begin{align}
    l^{\delta \vec I}(\x) = \E_{\vec v \sim \mathbb{B}} \interval{l(\x + \delta \vec v)}=  \frac{1}{\mathrm{vol}(\mathbb{B}_\delta)}\int_{\mathbb B_\delta} l(\x + \vec v) d \vec v,
\end{align}
 where ${\mathrm{vol}(\mathbb{B}_\delta)}$ is the volume of the $n$-dimensional ball of radius $\delta$. Similarly, we also have
 \begin{align}
     \E_{\vec u \sim \mathbb S_{n-1}} \interval{ l(\x + \delta \vec u) \vec u} = \frac{1}{\mathrm{vol}(\mathbb S_{n-1})} \int_{\mathbb S_\delta}l(\x + \vec u) \frac{\vec u}{\norm{\vec u}} d \vec u.
 \end{align}
The ratio of the volume of a ball in $n$ dimensions and the sphere of dimension $n-1$ of radii $\delta$ is $\mathrm{vol}(\mathbb{B}_n) / \mathrm{vol}{(\mathbb{S}_{n-1})} = \delta /n$. Combining these facts with Eq.~\eqref{eq:stokes}, we showed that:
\begin{align}
    \nabla l^{\delta \vec I} (\x) = \frac{n}{\delta}\E_{\vec v \sim \mathbb S_{n-1}} \interval{l(x+ \delta \vec v) \vec v}.\label{eq:gradient_estimator_identity}
\end{align}
\noindent \emph{Part II.} Let $g(\x) = l(\vec H \x)$, and $ g^{\vec I}(\x) = \E_{\vec v \in \mathbb{B}_n} \interval{ g(\x + \vec v)}$. We have the following 
\begin{align}
    n \E_{\vec v \sim \mathbb S_{n-1}} \interval{l(\x + \vec H \vec v) \vec{H}^{-1} \vec v} &= n \vec H^{-1} \interval{l(\x + \vec H \vec v)\vec v} \\
    &= n \vec H^{-1} \interval{g(\vec H^{-1}\x + \vec v)\vec v}\\
    &= \vec H^{-1} \nabla {g}^{\vec I}(\vec H^{-1} \x)\\
    &= \vec H^{-1} \vec H  \nabla  l^{\vec H}(\x) =\nabla  l^{\vec H}(\x).
\end{align}
This concludes the proof.
\end{proof}

\begin{lemma}\cite[Lemma~2]{zhang2022stochastic}
\label{lemma:aux}
Let $\efunc: \reals^n \to \reals$ be a monotone DR-submodular function 
and $\efunc(\vec 0) = 0$, $\vec x$, $\vec y \in \reals^n$. Let $\eFunc$ be the auxiliary function defined in~\eqref{def:aux}. Then    
\begin{align}
    &\nabla \eFunc(\x)  = \int^1_0 \exp (z-1) \nabla f (z \cdot \x ) dz, & (\y-\x) \cdot \nabla \eFunc(\x) &\geq (1 - 1/e) \efunc(\y) - \efunc(\x),&
\end{align}
\end{lemma}
\begin{proof}
This proof follows the same lines as the proof of \cite[Lemma 2]{zhang2022stochastic}. It is simplified for the case when $\theta (w) =1$ and ($w (z) = \exp(z-1)$). It follows from the DR-submodularity of $\efunc$:
\begin{align}
\efunc(\y) - \efunc(\x) = \int^1_{0} (\y -\x) \cdot \nabla \efunc(\x + z (\y - \x)) dz  \in [ (\y-\x)\cdot \nabla \efunc(\y),  (\y-\x)\cdot \nabla \efunc(\x)],
\end{align}
for $\x \leq \y$. The two inequalities follow from $\y \geq \x + z (\y - \x) \geq \x$ s.t. $\nabla \efunc(\x) \geq \nabla \efunc(\x + z (\y -\x)) \geq \nabla \efunc(\y)$ for any $z \in [0,1]$; a direct result from the DR-submodularity of $\efunc$, see Eq.~\eqref{eq:dr-submodularity_1}.

    The gradient of $\eFunc$ is given by $\nabla \eFunc(\x) = \nabla_{\x} \parentheses{\int^1_{0} \frac{\exp (z - 1)}{z} \efunc(z \cdot \vec x) dz} =\int^1_{0} \frac{z\exp (z - 1)}{z} \nabla \efunc (z \cdot \vec x) dz =\int^1_{0} \exp (z - 1) \nabla \efunc (z \cdot \vec x) dz $. Let $w(z) = \exp(z-1)$. Consider the following:
    \begin{align}
        \x \cdot \nabla \eFunc(\x) &= \int^1_{0} w(z) {\x\cdot\nabla \efunc(z\x)} dz = \int^1_{0} w(z) \frac{\partial \efunc(z\x)}{\partial z} dz =  \int^1_{0} w(z) {d\efunc(z\x)} \\
        &= \interval{ w(z)\efunc(z\x)}^1_{0} - \int^1_{0} w'(z) \efunc(z \x) dz =w(1) \efunc(\x) - \int^1_{0} w'(z) \efunc(z \x) dz \\
        &= \efunc(\x) - \int^1_{0} w(z) \efunc(z\x) dz.\label{eq:p1}
    \end{align}
    We use integration by part and the fact that $w'(z) = w(z)$.
    \begin{align}
        \y \cdot \nabla \eFunc(\y) &= \int^1_{0} w(z) {\y\cdot\nabla \efunc(z\x)} dz \geq \int^1_{0} w(z) {(\y \oplus z \x - z\x)\cdot\nabla \efunc(z\x)} dz \\
        &\geq  \int^1_{0} w(z) {(\efunc(\y \oplus z \x) - \efunc(z\x))\cdot\nabla \efunc(z\x)} dz \\
        &\geq \underbrace{\parentheses{\int^1_{0} w(z) dz} }_{\text{$1-1/e$}}\efunc(\y)  - \int^1_{0} w(z) \efunc(z \x) dz.\label{eq:p2}
    \end{align}
    The first inequality is obtained considering $\y \geq \y \oplus z\x - z\x \geq \vec 0 $ and $\nabla \efunc(z \x) \geq 0 $
    Combine Eqs.~\eqref{eq:p1} and \eqref{eq:p2} to obtain
    \begin{align}
        (\y - \x) \cdot \nabla \eFunc(\x) \geq (1-1/e) \efunc(\y) - \efunc(\x).
    \end{align}
    We conclude the proof.
\end{proof}

We define the following quantities which will be used in what follows.
\begin{align}
    \mathcal{H}_{l} \triangleq \set{\parentheses{\vec v_s, \vec u_s, z_s, t_s}: s \leq l}, \qquad 
    \efunc^\avg_l (\y) \triangleq \frac{1}{W}\sum_{t \in \mathcal W_l} \efunc_t(\y), \qquad
    \eFunc^\avg_l (\y) \triangleq \frac{1}{W}\sum_{t \in \mathcal W_l} \eFunc_t(\y), \qquad \text{ for $\y \in \Y$}. \label{eq:def_quantities}
\end{align}

\begin{lemma}\cite[Lemma~E.1]{wan2023bandit}
    Consider $\mathcal H_{l}$ and $\eFunc^{\avg}_t$  defined in Eq.~\eqref{eq:def_quantities},    the estimator $\tilde l_l$ in Alg.~\ref{alg:lisaoco} is an unbiased estimator for $\delta \vec v_l \cdot \nabla \eFunc^{\avg}_l(\vec y^\delta_{l})$, i.e., 
    \begin{align}
        \E \interval{\tilde l_l \given \mathcal{H}_{l-1}, \vec v_l} = \delta \vec v_l \cdot \nabla \eFunc^{\avg}_l(\vec y^\delta_l).
    \end{align}
\end{lemma}
\begin{proof} 
This proof follows the same lines as the proof of \cite[Lemma E.1]{wan2023bandit}.
    
    \noindent\emph{Part I ($z_t \geq 1/2$).} Condition on $\mathcal{H}_{l-1}, \vec v_l, \vec u_l, z_l, t_l$. If $z_l\geq1/2$ and $\vec u_l = \vec 0$, we have 
    \begin{align}
         \E \interval{\tilde l_l \given \mathcal{H}_{l-1}, \vec v_l, \vec u_l, z_l, t_l} = \E \interval{-2 \alpha \frac{n}{z_l} f_{t_l}(\x_{t_l})  \given \mathcal{H}_{l-1}, \vec v_l, \vec u_l, z_l, t_l} = - \frac{2 \alpha n}{z_l} \efunc_{t_l}(z_l \y^\delta_l). 
    \end{align}
    If $z_l\geq 1/2$ and $\vec u_l \neq \vec 0$, we have 
    \begin{align}
         \E \interval{\tilde l_l \given \mathcal{H}_{l-1}, \vec v_l, \vec u_l, z_l, t_l} = \E \interval{-2 \alpha \frac{n}{z_l} f_{t_l}(\x_{t_l})  \given \mathcal{H}_{l-1}, \vec v_l, \vec u_l, z_l, t_l} = \frac{2 \alpha n}{z_l} \efunc_{t_l}(z_l \y^\delta_l + z_l \delta \parentheses{\vec v_l  \odot \vec u_l}). 
    \end{align}
    Condition on $\mathcal{H}_{l-1}, \vec v_l, z_l, t_l$, to obtain 
    \begin{align}
        \E \interval{\tilde l_l \given \mathcal{H}_{l-1}, \vec v_l, z_l, t_l} &= \frac{1}{2} \parentheses{- \frac{2 \alpha n}{z_l} \efunc_{t_l}(z_l \y^\delta_l)} + \frac{1}{2n} \parentheses{\sum^n_{i=1} \frac{2 \alpha n}{z_l} \efunc_{t_l}(z_l \y^\delta_l + z_l \delta \parentheses{\vec v_l  \odot \vec e_i})}\\
        &=\alpha \sum^n_{i=1} \frac{1}{z_l} \parentheses{\efunc_{t_l} (z_l \y^\delta_l + z_l \delta \parentheses{\vec v_l  \odot \vec e_i}) -\efunc_{t_l}(z_l \y^\delta_l)}  \\
        &= \alpha \sum^n_{i=1} \frac{z_l}{z_l} \delta \vec v_l \cdot \vec e_i \frac{\partial \efunc_{t_l}(z_l \y^\delta_l)}{\partial x_i} = \alpha \delta \vec v_l \cdot \nabla \efunc_{t_l} (z_l \y^\delta_l).
    \end{align}
    \forcednewline \noindent\emph{Part II ($z_t < 1/2$).} Condition on $\mathcal{H}_{l-1}, \vec v_l, \vec u_l, z_l, t_l$. If $z_l<1/2$ and $\vec y_l = z_l \y^\delta_l$, we have 
\begin{align}
      \E \interval{\tilde l_l \given \mathcal{H}_{l-1}, \vec v_l, \vec u_l, z_l, t_l} &= \E \interval{-4 \alpha n \delta \parentheses{\vec v_l \cdot \vec u_l} f_{t_l}(\x_{t_l})  \given \mathcal{H}_{l-1}, \vec v_l,\vec u_l,  z_l, t_l} \\
      &=-4 \alpha n \delta \parentheses{\vec v_l \cdot \vec u_l} \efunc_{t_l}(z_l \y^\delta_l).
\end{align}
If $z_l<1/2$ and $\vec y_l = z_l \y^\delta_l + 1/2 \vec u_l $, we have 
\begin{align}
     \E \interval{\tilde l_l \given \mathcal{H}_{l-1}, \vec v_l, \vec u_l, z_l, t_l} &= \E \interval{4 \alpha n \delta \parentheses{\vec v_l \cdot \vec u_l} f_{t_l}(\x_{t_l})  \given \mathcal{H}_{l-1}, \vec v_l, \vec u_l, z_l, t_l} \\
     &=4 \alpha n \delta \parentheses{\vec v_l \cdot \vec u_l} \efunc_{t_l}( z_l \y^\delta_l + 1/2 \vec u_l ).
\end{align}
Condition on $\mathcal{H}_{l-1}, \vec v_l, z_l, t_l$, to obtain
\begin{align}
     \E \interval{\tilde l_l \given \mathcal{H}_{l-1}, \vec v_l, z_l, t_l} &= \sum^n_{i=1} \frac{1}{n} \parentheses{1/2 (4 \alpha n \delta \parentheses{\vec v_l \cdot \vec e_i} \efunc_{t_l}( z_l \y^\delta_l + 1/2 \vec e_i ) - 4 \alpha n \delta \parentheses{\vec v_l \cdot \vec e_i} \efunc_{t_l}(z_l \y^\delta_l))}\\
     &= \sum^n_{i=1} 2 \alpha \delta \parentheses{\vec v_l \cdot \vec e_i} \parentheses{\efunc_{t_l} ( z_l \y^\delta_l + 1/2 \vec e_i )  - \efunc_{t_l} ( z_l \y^\delta_l)}\\
     &= \sum^n_{i=1} 2/2 \alpha \delta \parentheses{\vec v_l \cdot \vec e_i} \frac{\partial\efunc_{t_l}(z_l \y^\delta_l)}{\partial x_i}\\
     &= \alpha \delta \vec v_l \cdot \nabla \efunc_{t_l} (z_l \y^\delta_l).
\end{align}

\noindent\emph{Part III. ($z_l \in [0,1]$)} Condition on $\mathcal{H}_{l-1}, \vec v_l$, to obtain
\begin{align}
      \E \interval{\tilde l_l \given \mathcal{H}_{l-1}, \vec v_l} &= \int^1_{0} \frac{\exp(z-1)}{\alpha} \frac{1}{W} \sum_{t \in \mathcal{W}_l}\parentheses{\alpha \delta \vec v_l \cdot \nabla \efunc_t (z_l \y^\delta_l)} dz\\
      &= \int^1_{0} \frac{\exp(z-1)}{\alpha} \parentheses{\alpha \delta \vec v_l \cdot \nabla \efunc^{\avg}_{l} (z_l \y^\delta_l)} dz\\
      &= \delta \vec v_l \cdot \nabla \eFunc^{\avg}_l(\vec y^\delta_l).
\end{align}
\end{proof}

\begin{lemma} \cite[Lemma 3.2]{wan2023bandit}
        Consider $\mathcal H_{l}$ and $\eFunc^{\avg}_t$  defined in Eq.~\eqref{eq:def_quantities}, the following holds for the supergradient estimator $\tilde{\vec g}_l$ in Alg.~\eqref{alg:lisaoco}:
    \begin{align}
        \E\interval{\tilde{\vec g}_l \given \mathcal{H}_{l-1}} = \nabla \eFunc^{\avg}_l(\y^\delta_l),\qquad \E\interval{\norm{\tilde{\vec g}_l}^2_2 \given \mathcal{H}_{l-1}} \leq 16 \frac{\alpha^2 n^4 L^2}{\delta^2}. 
    \end{align}
\end{lemma}

    \noindent\emph{Part I.}  We have the following
    \begin{align}
        \E\interval{\tilde{\vec g}_l \given \mathcal{H}_{l-1}} &=\E_{\vec v \sim \mathbb S_{n'-1}}\interval{ \E\interval{\tilde{\vec g}_l \given \mathcal{H}_{l-1}, \vec v} } =\E_{\vec v \sim \mathbb S_{n'-1}}\interval{ \frac{n'}{\delta} l_l \vec v} =  \E_{\vec v \sim \mathbb S_{n'-1}} \interval {n' \vec v \cdot \nabla \eFunc^{\avg}_l( \vec y^\delta_l) \vec v}\\
        &= \nabla l^{\vec I}_l(\vec 0) = \nabla l_l (\vec 0) = \nabla \eFunc^{\avg}_l(\vec y^\delta_l). \qquad\qquad \text{$(l_l(\x) \triangleq \x \cdot \nabla \eFunc^{\avg}_l(\y^\delta_l)).$}
    \end{align}Recall that $l^{\vec I}$ is the $\vec I$-smoothed version of $l$ in Def.~\ref{def:smoothed_func}, which are identical since $l$ is linear.
\begin{proof} This proof follows the same lines as the proof of \cite[Lemma 3.2]{wan2023bandit}.

\noindent\emph{Part II.} We have the following $|\tilde l_l| \leq 4 \alpha L$. It follows that $\norm{\tilde{\vec g}_l}_2 \leq 4 \frac{\alpha {n'}^2 L}{\delta} $. Thus, we have
\begin{align}
    \E\interval{\norm{\tilde{\vec g}_t}^2_2 | \mathcal{H}_{l-1}} \leq 16 \frac{\alpha^2 {n'}^4 L^2}{\delta^2}.
\end{align}
We conclude the proof.
\end{proof}

Recall that in the dynamic setting, the decision maker compares its performance to the best sequence of decisions $R(\x^\star_t)_{t \in \T} $ with a path length $P_T$ from the set   $$\Lambda_{\X} (T,P_T) = \set{\parentheses{\x_t}^T_{t=1} \in \X^T: \sum^T_{t=1} \norm{\x_{t+1} - \x_t} \leq P_T} \subset \X^T.$$ 
In the bandit setting, we define the following regularity condition that further restricts the movements of the comparator sequence:
\begin{align}
    \Lambda_{\X} (T, P_{T}, W) \triangleq \set{\parentheses{\x_t}^T_{t=1} \in \X^T: \sum^T_{t=1} \norm{\x_{t+1} - \x_t} \leq P_{T}, \x_t = \x_{t'}, t,t'\in \mathcal{W}_l, l \in [T/W] } \subset \X^T.
\end{align}
We extend the definition of the {dynamic $\alpha$-regret} in Sec.~\ref{sec:extensions}: 
\begin{align*}
     &\alpha\mathrm{-regret}_{T, P_T, W} ({\vec\P}_\X) \triangleq \sup_{\parentheses{{f}_t}^T_{t=1} \in {\F}^T} \bigg\{ \max_{\parentheses{\x^\star_t}^T_{t=1}\in    \Lambda_{\X} (T,P_T, W)}  \alpha\sum^T_{t=1} {f}_t(\x^\star_t) -\sum^T_{t=1} {f}_t(\x_t)\bigg\} 
\end{align*}
Note that this dynamic regret definition reduces to the static regret for $P_T=0$, and the definition in the full-information setting in Sec.~\ref{sec:extensions} for $W=1$ (i.e., $\alpha\mathrm{-regret}_{T} ({\vec\P}_\X) = \alpha\mathrm{-regret}_{T, 0, W} ({\vec\P}_\X)$ and $\alpha\mathrm{-regret}_{T, P_T, 1} ({\vec\P}_\X) = \alpha\mathrm{-regret}_{T, P_T} ({\vec\P}_\X)$).

\paragraph{Theorem~6 (Restatement).}  \emph{Under uniformly $L$-bounded  submodular monotone rewards and partition matroid constraint sets, LIRAOCO policy $\vec\P_\X$  in Alg.~\ref{alg:lisaoco} equipped with  an OCO policy $\vec\P_{\Y_\delta}$  
yields 
 $   \alpha\mathrm{-regret}_{T, P_T,W}  \parentheses{\vec\P_\X} \leq  W \cdot \mathrm{regret}_{T/W, P_T }\parentheses{\vec{\P}_{\Y_\delta}} + \frac{L T}{W} + 2 \alpha \delta r^2 n  L T,$}
 \emph{where $\mathrm{regret}_{T/W, P_T} \parentheses{\vec{\P}_{\Y_\delta}}$ is the regret of an OCO policy executed for $T/W$ timeslots.} 
\begin{proof}


The performance of $\vec \P_\X$ is measured against the  comparator sequence $$(\x_{\star, t})^T_{t=1} \in \argmax_{(\x_t)^T_{t=1} \in \Lambda_{\X} (T,P_T, W) } \sum^T_{t=1} f_t(\x_t).$$  
From the definition of the extension $\efunc$ in Lemma~\ref{lemma:bandit_extention}, we can always find a point $\y_{\star,t} \in \Y$ such that $f_t(\x_{\star,t}) = \efunc_t(\y_{\star,t})$  for every $t \in [T]$,  $\sum^T_{t=1} \norm{\y_{\star,t} - \y_{\star,t}} = \sum^T_{t=1} \norm{\x_{\star,t+1} - \x_{\star, t}} \leq P_T$, and $\y_{\star, t} = \y_{\star,t'}$ for $t,t' \in \mathcal{W}_l$ for $l \in [T/W]$ since $\x_{\star, t}$ is an integral point in $\set{0,1}^n$. We define $\y^\star_l = \y^\star_t$ for some  $t \in \mathcal W_l$,  $l \in [T/W]$.

\begin{align}
      &\E \interval{  \sum^T_{t=1}(1-1/e) f_t(\x_{\star, t}) - f_t (\x_t)} \\
      &=   \E \interval{  \sum^T_{t=1}(1-1/e) \efunc_t(\y_{\star, t}) - \E \interval{f_t (\x_t) \,\big| \,\mathcal{H}_{t-1}}}
  \\
  &=  \E \interval{  \sum^T_{t=1}(1-1/e) \efunc_t(\y_{\star, t}) - \efunc_t (\y_t)} \qquad \qquad \qquad \text{( $ \efunc_t (\y)= \E_{\x \sim \hat\Xi (\y)} \interval{f_t(\x)}$ in Lemma~\ref{lemma:bandit_extention})} \\
   &= \E \interval{  \sum^T_{t=1}(1-1/e) \efunc_t(\y_{\star, t}) - \efunc_t (\y_{\ceil{t/W}})} + \E \interval{  \sum^T_{t=1}\efunc_t (\y_{\ceil{t/W}}) - \efunc_t (\y_{t})} \\
   &=  \E \interval{ W \sum^{T/W}_{l=1} (1-1/e) {\efunc}^{\avg}_l(\y_{\star, l}) - {\efunc}^{\avg}_l (\y^\delta_{\ceil{t/W}})} + \E \interval{  \sum^{T/W}_{l=1}\efunc_t (\y^\delta_{l}) - \efunc_t (\y_{t_l}) \big| \mathcal{H}_{l-1}} \\
   &\leq  \E \interval{ W \sum^{T/W}_{l=1} (1-1/e) {\efunc}^{\avg}_l(\y^{\delta}_{l, \star}) - {\efunc}^{\avg}_l (\y^\delta_{\ceil{t/W}})}  + \frac{L T }{W} + 2 \alpha \delta r^2 n L T. 
\end{align}
The last inequality is obtained using $\efunc_t (\y^\delta_{l}) - \efunc_t (\y_{t_l}) \leq L$ to yield $\frac{L T }{W}$ term; moreover, consider that $\norm{\y_\star - \y^{\delta}_{t, \star}}_2 \leq \delta D$ ($D \leq 2r$) and $\efunc$ is $n' L$-Lipschitz (w.r.t. $\norm{\,\cdot\,}_2$) (see Lemma~\ref{lemma:bandit_extention}), this yields an additional $2 \alpha \delta n r^2 L T$ term ($n' \leq r n$).  It remains to bound $\E \interval{\sum^{T/W}_{l=1} (1-1/e) {\efunc}^{\avg}_l(\y^\delta_{l,\star}) - {\efunc}^{\avg}_l (\y^\delta_{\ceil{t/W}})}$ to complete the proof:

\begin{align}
  \E \interval{\sum^{T/W}_{l=1} (1-1/e) {\efunc}^{\avg}_l(\y^\delta_{\star,l}) - {\efunc}^{\avg}_l (\y^\delta_{\ceil{t/W}})}  &\leq  \sum^{T/W}_{l=1}  \E \interval{\nabla \hat{F}^{\avg}_l(\y^\delta_l) \cdot (\y^{\delta}_{l, \star} - \y^\delta_l ) \,\big|\, \mathcal{H}_{l-1}}  \\
  &\leq \mathrm{regret}_{T/W, P_{T}} \parentheses{\vec{\P}_{\Y_\delta}}.  
\end{align}
The first inequality follows from Lemma~\ref{lemma:aux}. The second inequality holds because the gradient estimates $\tilde{\vec{g}}_l$ are unbiased estimates of $\nabla \eFunc^{\avg}_l$ and the reduction of the regret bound of the first-order policy $\P_{\Y_\delta}$~\cite[Lemma 6.5]{hazan2016introduction}. We finally obtain the following:
\begin{align}
     \E \interval{  \sum^T_{t=1}(1-1/e) f_t(\x_{\star, t}) - f_t (\x_t)}  \leq W \cdot \mathrm{regret}_{T/W, P_T} \parentheses{\vec{\P}_{\Y_\delta}} + \frac{L T}{W} + 2 \alpha \delta r^2 n  L T.
\end{align}

\end{proof}

LIRAOCO policy $\vec\P_\X$  in Alg.~\ref{alg:lisaoco} equipped with   OOMA in Alg.~\ref{alg:ooma} configured with $\Phi$ satisfying Asms.~\ref{asm:mirror_map}--\ref{asm:bounded_gradient_mirror_map} as $\vec\P_{\Y_\delta}$  has the following regret:\footnote{Note that the result in Theorem~\ref{theorem:bandit} readily extends to dynamic regret as in the proof of Theorem~\ref{thm:dynamic}.}
     \begin{align}
         \aregret \parentheses{\vec\P_\X} \leq \frac{W}{\eta}\Big(4 r^2 + 2 L_{\Phi} {P_T}\Big) + \frac{W \eta }{2\rho} \parentheses{\sum_{l \in T/W}  \E \interval {\norm{\g_l - \pg_l}^2_\star}}+ \frac{L T}{W}  + 2 \alpha \delta r^2 n L T.
\end{align}

This bound is obtained considering the regret bound in Theorem~\ref{alg:lisaoco} and OOMA policy regret bound in Theorem~\ref{thm:general_regret_dynamic+optimistic}. Consider an Euclidean mirror map $\Phi(\y) = \frac{1}{2} \norm{\y}_2^2$ with the corresponding bounds in Table~\ref{tab:bounds}. 

\noindent\textbf{Static non-optimistic setting.} In the static non-optimistic setting ($P_T = 0, \vec g_l^\pi = \vec 0$), we have:
 \begin{align}
         \aregret \parentheses{\vec\P_\X} &\leq\frac{4 r^2  W}{\eta} +\eta \parentheses{ 8 \frac{\alpha^2 r^4 n^4 L^2}{\delta^2} T} + \frac{L T}{W}  + 2 \alpha \delta r^2 n L T\\
         &\leq (12 \alpha r^{5/2} n^{3/2}+1) L T^{4/5} 
\end{align}
For $\eta =  \frac{32^{1/4}}{2^{1/2} \alpha r^{1/2} n^{3/2} L}  T^{-3/5}$, where  $\delta = (32)^{1/4} (r^{1/2} n^{1/2}) T^{-1/5}$, $W = T^{1/5}$.

\noindent\textbf{Dynamic non-optimistic setting.}  In the static non-optimistic setting ($P_T > 0$ and $\vec g_l^\pi = \vec 0$). We have the following:
\begin{align}
    \aregret \parentheses{\vec\P_\X} = \BigO{\sqrt{P_T} T^{4/5}}.
\end{align}
For $\eta = \Theta\parentheses{\sqrt{P_T} T^{-3/5}}$, $\delta = \Theta \parentheses{T^{-1/5}}$, $W = \Theta\parentheses{T^{1/5}}$.

\noindent\textbf{Dynamic optimistic setting.}  In the dynamic optimistic setting ($P_T > 0$ and $\vec g_l^\pi \neq  \vec 0$). We have the following:
\begin{align}
    \aregret \parentheses{\vec\P_\X} = \BigO{ T^{\alpha}\sqrt{P_T \sum^{T/W}_{l=1} \norm{\tilde{\vec g}_{l} - \vec g^\pi_l}^2_2} + T^{1- \beta} + T^{1-\alpha}}.
\end{align}
For $\eta = \Theta \parentheses{\sqrt{P_T / \sum^{T/W}_{l=1} \norm{\tilde{\vec g}_{l} - \vec g^\pi_l}^2_2}}$, $\delta = \Theta \parentheses{T^{-\beta}} $, $W = \Theta \parentheses{T^\alpha}$. To interpret the above bound, consider for simplicity that $P_T = \Theta\parentheses{1}$ and  that the optimistic gradients $\g^\pi_t$ are inaccurate (but bounded by some constant). Then, when we take $\eta = \Theta\parentheses{T^{-3/5}}$, $\delta = \Theta \parentheses{T^{-1/5}}$, $W = \Theta\parentheses{T^{1/5}}$ and recover the worst-case regret  $\aregret \parentheses{\vec\P_\X}  = \BigO{T^{4/5}}$. However, when the optimistic gradients are accurate the above bound can be much tighter. For example, in the extreme scenario when $\sqrt{\sum^{T/W}_{l=1} \norm{\tilde{\vec g}_{l} - \vec g^\pi_l}^2_2} = \epsilon$ for $W = \Theta \parentheses{T}$, selecting $\delta =\Theta { \parentheses{T^{-1}}}$, and $\eta =  \Theta\parentheses{\frac{1}{\epsilon}}$, yields $\aregret  = \BigO{\epsilon}$.

\paragraph{Specialized Algorithms for Static Regret.} The regret attained through the reduction to an arbitrary OCO algorithm  ($\aregret \parentheses{\vec\P_\X}  = \BigO{T^{4/5}}$) is higher than with the specific instance of FTRL used by~\citet{wan2023bandit} ($\BigO{T^{2/3}}$), which however does not extend to the dynamic and optimistic settings. 

In particular, in the static regret setting, we can recover the bound by \citet{wan2023bandit} by instantiating the generic reduction in the manner they do, by employing a specialized OCO algorithm with appropriate exploration matrices $\vec H_l$. \citet{wan2023bandit}
 showed that when the OCO policy $\P_\Y$ is selected to be FTRL in Alg.~\ref{alg:ftrl} with a \emph{self-concordant regularizer}~\cite{hazan2016introduction} $R(\y): \reals^{n} \to \reals$, the exploration matrix can be  configured as $\H_l = \parentheses{\nabla^{2} R(\y_l)}^{1/2}$. Then, the sampled action $\y_l  + \vec H_l \vec v_l$ is located on the surface of the \emph{Dikin ellipsoid} centered at $\y_l$ defined as 
 \begin{align}
     \mathcal{E} (\y_l) \triangleq \set{\y \in \reals^n : \norm{\y - \y_l}_{l}\leq 1},
 \end{align}
where $\norm{\y}_l$ is a local norm given by $ \norm{\y}_l = \sqrt{\y^\intercal \nabla^2 R(\y_l) \y}$. The Dikin ellipsoid has the nice property that it is fully contained in $\Y$. This permits selecting $\delta = 0$, and this reduces the variance of the gradient estimates due to $1/\delta$ scaling. \citet{wan2023bandit} demonstrate that this configuration attains a regret $\aregret \parentheses{\vec\P_\X} = \BigO{T^{2/3}}$.

\subsection{Alternate Concave Relaxation to Reduce Computational Complexity}\label{appendix:alternate_relaxation}
In our setup so far, the concave relaxation  $\cfunc$ was the WTP function ``itself''. As an additional extension, we present a setting in which a concave relaxation is different from $f$, while still satisfying Asm.~\ref{assumption:sandwich}.  This example demonstrates that our framework and, in particular, Asm.~\ref{assumption:sandwich} provide additional flexibility to design algorithms with low time complexity. 

Consider the following set function
\begin{align}
    f(\x) = \sum_{\ell \in C} c_\ell \parentheses{ b_\ell - b_\ell \prod_{j \in S_\ell} (1 -  (w_{\ell,j}/b_\ell) x_{j})}, &&\text{$\x \in \set{0,1}^n$,}
\end{align}
where   $b_\ell \in \reals_{\geq 0} $,  $S_\ell \subseteq V$ is a subset of $V=[n]$,   $\vec w_\ell = (w_{\ell,j})_{j \in S}\in [0,b]^{|S|}$ bounded by $b_\ell$, and $c_\ell \in 
\reals_{\geq 0}$ for $\ell \in C$ for some index set $C$. 

This objective appears in many applications such as product ranking~\cite{ferreira2022learning,niazadeh2021online} and influence maximization~\cite{kempe2003maximizing, karimi2017stochastic}. Note for $\vec w_\ell = \vec 1, \ell \in C$, this set function reduces to the weighted coverage~\eqref{eq:wcf}. 

Lemma~\ref{lemma:tech_equiv_expression} implies that the function $f$  belongs to the class of WTP functions; in particular, 
 it can be expressed as follows:
\begin{align}
    f(\x) &= \sum_{\ell \in C}c_\ell\parentheses{ b_\ell -  b_\ell \sum_{S \subseteq S_\ell} \prod_{i \in S} (1-x_{i}) \prod_{i \in S} w_{\ell,i}/b_\ell\prod_{j \in S_\ell\setminus S} (1-w_{\ell,j}/b_\ell)}.
\end{align}
We obtain this equality by considering  that $1- \prod_{i \in S} (1- x_i) = \Psi_{1, \vec 1, S} (\x)$ for all $\x \in \set{0,1}^{|S|}$ and, hence:
\begin{align}
   f(\x) &=\sum_{\ell \in C}c_\ell\parentheses{ b_\ell -  b_\ell \sum_{S \subset S_\ell} \prod_{i \in S} w_{\ell,i}/b_\ell\prod_{j \in S_\ell\setminus S} (1-w_{\ell,j}/b_\ell)\parentheses{\Psi_{1, 
    \vec 1, S} (\x) - 1 }} \\
    &= \sum_{\ell \in C}c_\ell b_\ell \sum_{S \subset S_\ell} \parentheses{\prod_{i \in S} w_{\ell,i}/b_\ell\prod_{j \in S_\ell\setminus S} (1-w_{\ell,j}/b_\ell)}\Psi_{1, 
    \vec 1, S} (\x) \triangleq h(\x)
    & & \text{for $\x \in \set{0,1}^n$.}
\end{align}
The second inequality follows from the fact that $1 = \sum_{S \subset S_\ell} \prod_{i \in S} w_{\ell,i}/b_\ell\prod_{j \in S_\ell\setminus S} (1-w_{\ell,j}/b_\ell)$ for $\ell \in C$ (see Lemma~\ref{lemma:tech_equiv_expression}).

A straightforward approach to construct a concave relaxation of $f$ is to extend $h(\x)$ to fractional values $\y \in [0,1]^n$, i.e., $\cfunc (\y) = h(\y), \forall \y \in [0,1]^n$ as in Sec.~\ref{s:osm_wtp}. However, \emph{this extension is intractable, since it involves a summation of potentially $2^n$ terms}. Alternatively, we can consider the concave relaxation given by 
\begin{align}
    \cfunc (\y) = \sum_{\ell \in C} c_\ell \Psi_{b_\ell, \vec w_\ell, S_\ell} (\y) && \text{ for $\y \in [0,1]^n.$}
\end{align}
It is easy to check that Lemmas~\ref{lemma:util_bounds_lower}--\ref{lemma:util_bounds_upper} emply that this is a valid concave function satisfying Asm.~\ref{assumption:sandwich}. This relaxation is of interest for two reasons: (1) it provides an example of the ``sandwich'' property under which the two functions do not necessarily exactly match over integral values, and (2) it illustrates that  carefully designing the concave relaxation  can greatly change the computational complexity of the proposed method in Sec.~\ref{s:osm_genral}.

\section{Additional Experimental Details \& Results}
\label{appendix:experiments}

\subsection{Datasets \& Experiment Setup}
We provide additional details about the datasets we use and our experiment setup below. For all experiments reported in Tables~\ref{tab:res} and~\ref{tab:frac_res_large}, we repeat experiments 5 times per policy; we ensure that the adversary presents the same five sequences across all competitor policies. 

\paragraph{Influence Maximization.}
    In the case of influence maximization (see Appendix~\ref{appendix:wcf}), we use the Zachary Karate Club (\IMZKC) and the Epinions (\IMEpinions) datasets \cite{nr}. 
    We denote by $m$ the number of partitions of the partition matroid. 
    We also sort the number of nodes of the original graph by their degrees and we divide the nodes into $m=2$ partitions where sorted nodes with even indices are assigned to one partition and the odd indices are assigned to the other. At every timeslot $t$, the online policy selects $\frac{r_{\mathrm{part}}}{2}=2$ seeds from each partition. Then, the adversary generates
   a cascade reachability graph following the independent cascade (IC) model,  \cite{kempe2003maximizing} by independently sampling the edges of the original graph with probability $p = 0.1$. We repeat the above for $T=100$ timeslots. 
    We follow a similar procedure with the \IMEpinions{} dataset. This time, we generate a subgraph of the original dataset by sorting the nodes by their outdegrees and keeping the top $200$ nodes. Then, we sample the edges of this subgraph with probability $p=0.1$ and generate $T = 150$ instances. We divide the dataset into $m=2$ equal-sized partitions u.a.r and select at most $\frac{r_{\mathrm{part}}}{2}=5$ from each partition.
\paragraph{Facility Location.}      For this application, we experiment on a subset of the MovieLens 10M (\FLMovieLens{}) dataset \cite{movielens}. The original dataset has more than $10M$ ratings in total from $71567$ users for $10681$ movies. We sort the users based on the number of movies they have rated and keep the top $T=294$ of these users. Then, we take the user with the most number of ratings and only keep the movies that have been rated by this user in our subset. This leaves us with $|C|=21$ movies. In the facility location context, we treat movies as facilities, users as customers, and normalized ratings $w_{i,j}$ as the utilities. For the partition matroid constraint, we divide the movies into $m = 6$ partitions based on the first genre listed for each movie. At each timeslot, the adversary selects a user in the order they show up in the original dataset. The online algorithm maximizes the total utility by selecting $\frac{r_{\mathrm{part}}}{6} = 1$ movie from each genre. For the uniform matroid constraint, we set $r_{\mathrm{uni}}=6$.  
\paragraph{Team Formation.}     Our experiments for team-formation 
focus on objective functions that are monotone, submodular quadratic functions  (see Appendix~\ref{appendix: quadratic}). 
We generate $5$ monotone and submodular quadratic functions as follows: we generate the $h$ vector by sampling $n$ numbers from a normal distribution (with $\mu = 60$ and $\sigma = 20$) and enforcing every coordinate of $\mathbf{h}$ to be between $0$ and $100$. We generate the $\mathbf{H}$ (symmetric) matrix by sampling each entry ${H}_{i,j} = {H}_{j,i}$ from a normal distribution (with $\mu = -20$ and $\sigma = 10$). We enforce every entry to be less than or equal to 0 and the main diagonal of $H$ to be equal to 0. We shrink the rows (and columns) of $\mathbf{H}$ by multiplying them with a constant factor until the function satisfies the monotonicity property. 
    We consider a pool of $n = 100$ individuals. At timeslot $t$, the adversary selects one of the five functions uniformly at random and the online algorithm forms a team maximizing the function. We consider a total of $T = 100$ timeslots. For the uniform matroid constraint, the online algorithm can select at most $r_{\mathrm{uni}}=2$ individuals. For the partition matroid constraint, we divide the individuals into $m = 2$ partitions and the online algorithm selects at most $\frac{r_{\mathrm{part}}}{2} = 2$ individuals from each partition.
\paragraph{Synthetic Weighted Coverage.} We generate two variants of the weighted-coverage problem with a ground set size $n=20$, designed to study reward function distribution shifts in the dynamic regret setting (Appendix~\ref{appendix: dynamic regret}), and an additional variant designed to study optimistic learning (Appendix~\ref{appendix: optimistic learning}).  In all cases, we consider a uniform matroid constraint with rank $r_{\mathrm{uni}}=5$. The objectives are designed such that the maximizing decisions are disjoint under the two objectives. Each objective is the sum of $|C| = 63$ threshold potentials.  Our first (stationary setting) problem considers a fixed function equal to the average of the two objectives. In the second (time-varying setting) problem, we set the rewards to be equal to the first objective for the first $\frac{T}{2} = 25$ timeslots, and then abruptly set the rewards to be the second objective for the remaining $\frac{T}{2} = 25$ timeslots. Finally, in the third (also time-varying) setting,  we alternate between the two objectives at each timeslot.

\subsection{Policy Implementation}
All timing experiments were run on a machine with a Broadwell CPU and 128GB RAM. 
All policies, including ones by other researchers ($\texttt{FSF}^\star$ and \texttt{TabularGreedy}), were implemented by us in Python. 
For solving convex optimization problems, we used the CVXPY package, which is an open-source python library. 

Even though hyperparameters differ per algorithm, for a fair comparison, we explore the same number of configurations across all our algorithms and baselines. For the \RAOCOOGA{} algorithm we try $\eta$ values from the set $\{0.001, 0.01, 0.1, 0.5, 1, 1.5, 2, 2.5, 3, 3.5, 4, 6, 8, 10\}$. For the \RAOCOOMA{}, \Algorithm{FSF$^\star$}, and \Algorithm{TabularGreedy} algorithms; we perform grid search where $(\eta, \gamma) \in \{0.05, 0.1, 6.5, 10\} \times \{0.001, 0.01, 0.05, 0.1\}$, $(\eta, \gamma) \in \{1, 10, 75, 100\} \times \{0, 0.001, 0.01, 0.1\}$, and $(\eta, c_p) \in \{0.1, 1, 10, 160\} \times \{1, 2, 4, 8\}$ respectively. In order to choose these values, we first perform an exponential search and then conduct a linear search on the interval with maximum reward.
\begin{table}[!t]
    \centering
    \begin{tiny}
    \begin{tabular}{|c|c|c|c|c|c|c|c|c|c|c|c|c|}
    \hline
         \textbf{Dataset} & \textbf{Problem} &  \textbf{${T}$} & $n$ & $|C|$ & 
         $\mathrm{min}(|S_\ell|)$ & $\max(S_\ell)$ & $\mathrm{avg}(|S_\ell|)$ & 
         \textbf{${m}$} & $r_{\texttt{part}}$ & $r_{\texttt{uni}}$\\ 
         
         \hline
         \IMZKC & Inf. Max. & $100$ & $34$ & $34$ & 
         $1$ & $8$ & $1.27$ & 
         $2$ & $4$ & $4$\\
         \IMEpinions& Inf. Max.  & $150$ & $200$ & $200$ & 
         $1$ & $47$ & $2.80$ & 
         $2$ & $10$ & $10$\\
         \FLMovieLens{} & Fac. Loc. & $294$ & $21$ & $21$ &
         $1$ & $21$ & $1.80$ & 
         $6$ & $6$ & $6$\\
         \TeamFormation{}& Team Form.  & $100$ & $100$ & $4951$ & 
         $2$ & $100$  & $2.02$ & 
         $2$ & $4$ & $2$\\
         \WeightedCoverage{}& Weight. Cov.  & $50$ & $20$ & $24$ & 
         $1$ & $4$ & $2.04$ & 
         $4$ & 8 & $5$\\       
    \hline
    \end{tabular}
    \end{tiny}
    \caption{Dataset properties and experiment parameters. The columns $T$, $n$, $|C|$ specify the number of timeslots, the number of elements of the matroid, and the number of threshold potentials in each objective function, respectively. The operators $\avg$, $\min$, and $\max$ are applied on possible values of the input parameters $S_\ell$ of the WTP class in Eq.~\eqref{eq:wtp}. The columns $m$, $r_{\mathrm{part}}$, and $r_{\mathrm{uni}}$ specify the number of partitions in the partition matroid, the rank of the partition matroid, and the rank of the uniform matroid, respectively, used across the different problem instances.}
    \vspace*{-10pt}
    \label{tab:exp}
\end{table}

\begin{table}[!t]
    \centering
    \resizebox{\columnwidth}{!}{%
    \begin{scriptsize}
    \begin{tabular}{|c|c|c|c||c|c||c|c||c||c||c|c|c|}
    \cline{5-13}
         \multicolumn{4}{c}{} & \multicolumn{2}{|c||}{\RAOCOOGA{}} & \multicolumn{2}{c||}{\RAOCOOMA{}} & \multicolumn{1}{c||}{$\texttt{FSF}^\star$} & \multicolumn{1}{c||}{\texttt{TabularGreedy}} & \multicolumn{3}{c|}{\texttt{Random}} \\
    \hline
        \multicolumn{2}{|c|}{\makecell{datasets \\ \& constr.}} & $F^\star$ &$t$ & $\bar{F}_{\Y}/F^\star$ & \makecell{avg. time per\\timeslot (s)} &  $\bar{F}_{\Y}/F^\star$ & \makecell{avg. time per\\timeslot (s)} & 
        \makecell{avg. time per\\timeslot (s)} & 
        \makecell{avg. time per\\timeslot (s)} & $\bar{F}_{\Y}/F^\star$ & std. dev. & \makecell{avg. time per\\timeslot (s)} \\
    \hline
         \multirow{6}{*}{\rotatebox{90}{\IMZKC{}}} & \multirow{3}{*}{\rotatebox{90}{Uni.}} & \multirow{3}{*}{$0.234$} & $33$ & $0.912$ & \multirow{3}{*}{$7.62 \times 10^{-3}$} & $\mathbf{0.957}$ & \multirow{3}{*}{$7.71 \times 10^{-3}$} & 
         \multirow{3}{*}{$0.158$} & 
         \multirow{3}{*}{$7.63 \times 10^{-2}$} & $0.646$ & $0$ & \multirow{3}{*}{$3.32 \times 10^{-3}$} \\
    \cline{4-5}\cline{7-7}
    \cline{11-12}
         & & & $66$ & $0.929$ & & $\mathbf{0.965}$ &  & 
         & & 
         $0.643$ & $0$ &  \\
    \cline{4-5}\cline{7-7}
    \cline{11-12}
             & & & $99$ & $0.948$ & & $\mathbf{0.981}$ & & 
             &  & 
             $0.640$ & $0$ &  \\
    \cline{2-13}
         & \multirow{3}{*}{\rotatebox{90}{Part.}}               & \multirow{3}{*}{$0.83$} & $33$ & $0.995$ & \multirow{3}{*}{$6.85 \times 10^{-3}$} & $\mathbf{0.996}$ & \multirow{3}{*}{$7.09 \times 10^{-3}$} & \multicolumn{1}{c||}{\multirow{3}{*}{\xmark}}  & 
         \multirow{3}{*}{$1.97 \times 10^{-2}$} & $0.957$ & $0$ & \multirow{3}{*}{$2.84 \times 10^{-3}$}\\
    \cline{4-5}\cline{7-7}
    \cline{11-12}
         & & & $66$ & $0.99$ &  & $\mathbf{0.993}$ & & \multicolumn{1}{c||}{}& 
         & $0.953$ & $1.34 \times 10^{-16}$ & \\
    \cline{4-5}\cline{7-7}
    \cline{11-12}
         & & & $99$ & $0.993$ &  & $\mathbf{0.996}$ &  & \multicolumn{1}{c||}{} & 
         & $0.955$ & $0$ &  \\
    \hline
    \multirow{6}{*}{\rotatebox{90}{\IMEpinions{}}} & \multirow{3}{*}{\rotatebox{90}{Uni.}} & \multirow{3}{*}{$0.171$} & $50$ & $\mathbf{0.919}$ & \multirow{3}{*}{$7.27 \times 10^{-2}$} & $0.905$ & \multirow{3}{*}{$7.4 \times 10^{-2}$} 
    & \multirow{3}{*}{$14.2$} & 
    \multirow{3}{*}{$6.69$} & $0.724$ & $0$ & \multirow{3}{*}{$1.93 \times 10^{-2}$} \\
    \cline{4-5}\cline{7-7}
    \cline{11-12}
         & & & $100$ & $\mathbf{0.937}$ & & $0.932$ &  & 
         & 
         & $0.711$ & $0$ &  \\
    \cline{4-5}\cline{7-7}
    \cline{11-12}
             & & & $149$ & $\mathbf{0.946}$ & & $0.943$ & & 
             & 
             & $0.720$ & $0$ &  \\
    \cline{2-13}
         & \multirow{3}{*}{\rotatebox{90}{Part.}}               & \multirow{3}{*}{$0.171$} & $50$ & $\mathbf{0.918}$ & \multirow{3}{*}{$7.14 \times 10^{-2}$} & $0.902$ & \multirow{3}{*}{$7.41 \times 10^{-2}$} & \multicolumn{1}{c||}{\multirow{3}{*}{\xmark}}  & 
         \multirow{3}{*}{$3.3$} & $0.724$ & $0$ & \multirow{3}{*}{$1.86 \times 10^{-2}$}\\
    \cline{4-5}\cline{7-7}
    \cline{11-12}
         & & & $100$ & $\mathbf{0.938}$ &  & $0.929$ & & \multicolumn{1}{c||}{} & 
         & $0.711$ & $8.12 \times 10^{-17}$ & \\
    \cline{4-5}\cline{7-7}
    \cline{11-12}
         & & & $149$ & $\mathbf{0.949}$ &  & $0.941$ &  & \multicolumn{1}{c||}{}  & 
         & $0.720$ & $0$ &  \\
    \hline
    \multirow{6}{*}{\rotatebox{90}{\FLMovieLens{}}} & \multirow{3}{*}{\rotatebox{90}{Uni.}} & \multirow{3}{*}{$0.407$} & $98$ & $0.838$ & \multirow{3}{*}{$1.26 \times 10^{-2}$} & $\mathbf{0.875}$ & \multirow{3}{*}{$1.25 \times 10^{-2}$} & 
    \multirow{3}{*}{$1.29 \times 10^{-2}$} & 
    \multirow{3}{*}{$7.83 \times 10^{-3}$} & $0.833$ & $0$ & \multirow{3}{*}{$1.18 \times 10^{-3}$} \\
    \cline{4-5}\cline{7-7}
    \cline{11-12}
         & & & $196$ & $0.817$ & & $\mathbf{0.820}$ &  & 
         & 
         & $0.774$ & $0$ &  \\
    \cline{4-5}\cline{7-7}
    \cline{11-12}
             & & & $293$ & $0.868$ & & $\mathbf{0.893}$ & & 
             & 
             & $0.799$ & $0$ &  \\
    \cline{2-13}
         & \multirow{3}{*}{\rotatebox{90}{Part.}}               & \multirow{3}{*}{$0.419$} & $98$ & $0.896$ & \multirow{3}{*}{$1.2 \times 10^{-2}$} & $\mathbf{0.96}$ & \multirow{3}{*}{$4.69 \times 10^{-3}$} & \multicolumn{1}{c||}{\multirow{3}{*}{\xmark}} & 
         \multirow{3}{*}{$1.71 \times 10^{-2}$} & $0.882$ & $0$ & \multirow{3}{*}{$1.04 \times 10^{-3}$}\\
    \cline{4-5}\cline{7-7}
    \cline{11-12}
         & & & $196$ & $0.874$ &  &$ \mathbf{0.911}$ & & \multicolumn{1}{c||}{} & 
         & $0.852$ & $0$ & \\
    \cline{4-5}\cline{7-7}
    \cline{11-12}
         & & & $293$ & $0.940$ &  & $\mathbf{0.949}$ &  & \multicolumn{1}{c||}{}  & 
         & $0.903$ & $0$ &  \\
    \hline
    \multirow{6}{*}{\rotatebox{90}{\TeamFormation{}}} & \multirow{3}{*}{\rotatebox{90}{Uni.}} & \multirow{3}{*}{200} & $33$ & $0.983$ & \multirow{3}{*}{$0.515$} & $\mathbf{0.983}$ & \multirow{3}{*}{$0.517$} & 
    \multirow{3}{*}{$45.5$} & 
    \multirow{3}{*}{$44.4$} & $0.611$ & 0 & \multirow{3}{*}{$0.224$} \\
    \cline{4-5}\cline{7-7}
    \cline{11-12}
         & & & $66$ & $0.992$ & & $\mathbf{0.995}$ & &
         & 
         & $0.609$ & $0$ &  \\
    \cline{4-5}\cline{7-7}
    \cline{11-12}
             & & & $99$ & $0.995$ & & $\mathbf{0.999}$ & & 
             & 
             & $0.611$ & 0 &  \\
    \cline{2-13}
         & \multirow{3}{*}{\rotatebox{90}{Part.}}               & \multirow{3}{*}{400} & $33$ & $0.982$ & \multirow{3}{*}{$0.526$} & $\mathbf{0.985}$ & \multirow{3}{*}{$0.518$} & \multicolumn{1}{c||}{\multirow{3}{*}{\xmark}} & 
         \multirow{3}{*}{$22.2$} & $0.611$ & 0 & \multirow{3}{*}{$0.228$}\\
    \cline{4-5}\cline{7-7}
    \cline{11-12}
         & & & $66$ & $0.99$ &  & $\mathbf{0.992}$ & &  \multicolumn{1}{c||}{} & 
         & $0.609$ & $0$ & \\
    \cline{4-5}\cline{7-7}
    \cline{11-12}
         & & & $99$ & $0.993$ &  & $\mathbf{0.995}$ &  &  \multicolumn{1}{c||}{}  & 
         & $0.611$ & $0$ &  \\
    \hline
    
    \end{tabular}       
    \end{scriptsize}
    }
    \caption{Average cumulative reward $\bar{F}_{\Y}$ ($t = T/3, 2T/3, T$), normalized by fractional optimal $F^\star$, of fractional policies (\Algorithm{RAOCO} and \Algorithm{Random}) across different datasets and constraints, along with average execution time per timeslot (in seconds). Highest  $\bar{F}_{\Y}/F^\star$ values are indicated in bold.  The policies $\texttt{FSF}^*$ and \texttt{TabularGreedy} by construction only produce integral solutions (whose performance is reported in Table~\ref{tab:res}). The policy $\texttt{FSF}^*$ only operates on uniform matroids. The standard deviation of the rewards of \texttt{RAOCO} policies is 0, since the fractional rewards are deterministic, and thus omitted from the table. \Algorithm{RAOCO} combined with \Algorithm{OGA} or \Algorithm{OMA} outperforms \Algorithm{Random}, while almost reaching the optimal value~1. As \Algorithm{Random} also performs well on \texttt{MovieLens}, this indicates that the (static) offline optimal is quite poor for this reward sequence. With respect to the execution time, \Algorithm{RAOCO} policies are consistently faster by orders of magnitude than competing policies $\texttt{FSF}^*$ and \texttt{TabularGreedy}, and pretty close to (trivial to execute) policy \Algorithm{Random}; an exception is \Algorithm{TabularGreedy} on \texttt{MovieLens} under a partition matroid constraint, where its execution time is slightly better, but still comparable to \Algorithm{RAOCO} policies.  }
    \label{tab:frac_res_large}
\end{table}

\subsection{Additional Results}\label{sec:appendix_additional2}
\paragraph{Fractional policies.}
We report the fractional rewards $\bar{F}_{\Y}$ in Table \ref{tab:frac_res_large}. 
We observe that both \RAOCOOGA{} and \RAOCOOMA{} significantly outperform \Algorithm{Random}. Note that in the \FLMovieLens{} dataset all algorithms, including \texttt{Random}, attain almost optimal rewards. In comparison with the integral rewards (Table \ref{tab:res}), \RAOCOOGA{} slightly outperforms \RAOCOOMA{} in the \IMEpinions{} dataset. Overall, in all other datasets, fractional results are on par with the integral results reported in Table~\ref{tab:res}.

\paragraph{Execution times.}
We report the average execution time per timeslot (avg. time per timeslot) in Table~\ref{tab:frac_res_large}. We observe that \RAOCOOGA{} and \RAOCOOMA{}, significantly outperform all other policies, except \texttt{Random}. Particularly, in the 
\IMZKC{} dataset \texttt{RAOCO} algorithms are faster than the competitors by one order of magnitude, while in the
\IMEpinions{} dataset \texttt{RAOCO} algorithms are faster than the competitors by two orders of magnitude.

\paragraph{Dynamic Regret.} 
\label{appendix: dynamic regret} Recall that we consider two settings for the \WeightedCoverage{} dataset, which we designed to study dynamic regret: one with static rewards, and one with a distribution shift.  Recall that optimal decision under the two objectives are disjoint; therefore, the different policies need to relearn for the different objective.   We consider the following configurations for the different policies: \Algorithm{OGA} ($\eta=0.001$), \Algorithm{OMA} ($\eta=0.05, \gamma =0$) configured with negative-entropy, \Algorithm{OMA$^\star$} ($\eta=0.05, \gamma = 0.02$) configured with negative-entropy, \Algorithm{FSF$^\star$} ($\eta=0.05, \gamma = 0.02$), \Algorithm{TabularGreedy} ($\eta =0.05, c_p = 1$).

The results are provided in Fig.~\ref{fig:dynamicregret}. We observe that all the policies have similar performance in the stationary setting and are  able to match the performance of the static optimum. In the non-stationary setting, we observe that that  \Algorithm{OGA}, \Algorithm{OMA}, and \Algorithm{FSF$^\star$} are robust to to non-stationarity and are able to match the performance of a dynamic optimum. \Algorithm{TabularGreedy} is less robust to non-stationarity and is only able to match the performance of  a static optimum. We observe that the shifted version of \Algorithm{OMA} increases its robustness and provides a slightly better performance compared to \Algorithm{OGA}.

\paragraph{Optimistic Learning.}\label{appendix: optimistic learning}
We now consider  the third variant of \WeightedCoverage{} dataset, in which we alternate between the two objectives at each timeslot. Under this setting, it is difficult to track the dynamic optimum since $P_T =\Omega\parentheses{T}$ (see our theoretical result in Sec.~\ref{sec:extensions}; moreover, the regret $\BigO{\sqrt{P_T T}}$ is order-wise optimal~\cite{zhang2018adaptive}). However, if we incorporate predictions, the large $P_T$ can be dampened when the predictions are accurate. We provide as predictions the future supergradient $\vec g_t$ with added Gaussian noise with mean $\vec 0$ and standard deviation $n_\sigma \in \set{10, 200}$.

The results are provided in Fig.~\ref{fig:optimistic_figure}. We observe that when the predictions are accurate \Algorithm{Optimistic OGA} is able to exploit these predictions to match the performance of the dynamic optimum. We note that for larger learning rate, the performance improves as the policy can follow aggressively the predictions instead of the past rewards, whereas classical  \Algorithm{OGA} diverges for this learning rate configuration. The main difficulty with optimistic leaning is that the learning rate selection depends on the quality of the received predictions. This motivates the next experiment that considers a meta-learning setup that learns the appropriate learning rates in a parameter-free fashion.

    
\paragraph{Meta-Policies.}
We consider the \WeightedCoverage{} dataset and a fixed function equal to the average of the two objectives. We construct a parameter-free \Algorithm{OGA} that does not require a learning rate. The learning rate is adapted through the received gradients; such a learning rate schedule is known as  ``self-confident''~\cite{auer2002adaptive}.  This policy acts as a meta-policy that learns over different policies configured with different learning rates $\eta \in \set{5\times10^{-4}, 1\times10^{-3}, 2\times10^{-3}, 4\times10^{-3}}$. The results are provided in Fig.~\ref{fig:meta}. We observe that the meta-policy quickly learns the best learning rate.

\input{figs/fig2}
\begin{figure}[t]
    \centering
    \includegraphics[width = .4\linewidth]{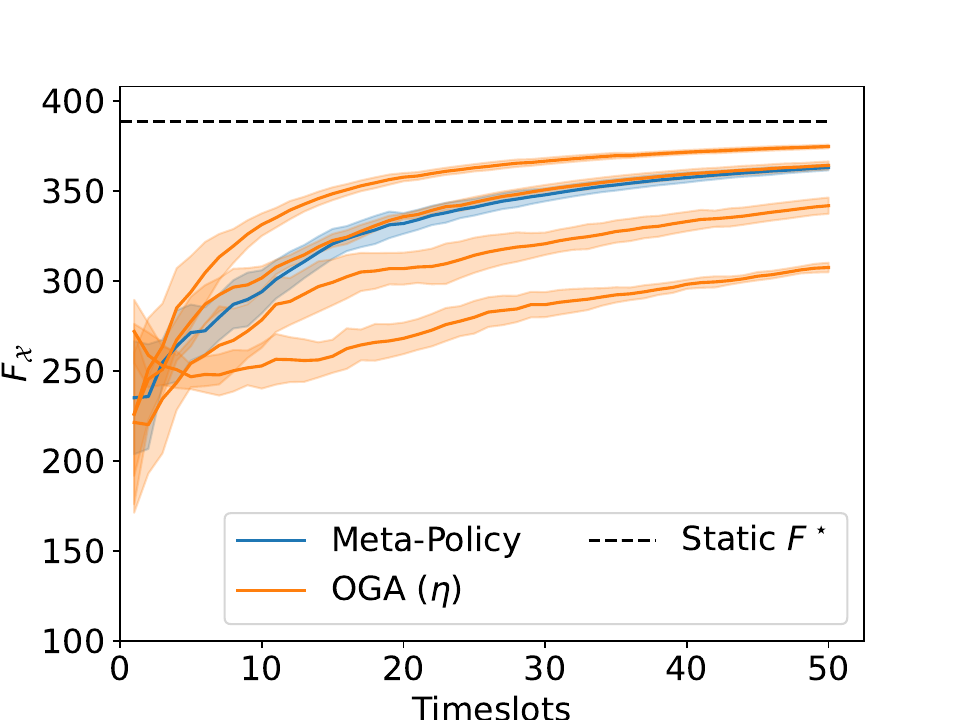}
    \caption{ Average cumulative reward $\bar F_\X$ of  \Algorithm{OGA} under \WeightedCoverage{} dataset for different learning rates.  The meta-policy can learn the best configuration of OGA without tuning the learning rate. The area depicts the standard deviation over 5 runs.}\label{fig:meta}
\end{figure}

%% file: figs/fig2.tex
\begin{figure}[t]
    \centering
   {
   \includegraphics[width = .4\textwidth]{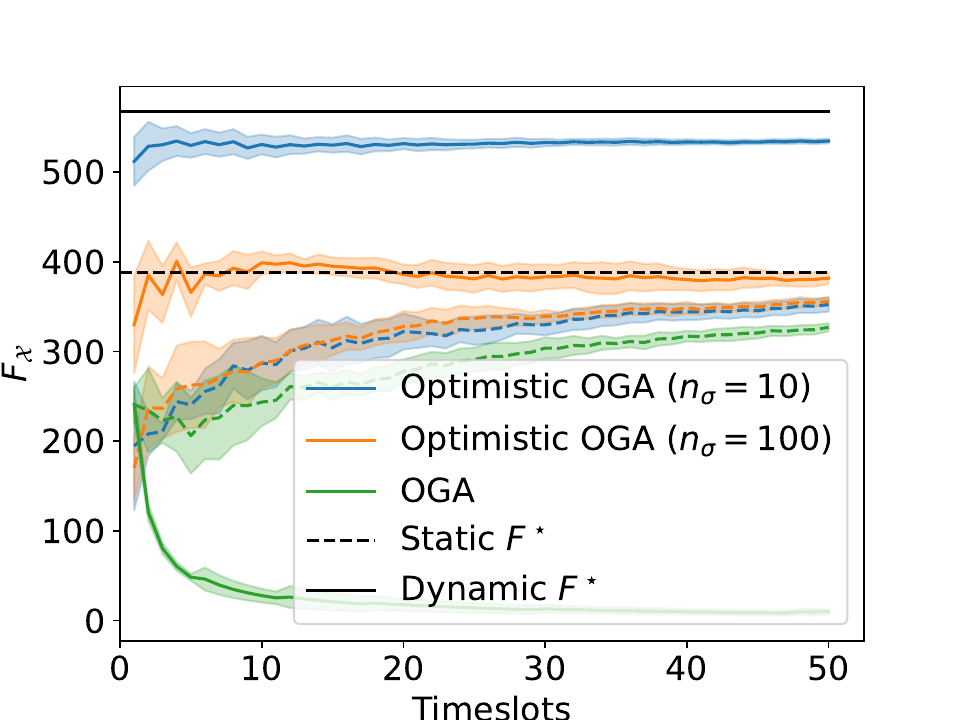}}
    \caption{ Average cumulative reward $\bar F_\X$ of the different policies under \WeightedCoverage{} dataset under a non-stationary setup: the objective is changed at every timeslot.  
 The algorithms \Algorithm{Optimistic OGA} and \Algorithm{OGA} are executed with  different learning rates under different prediction accuracy (noise with std. dev. $n_{\sigma} \in \set{10, 100}$). The larger learning rate is depicted by a solid line. 
    The area depicts the standard deviation over 5 runs.}\label{fig:optimistic_figure}
\end{figure}